\documentclass[journal, onecolumn, 12pt]{IEEEtran}

\usepackage{amsmath, amsthm, amsfonts, amssymb}
\usepackage{graphicx}
\usepackage{epstopdf}
\usepackage{tikz}
\usepackage{dsfont}
\usepackage{algorithm,algorithmicx,algpseudocode}

\newcommand{\norm}[1]{\lVert#1\rVert}
\newcommand{\diverg}[1]{\operatorname{div}\left( #1 \right)}

\newcommand{\inprod}[2]{\langle #1, #2 \rangle}

\newtheorem{proposition}{Proposition}
\makeatletter
\newcommand{\eqnum}{\leavevmode\hfill\refstepcounter{equation}\textup{\tagform@{\theequation}}}
\makeatother
\if CLASSOPTIONcompsoc
\usepackage[caption=false,font=normalsize,labelfont=sf,textfont=sf]{subfig}
\else
\usepackage[caption=false,font=footnotesize]{subfig}
\fi

\ifCLASSOPTIONcompsoc
\usepackage[nocompress]{cite}
\else
\usepackage{cite}
\fi

\begin{document}

\title{A Finite Element Computational Framework for Active Contours on Graphs}

\author{Nikolaos~Kolotouros and~Petros~Maragos
	\thanks{This work was primarily performed while N. Kolotouros was at the National Technical University of Athens.}
	\thanks{N. Kolotouros is with the Department of Computer and Information Science, University of Pennsylvania, Philadelphia, PA 19104 USA (e-mail: nkolot@seas.upenn.edu)}
	\thanks{P. Maragos is with the School of Electrical and Computer Engineering,
		National Technical University of Athens, Athens, Greece (e-mail: maragos@cs.ntua.gr)}
}

\maketitle

\begin{abstract}
In this paper we present a new framework for the solution of active contour models on graphs. With the use of the Finite Element Method we generalize active contour models on graphs and reduce the problem from a partial differential equation to the solution of a sparse non-linear system. Additionally, we extend the proposed framework to solve models where the curve evolution is locally constrained around its current location. Based on the previous extension, we propose a fast algorithm for the solution of a wide range active contour models. Last, we present a supervised extension of Geodesic Active Contours for image segmentation and provide experimental evidence for the effectiveness of our framework.
\end{abstract}

\begin{IEEEkeywords}
Active contours, finite element analysis, image segmentation, graph segmentation
\end{IEEEkeywords}

\IEEEpeerreviewmaketitle

\section{Introduction}\label{intro}
Graph-based methods have become very popular in computer vision and image processing. Graph theory provides an efficient and rigorous framework to model complex relationship between data and has been used for various tasks such image segmentation, motion detection and pattern matching. Typically, images are represented as graphs with the image pixels treated as the graph vertices, whereas the graph edges can be chosen from a wide range of options, depending on the suitability for each particular application. Examples of popular graph-based methods include \emph{Graph Cuts} \cite{boykov_jolly,boykov_veksler_zabih,kolmogorov_zabih}, \emph{Random Walker} \cite{grady}, \cite{bampis} and \emph{Power Watershed} \cite{couprie_grady_najman_talbot,cousty_bertrand_najman_couprie_2009}.

Active Contours form a class of methods that try to minimize continuous energy functionals associated with a particular problem. As implied by their name, they are curves that evolve in the domain of an image and are used, among other applications, for object detection and image segmentation. They were introduced by Kass et al. in the form of ``snakes'' \cite{kass}. Widely used active contour models include \emph{Geodesic Active Contours} \cite{cas_kim_sap_97} and \emph{Active Contours Without Edges} \cite{chan_vese}.

In this paper we will focus on the problem of active contours on graphs. We will generalize the notion of active contours on graphs and present a computational framework that allows the solution of a wide range of active contour models that can be modeled with levelset equations. Our approach is based on the Finite Element method that simplifies the form of the problem from a partial differential equation to a sparse non-linear system. We will also present an extension to our model that under reasonable assumptions can be used to speed up the curve evolution on large graphs. Our analysis will focus mainly on segmentation applications, however its scope is much more general.

Our main contributions are:
\begin{itemize}
\item We develop a novel theoretical and computational framework for the solution of general active contour models on graphs using the Finite Element method.
\item We generalize narrow band levelset methods on graphs and present an efficient algorithm for active contour evolution on large graphs.
\item We present an extension of the Geodesic Active Contour model that incorporates statistical region information, which achieves results that are within state-of-the-art.
\end{itemize}

In Section \ref{section:acog} we present a novel approach for the solution of general active contour models on graphs using the Finite Element method. In Section \ref{section:constrained_acog} we provide an extension to our framework that can solve locally constrained active contour models and can also be used for fast contour evolution on graphs. In Section \ref{section:gar} we propose a modification of the Geodesic Active Contours model that includes statistical region information. In Section \ref{section:experimental} we provide experimental evidence for the efficiency of our method, focusing on segmentation applications and last, in Section \ref{section:concl} we make some concluding remarks.

\section{Active Contours on Graphs}\label{section:acog}
\subsection{Problem Formulation}
In this section we will present a novel method for solving Active Contour evolution equations on graphs employing the Finite Element method. Finite Element Analysis is a powerful framework that enables us to solve complex partial differential equations in a simple and elegant manner. The other popular family of methods for the solution of PDEs is the Finite Difference Method. The main difference between these 2 approaches is that Finite Difference methods try to approximate the differential operators using discrete schemes whereas in the Finite Element Method we approximate the solution of the equation with functions belonging to finite dimensional function spaces. This can be useful, especially in the case of small graphs where the accurate approximation of differential operators is a challenging problem. An example of a Finite Difference approach for the solution of active contour models on graphs is proposed in \cite{sakaridis}.

In this paper, we will consider general Active Contour models that can be modeled with levelset equations of the form

\begin{gather}
\label{eq:levelset_general}
F(u)\frac{\partial u}{\partial t} = \diverg{G(u)\nabla u} + H(u)\\
\label{eq:levelset_general_boundary_cond}
\nabla u \cdot \mathbf{n} = 0 \text{ on the boundary } \partial \Omega\\
u(x,y,0) = \operatorname{dist^\ast}(x,y)
\end{gather}
where $A$, $B$ and $C$ are functionals of $u$ and $\operatorname{dist^\ast}(x,y)$ is the signed distance function from the initial curve.

Popular active contour models that can be expressed in the above form are
\begin{itemize}
\item \emph{Erosion/Dilation}:
\begin{equation}
F(u) = \frac{1}{\norm{\nabla u}}, \quad G(u) = 0, \quad H(u) = c
\end{equation}
\item \emph{Geometric Active Contours}:
\begin{equation}
F(u) = \frac{1}{g_I\norm{\nabla u}}, \quad G(u) = \frac{1}{\norm{\nabla u}}, \quad H(u) = 0
\end{equation}
\item \emph{Geodesic Active Contours}:
\begin{equation}
F(u) = \frac{1}{\norm{\nabla u}}, \quad G(u) = \frac{g_I}{\norm{\nabla u}}, \quad H(u) = \beta
\end{equation}
\end{itemize}
where $g_I(x,y)$ is the edge stopping function defined as
\begin{equation}
g_I = \frac{1}{1+\lambda \norm{\nabla I}^2}
\label{eq:edge_stopping_function}
\end{equation}
As a first step, we will convert \eqref{eq:levelset_general}' into an ``equivalent'' integral equation. To do this we multiply both parts of \eqref{eq:levelset_general} with a function $\phi \in H^1(\Omega)$ and integrate in $\Omega$. $H^1(\Omega)$ is the Sobolev space consisting of all functions defined in $\Omega$ whose first order derivatives --in the distributional sense-- belong in $L^2(\Omega)$, the space of Lebesgue square-integrable functions.
Thus we have
\begin{equation}
\begin{aligned}
\iint\limits_{\Omega}F(u)\frac{\partial u}{\partial t}\phi \, {d}x {d}y
&= \iint\limits_{\Omega}\diverg{G(u)\nabla u}\phi \, {d}x {d}y\\
&+ \iint\limits_{\Omega}H(u)\phi \, {d}x {d}y
\end{aligned}
\end{equation}
Using Green's identity
\begin{equation}
\iint\limits_{\Omega}f\diverg{\mathbf{F}}\, {d}x {d}y = - \iint\limits_{\Omega}\nabla f \cdot \mathbf{F}\, {d}x {d}y+ \oint\limits_{\partial \Omega}f~\mathbf{F}\cdot \mathbf{n}\, dr
\end{equation}
and the boundary condition \eqref{eq:levelset_general_boundary_cond} we can obtain the final form of the equation
\begin{equation}
\begin{aligned}
\iint\limits_{\Omega}F(u)\frac{\partial u}{\partial t}\phi\, {d}x {d}y
&= -\iint\limits_{\Omega}G(u)\nabla u \cdot \nabla \phi\, {d}x {d}y\\
&+ \iint\limits_{\Omega}H(u)\phi\, {d}x {d}y
\end{aligned}
\label{eq:weak_form}
\end{equation}
or more generally in functional form
\begin{equation}
\mathcal{F}(u,\dot{u},\phi) = 0
\label{eq:functional_form}
\end{equation}
We then demand that \eqref{eq:functional_form} holds for all $\phi \in H^1(\Omega)$. This is called the weak form of \eqref{eq:levelset_general}.
\subsection{Galerkin Approximation}
Until now we have not made a numerical approximation to the problem. We only converted it to an integral form that we expect to be equivalent to the original problem in some sense. We will try to approximate the solution using the Galerkin method. The core of the Finite Element analysis is that we do not approximate continuous differential operators using finite differences but instead approximate the solution of the equations with functions belonging to finite dimensional subspaces of $H^1(\Omega)$.

Let $V_N$ a N-dimensional subspace of $H^1(\Omega)$ and $\{\phi_i\}_{i=1}^n$ a basis of $V_N$. The solution approximation $\bar{u}$ can be written as a linear combination of the basis functions and since we have a time-dependent problem we allow the coefficients of the linear combination to be functions of time.
Thus we have
\begin{equation}
\bar{u}(x,y,t) = \sum_{i=1}^{N}c_i(t)\phi_i(x,y)
\end{equation}

We demand that \eqref{eq:weak_form} holds at least for all the functions that belong in the subspace $V_N$. Since \eqref{eq:functional_form} is a linear functional of $\phi$ and $\{\phi_i\}_{i=1}^N$ form a basis of $V_N$, this is equivalent to demanding that it holds for each of the basis functions $\phi_i$.
If we use the fact that
\begin{equation}
\frac{\partial \bar{u}}{\partial t}(x,y,t) = \sum_{i=1}^{N}\dot{c}_i(t)\phi_i(x,y)
\end{equation}
and substitute in \eqref{eq:weak_form} we get
\begin{equation}
\begin{aligned}
\sum_{i=1}^{N}\dot{c}_i\iint\limits_{\Omega}F(\bar{u})\phi_i\phi_j\,dxdy
&= -\sum_{i=1}^{N}\iint\limits_{\Omega}G(\bar{u}) \nabla \phi_i \cdot \nabla \phi_j\,dxdy\\
&+ \iint\limits_{\Omega}H(\bar{u})\phi_j, \, j=1,\dots n
\end{aligned}
\end{equation}
where
\begin{equation}
\dot{c}_i = \frac{{d} c_i}{{d} t}
\end{equation}

This is a non-linear system of ODEs that can be written in the form
\begin{equation}
\mathbf{A}(\mathbf{c})\dot{\mathbf{c}} = \mathbf{b}(\mathbf{c})
\label{eq:non_linear_ode}
\end{equation}
where
$\mathbf{A} = \{\mathrm{A}_{ij}\}$ a $N \times N$ matrix that depends on $\mathbf{c}$ and $\mathbf{b} = \{b_i\}_{1}^{N}$, $\mathbf{c} = \{c_i\}_{1}^{N}$ $N$-dimensional vectors with
\begin{equation}
\mathrm{A}_{ij}(\mathbf{c}) = \mathrm{A}_{ji}(\mathbf{c}) = \iint_{\Omega}F\left(\sum_{k=1}^{N}c_k\phi_k\right)\phi_i \phi_j\, dxdy
\label{eq:A_elements}
\end{equation}
and
\begin{equation}
\begin{aligned}
{b}_i(\mathbf{c}) &=  -\sum_{j=1}^{N}\iint\limits_{\Omega}G\left(\sum_{k=1}^{N}c_k \phi_k\right)\nabla \phi_i \cdot \nabla \phi_j\, dxdy\\
&+ \iint\limits_{\Omega}H\left(\sum_{k=1}^{N}c_k\phi_k\right)\phi_i\, dxdy
\end{aligned}
\label{eq:b_elements}
\end{equation}

The initial condition for the above system of differential equations can be obtained from the projection $\bar{u}(x,y,0)$ of $u(x,y,0)$ in the subspace $V_N$, i.e.
\begin{equation}
\bar{u}(x,y,0) = \sum_{i=1}^{N}\inprod{u(x,y,0)}{\phi_i}\phi_i
\end{equation}
and thus
\begin{equation}
c_i(0) = \inprod{u(x,y,0)}{\phi_i}
\label{eq:c0}
\end{equation}
where $\inprod{\cdot}{\cdot}$ is the usual inner product defined on $H^1(\Omega)$ with
\begin{equation}
\inprod{f}{g} =\iint\limits_{\Omega}fg\, dx dy + \iint\limits_{\Omega}\nabla f\cdot\nabla g\, dx dy
\end{equation}
\subsection{Generalization on Graphs}
Consider a graph $G(V,E)$ where $V$ is the set of vertices and $E$ the set of edges. We assume that the graph has a total of $N$ vertices and each vertex $v_i$ is a point in $\Omega$ and can be described by its respective planar coordinates $(x_i, y_i)$. Note that every finite set of points in the plane can be mapped into the unit square by applying a translation followed by a scaling. Thus if we apply an appropriate geometrical transformation we can map the set of vertices to a subset of $\Omega$.  For each vertex $v_i \in V$ we choose a function $\phi_i$ that belongs in $H^1(\Omega)$ that has the following properties

\begin{itemize}
\item $ \{\phi_i\}_{i=1}^N$ are linearly independent \eqnum\label{eq_phi_prop_1}
\item  for $v_i \nsim v_j$, $\operatorname{supp}(\phi_i) \cap \operatorname{supp}(\phi_j) = \emptyset$ \eqnum\label{eq_phi_prop_2}
\end{itemize}

We can see that the functions $  \{\phi_i\}_{i=1}^n$ form a basis of some subspace $S$ of $H^1(\Omega)$. Also, for reasonably smooth functions $\phi_i$, $\operatorname{supp}(\nabla \phi_i) \backslash \operatorname{supp}(\phi_i)$ is a null set\footnote{If $\phi_i$ is zero in an interval, then its derivative is also zero. Thus, if $\phi_i$ has at most countably infinite isolated roots, the above holds.}  and thus $\operatorname{supp}(\nabla \phi_i) \cap \operatorname{supp}(\nabla \phi_j)$ is also a null set. Thus we can conclude that if $v_i$ and $v_j$ are two non-adjacent vertices of $G$, then $\inprod{\phi_i}{\phi_v} = 0$ and thus $\phi_i, \phi_j$ are orthogonal.

\begin{proposition}
For any graph $G(V,E)$ with $v_i = (x_i, y_i) \in \Omega$ there is at least one set of functions $ \{\phi_i\}_{i=1}^N$ with the properties \eqref{eq_phi_prop_1} and \eqref{eq_phi_prop_2}.
\end{proposition}
\begin{proof}
Let $ d = \operatorname{min}\limits_{i\neq j}\norm{v_i - v_j}$. We define the functions
\begin{equation}
\phi_i(x,y)= (d^2-\norm{(x,y) - (x_i,y_i)}_2^2) \cdot \mathds{1}_{\{\norm{(x,y) - (x_i,y_i)}_2 < d\}}
\end{equation}
where $\mathds{1}_A$ is the indicator function of set $A$.
Each function $\phi_i$ is differentiable with
\begin{equation}
\nabla \phi_i(x,y) = 2\cdot\mathds{1}_{\{\norm{(x,y) - (x_i,y_i)}_2 < d\}}\cdot(x_i -x, y_i -y)
\end{equation}
Moreover we can see that $\operatorname{supp}(\phi_i) \cap \operatorname{supp}(\phi_j) = \emptyset$ for all $v_i, v_j \in V$ and since $\phi_i$ is not zero everywhere then $ \{\phi_i\}_{i=1}^N$ are linearly independent.
\end{proof}

Next we will describe how we can obtain the solution of \eqref{eq:non_linear_ode} on a graph $G(V,E)$ equipped with the functions $ \{\phi_i\}_{i=1}^n$. First, from \eqref{eq:A_elements} we can see that $\mathrm{A}_{ij}(\mathbf{c}) = 0$ if $v_i \nsim v_j$. If the number of edges of each vertex is small compared to the total number of vertices then $\mathbf{A}(\mathbf{c})$ is sparse. Additionally the summation in \eqref{eq:b_elements} reduced to a summation in the neighborhood of $v_i$ and thus
\begin{equation}
\begin{aligned}
{b}_i(\mathbf{c}) &=  -\sum_{v_j \in N_i \cup \{v_i\}}\iint\limits_{\Omega}G\left(\sum_{k=1}^{N}c_k\phi_k\right) \nabla \phi_i \cdot \nabla \phi_j\, dxdy\\
&+ \iint\limits_{\Omega}H\left(\sum_{k=1}^{N}c_k\phi_k\right)\phi_i\, dxdy
\end{aligned}
\label{eq:b_elements_neighborhood}
\end{equation}
where $N_i$ is the neighborhood of vertex $v_i$.
\subsection{Delaunay Graphs}
In the previous paragraphs we described how we can generalize active contour models on graphs with arbitrary structure. In the rest of this paper we will limit our analysis in the case of planar graphs and more specifically the case of Delaunay graphs that are constructed from the Delaunay triangulation of a finite set of points in the unit square.

Delaunay triangulation is a method of dividing the convex hull of a set points into triangles that tends to avoid creating sharp triangles, a property that ensures good convergence results for the solution of PDEs. For more details about the Delaunay triangulation and the algorithms used to produce, we refer the reader to the work of Persson and Strang in \cite{persson_strang}.

In \figurename{\ref{fig:delaunay_graphs}}(b) we can see an example of a Delaunay graph with 20 vertices. The vertices were chosen randomly according to the uniform distribution in the unit square. For a Delaunay graph $G$ we will denote the vertices, edges and triangles of the graph by $V(G)$, $E(G)$ and $T(G)$ respectively.

Images can be thought as a special case of Delaunay graphs with the vertices being the pixels of the image. In \figurename{\ref{fig:delaunay_graphs}}(a) we can see the triangulation of the domain of a $7 \times 7$ image and the corresponding Delaunay graph. In this special case, i.e. a Delaunay graph created from a regular grid of points, each non-boundary pixel is connected with an edge with 6 of its adjacent pixels, forming an asymmetric hexagonal neighborhood.

In the case of graphs $I$ is not an image function in the usual sense. Instead we define it to be a function $I: V(G)\to \mathbb{R}$. For the case of the image gradient $\nabla I$, however, since we will be using numerical integration it is more convenient to define it as 2-D function defined in $\Omega$. This function will piecewise constant in each triangle of $T(G)$ and its value is the gradient of the plane that is formed by the values of $I$ in the three vertices of the triangle. Similarly, $g_I = g(\norm{\nabla I})$ will also be constant in each triangle of the triangulation.

\begin{figure}[!t]
\centering
\subfloat[]{
\includegraphics[width=.35\columnwidth]{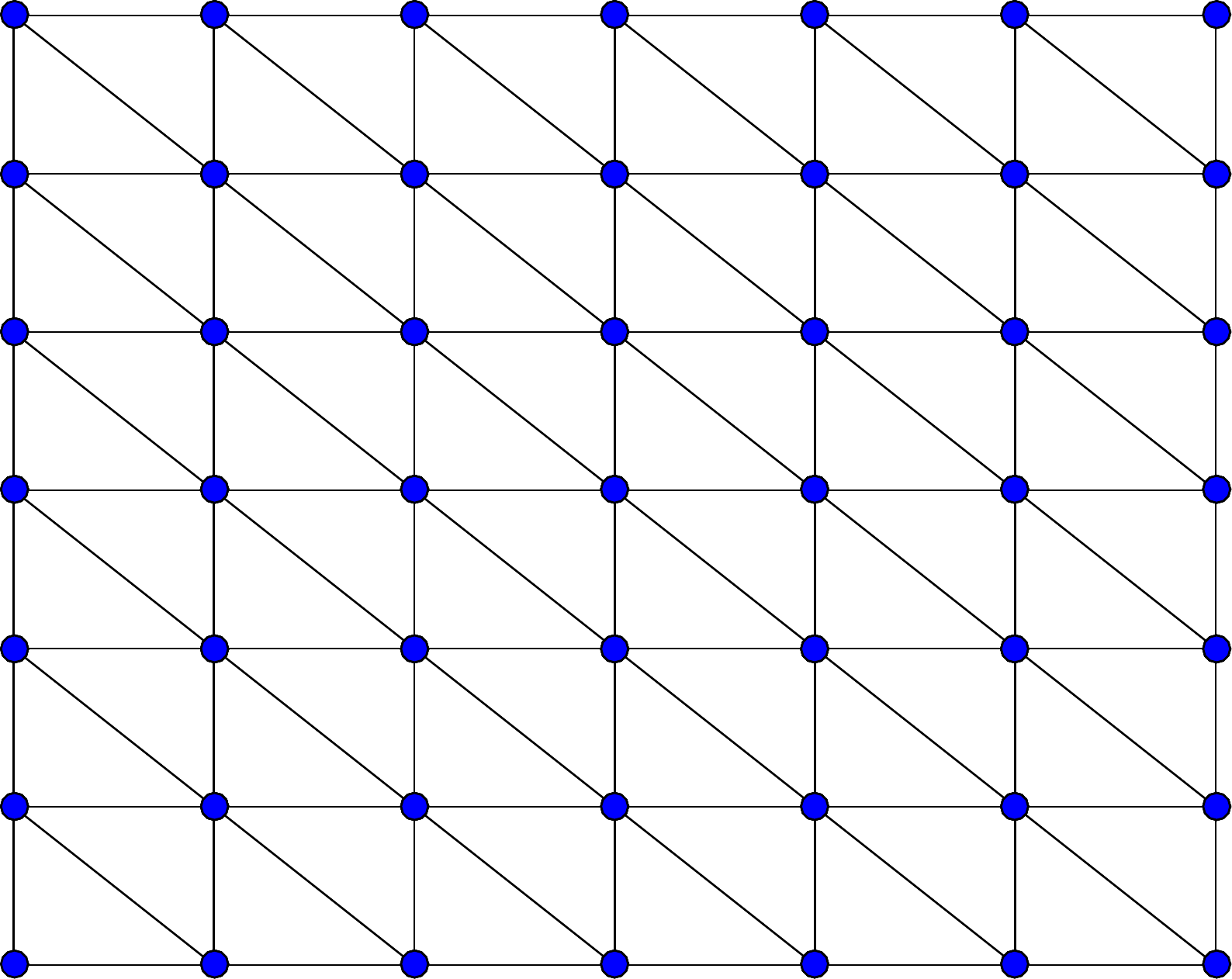}}
\hfil
\subfloat[]{
\includegraphics[width=.35\columnwidth]{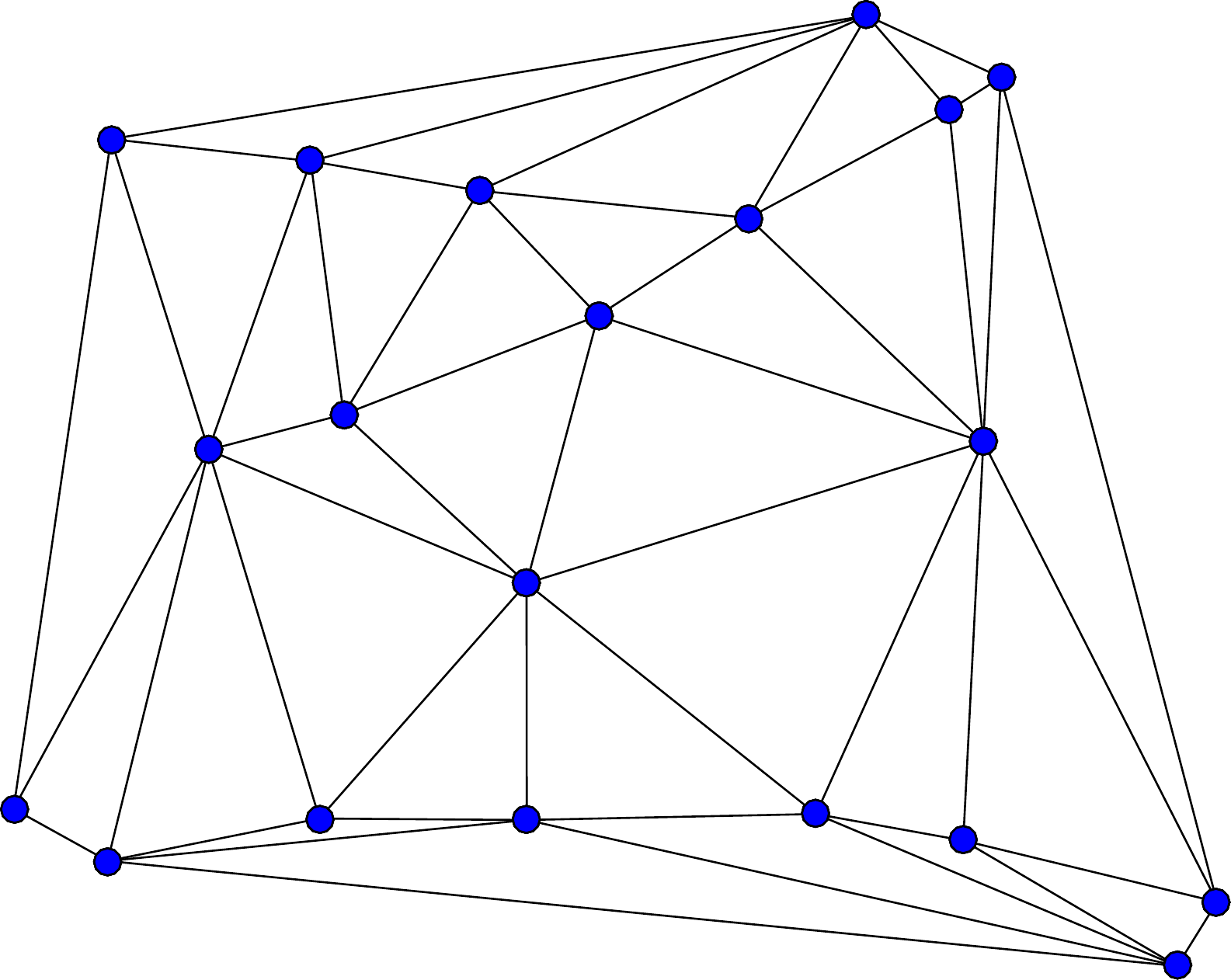}}
\caption{(a): Delaunay Graph with 20 vertices. (b) Delaunay Triangulation of a $7 \times 7$ image }
\hfil
\label{fig:delaunay_graphs}
\end{figure}

\subsection{Choice of subspace and basis}
One of the most important parts in our method is the choice of a subspace of $H^1(\Omega)$ and a set of basis functions $\{\phi_i\}_{i=1}^N$.

Let $\{T_i\}_{i=1}^{M}$ be the triangles of the Delaunay Triangulation. We will try to approximate the solution with a function belonging to the space of continuous functions in $\Omega$ which are linear in each triangle of the triangulation. We will denote this space as $S_{lin}$. More formally
\begin{equation}
S_{lin} = \{f \in C(\Omega): f|_{T_i} = a_i x + b_i y + c_i, \text{ για } i=1,\dots,M\}
\label{linear_space}
\end{equation}
This choice of subspace provides good convergence properties and also simplifies calculations as we shall see later on.

After the selection of the subspace we have to choose a basis of $S_{lin}$.

A possible choice satisfying the properties previously are the pyramid functions. For each vertex $v_i$ of the graph we define the function $\phi_i$ as
\begin{itemize}
\item $\phi_i(v_i) = 1$
\item $\phi_i(v_j) = 0$, for  $j \neq i$
\item $\phi_i$ is linear in each triangle
\end{itemize}

It easy to verify that $\operatorname{supp}(\phi_i) \cap \operatorname{supp}(\phi_j) = \emptyset$ if $v_i \nsim v_j$ and $\phi_i \in S_{lin}$. It can also be proven that the $\{\phi_i\}_{i=1}^{N}$ are linearly independent and thus they form a basis of $S_{lin}$.

In \figurename{\ref{fig:phi}} we can see the form of the $\phi_i$ functions for the case of a rectangular grid. With this choice we can see that $\phi_i$ overlaps only with the 6 other $\phi_j$ that correspond to the vertices $v_j$ that are adjacent to $v_i$. Thus each row of $\mathbf{A}$ will contains at most 7 non zero elements and at the same time only 7 different coefficients $c_i$ will appear in these expressions.
\begin{figure}[!t]
\centering
\includegraphics[width=.5\columnwidth]{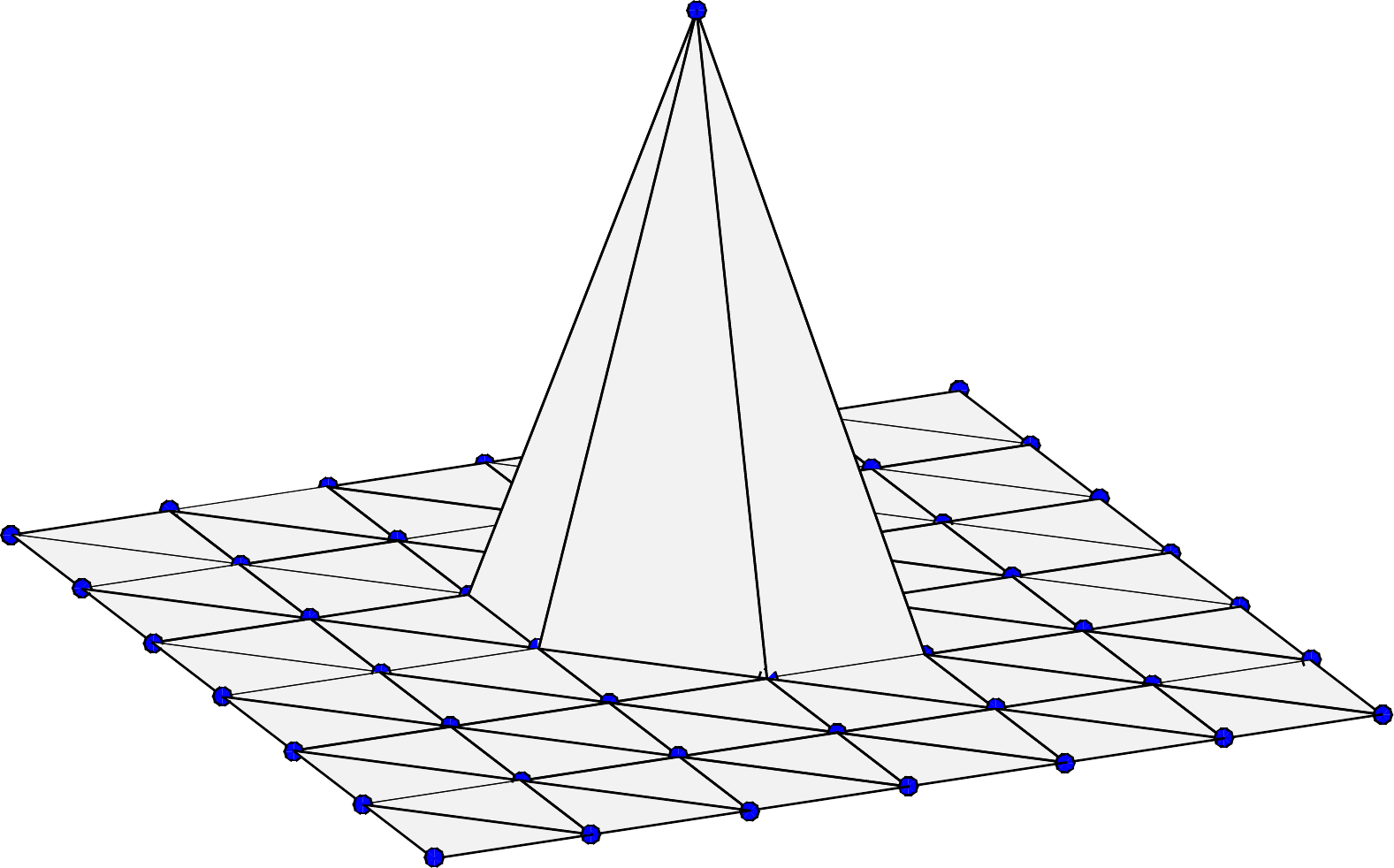}
\caption{Pyramid function $\phi_i$ centered in node $v_i$ for a rectangular grid}
\label{fig:phi}
\end{figure}

We can think of the $\phi_i$ as functions that interpolate discrete data in a rectangular grid. Assume for example that we have samples of a 2D function $f$ at $N$ points $(x_i, y_i)$ in the plane. By creating the Delaunay graph that corresponds to the above set of points we can construct a continuous function $\bar{f}$ that approximates $f$ as
\begin{equation}
\bar{f}(x,y) = \sum_{i=1}^{N}f(x_i,y_i)\phi_i(x,y)
\end{equation}
From the above construction and the properties of $\phi_i$, it is easy to verify that $\bar{f}(x_i,y_i) = f(x_i,y_i)$ and in each triangle of the triangulation we perform a linear interpolation. In \figurename \ref{fig:interpolation} we can see an example of an interpolation of discrete data in a rectangular grid.
\begin{figure}[!htb]
\centering
\includegraphics[width=.5\columnwidth]{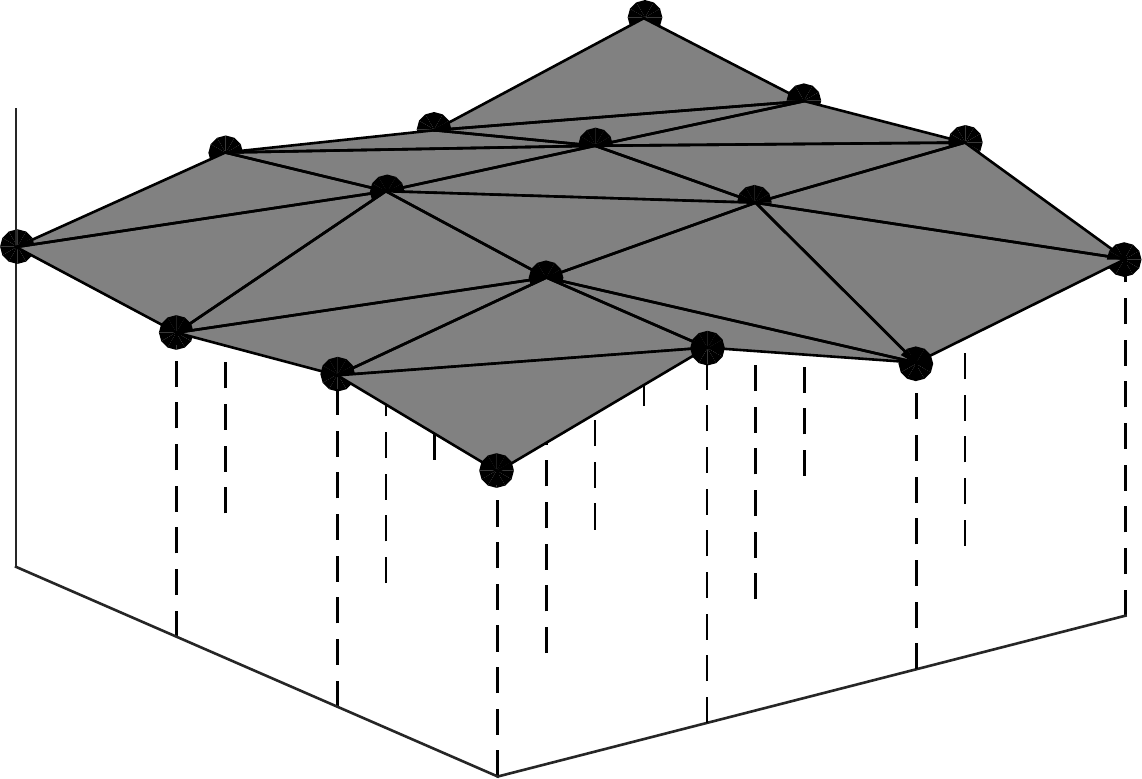}
\caption{Interpolation of discrete data in a $4 \times 4$ rectangular grid using the $\phi_i$ functions. With solid circles we represent the original discrete data}
\label{fig:interpolation}
\end{figure}
\subsection{Time Evolution}
The final step that remains is the discretization in time. For a system of ODEs we have several options for the approximation of derivatives, each one with its advantages and disadvantages and the choice of a particular method depends on the specific properties of each problem.

We want to caclulate the solution in a subset of the time interval $[0,+\infty)$ starting from the initial condition until we reach convergence. Let $t_k, \, k\geq 0$ be the sequence of time points at which we calculate the solution with $t_0=0$ and $t_k = k\Delta t$. Here we assume that we use a fixed time step $\Delta t$. Also with $\mathbf{c}_k$ we will denote the approximation of $\mathbf{c}(t_k)$.

\subsubsection{Explicit Euler method}
Here we approximate the time derivative with the forward difference
\begin{equation}
\dot{\mathbf{c}}(t_k) = \frac{\mathbf{c}_{k+1}-\mathbf{c}_k}{\Delta t}
\end{equation}
If we substitute this into \eqref{eq:non_linear_ode} we obtain
\begin{equation}
\mathbf{A}(\mathbf{c}_k)\mathbf{c}_{k+1} =  \mathbf{A}(\mathbf{c}_k)\mathbf{c}_{k} + \Delta t \cdot \mathbb{b}(\mathbf{c}_k)
\label{eq:explicit_euler}
\end{equation}
This means that we need to solve a linear system for each time step. This method is easy to implement but puts limits on the choice of time step $\Delta t$. We need to choose a very small time step --often in the order of $1/N^2$ to ensure that the curve evolution is stable.
\subsubsection{Implicit Euler method}
In this method we approximate the time derivative with the backward difference
\begin{equation}
\dot{\mathbf{c}}(t_k) = \frac{\mathbf{c}_{k}-\mathbf{c}_{k-1}}{\Delta t}
\end{equation}
If we substitute again the above relationship into \eqref{eq:non_linear_ode} we obtain
\begin{equation}
\mathbf{A}(\mathbf{c}_{k+1})\left(\mathbf{c}_{k+1}-\mathbf{c}_{k}\right) - \Delta t \cdot \mathbf{b}(\mathbf{c}_{k+1}) = 0
\end{equation}
which is a non linear system of equations. This adds a further computational burden in each time step but its main advantages is that it allows us to use a relatively larger time step than the explicit Euler since it has a larger region of stability.
\subsection{Algorithmic Complexity}
In this part we will try to provide an estimate for the complexity of the presented algorithm using the explicit Euler method for the time derivative approximation. For the sake of simplicity let us consider the case of the Delaunay graphs that correspond to images similar to the one depicted in \figurename{\ref{fig:delaunay_graphs}}(a). We assume that we have a $m \times n$ image. The resulting graph will have a total of $N=mn$ vertices. Since each row of $\mathbf{A}$ contains at most 7 nonzero elements $\mathbf{A}(\mathbf{c}_k)$ can be computed in linear time in each time step. Also due to its sparsity, matrix multiplication can also be done in linear time. Observing that $\mathbf{b}(\mathbf{c}_k)$ can also be computed in linear time the right hand side needs $O(N)$ time in each step. With an appropriate node labeling, $\mathbf{A}(\mathbf{c}_k)$ takes the form of a band matrix with a band length of $\min(m,n)$. In the case of square $n\times n$ images this is equal to $\sqrt{N}$. For rectangular images it is safe to assume that $m$ and $n$ are of comparable size and thus this holds also for $\sqrt{N}$ and $\min(m,n)$, so the results obtained for square images will be also valid in the more general case.
We can see the sparsity pattern in \figurename{\ref{fig:sparsity_pattern}}.
\begin{figure}[!t]
\centering
\includegraphics[width=.35\columnwidth]{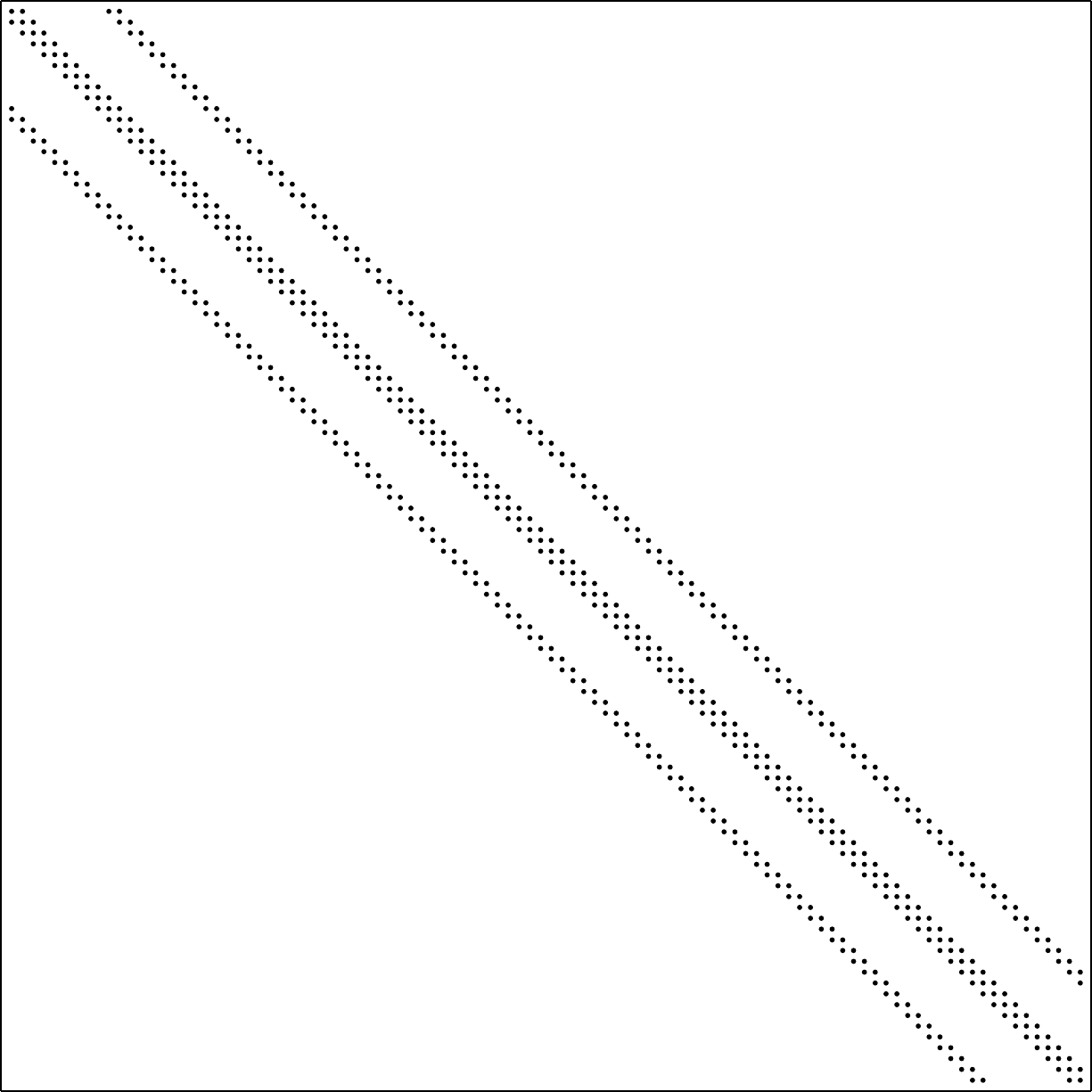}
\caption{Sparsity pattern of matrix A}
\label{fig:sparsity_pattern}
\end{figure}

For a $N\times N$ band matrix with band length $d$ the solution of a linear system needs $O(Nd^2)$ operations. Thus in the case of a $n\times n$ image we need to do $O(N^2) = O(n^4)$ operations in each time step, which is prohibitive for large images. 

As far as it concerns the time evolution, the total number of steps until convergence cannot be specified a priori since it depends on the shape and position of the initial curve. If the initial curve is close to the object boundaries only a few time steps are needed until convergence. Thus the computational complexity of the method can is proportional to the total number of steps until convergence.

\subsection{Numerical Calculations}

Here we will briefly describe how we can compute the elements of the matrix $\mathbf{A}$ and vector $\mathbf{b}$. First note that since $\Omega = \cup_{i=1}^M T_i$, the equations \eqref{eq:A_elements} and \eqref{eq:b_elements} can be rewritten as
\begin{equation}
\mathbf{A}_{ij}(\mathbf{c}) = \sum_{k=1}^{M}\iint_{T_k}F\left(\sum_{k=1}^{N}c_k\phi_k\right)\phi_i \phi_j\, dxdy
\label{eq:A_elements_t}
\end{equation}
and
\begin{equation}
\begin{aligned}
\mathbf{b}_j(\mathbf{c}) =  &-\sum_{k=1}^{M}\sum_{i=1}^{N}\iint\limits_{T_k}G\left(\sum_{k=1}^{N}c_k\phi_k\right) \nabla \phi_i \cdot \nabla \phi_j\, dxdy\\
&+ \sum_{k=1}^{M}\iint\limits_{T_k}H\left(\sum_{k=1}^{N}c_k\phi_k\right)\phi_j\, dxdy
\end{aligned}
\label{eq:b_elements_t}
\end{equation}
For each $i$, $\phi_i$ is zero everywhere except for the triangles formed by the vertex $v_i$ where it is linear. Thus $\nabla \phi_i$ is piecewise constant in each triangle with $\nabla \phi_ι = \mathbf{0}$ at least where $\phi_i = 0$. If we want to be mathematically precise, $\phi_i$ is differentiable everywhere except for the edges of the triangles, but since lines are sets of zero measure in $\mathbb{R}^2$ the values at the edges do not contribute to the value of the integral.

Taking all these into account we conclude that to calculate $\mathbf{A}_{ij}$ and $\mathbf{b}_j$ we have to sum the contributions of the integrals in a small number of triangles. In the case of the diagonal elements $\mathbf{A}_{ii}$, we have to sum over all the triangles that have $v_i$ as a vertex, whereas for the non-diagonal terms $\mathbf{A}_{ij}$ the summation is done over at most 2 triangles, the shared triangles between the vertices $v_i$ and $v_j$.

If $T$ is a triangle with vertices $P_1$, $P_2$ and $P_3$ and $f$ a polynomial of degree at most 2 then it holds that
\begin{equation}
\iint\limits_{T}f(x,y){d}x{d}y = \frac{E(T)}{3}(f(P_1)+f(P_2)+f(P_3))
\label{eq:numerical_triangle}
\end{equation}
where $E(T)$ is the area of the triangle. Thus using the above formula we can obtain closed-form formulas to compute equations \eqref{eq:A_elements_t} and \eqref{eq:b_elements_t}.

Finally, as far as it concerns the initial condition we can approximate \eqref{eq:c0} by setting
\begin{equation}
{c}_i(0) = u(x_i, y_i, 0)
\label{eq:c_initial_cond}
\end{equation}
 and thus
\begin{equation}
\bar{u}(x,y,0) = \sum_{i=1}^{N} u(x_i, y_i, 0)\phi_i(x,y) 
\end{equation}
Under reasonable assumptions we expect that the two expressions for the initial condition give similar values and thus we can adopt this more straightforward approach that does not involve complex integrations.

\section{Locally Constrained Contour Evolution}\label{section:constrained_acog}
\subsection{Method Description}
In this section we will extend the previous framework to solve curve evolution models of the form
\begin{equation}
\begin{gathered}
F(u)\frac{\partial u}{\partial t} = \delta_\epsilon(u)\left(\diverg{G(u)\nabla u} + H(u)\right)\\
\nabla u \cdot \mathbf{n} = 0 \text{ on the boundary } \partial \Omega\\
u(x,y,0) = \operatorname{dist^\ast}(x,y)
\end{gathered}
\label{eq:levelset_general_constrained}
\end{equation}
where $\delta_\epsilon(x)$ is an approximation of the Dirac $\delta$ function, with $\delta_\epsilon \to \delta$, as $\epsilon \to 0$. A typical choice for $\delta_\epsilon$ is the piecewise constant approximation
\begin{equation}
\delta_\epsilon(x) = \begin{cases}
1 / \epsilon,&|x| \leq \epsilon,\\
0, &|x| > \epsilon
\end{cases}
\label{eq:delta_approximation}
\end{equation}
The term $\delta_\epsilon$ constrains the curve evolution in a small area near the current position of the curve, i.e the 0-levelset. One popular active contour model that can be modeled using \eqref{eq:levelset_general_constrained} is the Active Contour Without Edges (ACWE), presented by Chan and Vese in \cite{chan_vese}. In the ACWE model, if we use the same notation with \cite{chan_vese},
\begin{equation}
\begin{gathered}
F(u) = 1, \quad G(u) = \frac{\mu}{\norm{\nabla u}}\\
H(u) = -\nu - \lambda_1(u_0 - c_1)^2 + \lambda_2(u_0-c_2)^2
\end{gathered}
\end{equation}
The constants $c_1$ and $c_2$ are the mean values of the image function $u_0$ on the inside and the outside of the contour respectively and $\lambda_1$, $\lambda_2$ and $\mu$ positive constants.
Also, as will be shown later, constrained curve evolution can be used to speed up the solution of active contour models that can take the form of \eqref{eq:levelset_general}.

First, we will describe how \eqref{eq:levelset_general_constrained} can be approximated on graphs. For this reason we define the sets
\begin{equation}
\mathcal{U}_t^+ = \{(x,y) \in V(G) : u(x,y,t) > 0\}
\end{equation}
\begin{equation}
\mathcal{U}_t^- = \{(x,y) \in V(G) : u(x,y,t) < 0\}
\end{equation}
and
\begin{equation}
\mathcal{U}_t^0 = \{(x,y) \in V(G) : u(x,y,t) = 0\}
\end{equation}
$\mathcal{U}_t^+$, $\mathcal{U}_t^-$ and $\mathcal{U}_t^0$ are the set of points that are outside, inside and on the curve at time $t$.
We also define the set of \emph{active points} at time $t$
\begin{equation}
\mathcal{U}_t^\ast = \partial U_t^+ \cup \partial U_t^- \cup \mathcal{U}_t^0
\label{eq:active_points}
\end{equation}
where with $\partial S$ we denote the boundary of the set $S \subseteq V(G)$ that can be calculated as
\begin{equation}
\partial S = \{v \in S : \exists v' \in V\backslash S \text{ s.t. } v' \sim v\}
\end{equation}
The set of active points represent the points that are within ``unit'' distance from the current position of the active contour. In \figurename{\ref{fig:graph_active_points}} we depict visually how the active points are computed.

\begin{figure}[!t]
\centering
\includegraphics[width=.5\columnwidth]{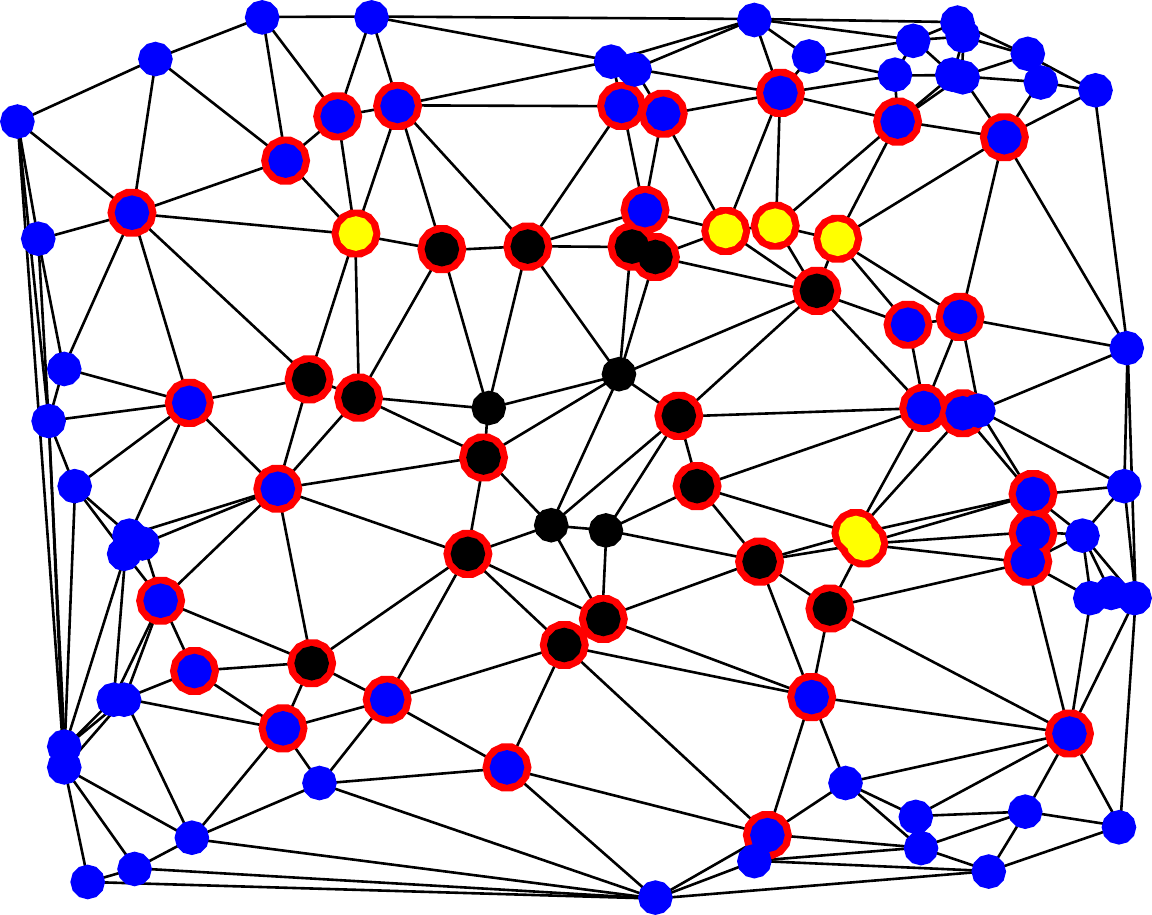}
\caption{Illustration of active points for a Delaunay graph with 100 vertices. Blue nodes: Vertices outside the contour. Black nodes: Vertices inside the contour. Yellow nodes: Vertices on the contour. Nodes with red boundary: Active points }
\label{fig:graph_active_points}
\end{figure}

Then ---up to a positive scaling factor--- we approximate $\delta_\epsilon(u)$ on $G$ at time $t$ with
\begin{equation}
\delta^G_t(v) \sim \begin{cases}
1, &\text{ if } v \in \mathcal{U}_t^\ast\\
0, &\text{ elsewhere}
\end{cases}
\end{equation}
Using the above definition, the evolution equation for \eqref{eq:levelset_general_constrained} can be approximated as
\begin{equation}
\mathbf{\tilde{A}}(\mathbf{c})\dot{\mathbf{c}} =\mathbf{\tilde{b}}(\mathbf{c})
\label{eq:non_linear_ode_constrained}
\end{equation}
where $\mathbf{\tilde{A}}(\mathbf{c})$ is a $N \times N$ matrix with
\begin{equation*}
\mathrm{\tilde{A}}_{ij}(\mathbf{c}) = \begin{cases}
\mathrm{{A}}_{ij}(\mathbf{c}), &\text{ if } v_i \in \mathcal{U}_t^\ast\\
1, &\text{ if } v_i \notin \mathcal{U}_t^\ast \text{ and } j=i\\
0, &\text{ elsewhere}
\end{cases}
\end{equation*}
and $\mathbf{\tilde{b}}(\mathbf{c})$ a $N$-dimensional vector with
\begin{equation}
\tilde{b}_i(\mathbf{c}) = \begin{cases}
b_i(\mathbf{c}), &\text{ if } v_i \in \mathcal{U}_t^\ast\\
0, &\text{ elsewhere}
\end{cases}
\end{equation}
It is easy to observe that for all non-active vertices $v_i$, $\dot{c}_i = 0$ and thus the curve evolution is indeed constrained at the subset $\mathcal{U}_t^\ast$.

The above formulation will help us reduce significantly the computational complexity of each time step. Consider the explicit Euler approximation of \eqref{eq:non_linear_ode_constrained}
\begin{equation}
\mathbf{\tilde{A}(\mathbf{c}_{k})}\left(\mathbf{c}_{k+1} - \mathbf{c}_{k}\right) = \Delta t_k\mathbf{\tilde{b}(\mathbf{c}_k)}
\end{equation}
For all vertices $v_i$ that are not in the set of active points $(c_i)_{k+1} - (c_i)_{k} = 0$. So, for all $v_i \in \mathcal{U}_t^\ast$, $(c_i)_{k+1}$ does not depend on the values $(c_j)_{k}$, where $v_j \notin \mathcal{U}_{t_k}^\ast$. Thus the corresponding values $\mathrm{A}_{ij}(\mathbf{c}_k)$ have no influence on the solution of the linear system, so they can be set to $0$.

Let $\mathbf{A}_{k}^\ast$ be a $|\mathcal{U}_{t_k}^\ast| \times |\mathcal{U}_{t_k}^\ast|$ submatrix of $\mathbf{\tilde{A}}(\mathbf{c_k})$ that is constructed by keeping only the elements $\mathrm{A}_{ij}$ with $v_i$ and $v_j \in \mathcal{U}_{t_k}^\ast$. Similarly $\mathbf{b}^\ast_k$ is a $|\mathcal{U}_t^\ast|$-dimensional vector that is derived from $\mathbf{b}(\mathbf{c}_k)$ by discarding all the elements $v_i$ with $v_i \notin \mathcal{U}_t^\ast$.
If we set $\Delta\mathbf{c}_k = \mathbf{c}_{k+1} - \mathbf{c}_k$ and $\Delta\mathbf{c}_k^\ast = \mathbf{I}_{k}^\ast\Delta\mathbf{c}_k$, 
where $\mathbf{I}_k$ is the $|\mathcal{U}_{t_k}^\ast| \times N$ projection matrix with $\mathrm{I}_{ii} = 1$ and $\mathrm{I}_{ij} = 0$ if $i \neq j$
then we get the equivalent linear system
\begin{equation}
\mathbf{A}_{k}^\ast \Delta\mathbf{c}_k^\ast = \Delta t_k\mathbf{b}^\ast_k
\label{eq:constrained_linear_system}
\end{equation}
After computing $\Delta\mathbf{c}_k$ from \eqref{eq:constrained_linear_system} the update rule for the coefficients $\mathbf{c}$ in matrix notation becomes
\begin{equation}
\mathbf{c}_{k+1} = \mathbf{c}_{k} + \mathbf{I}_k^{\mathrm{T}}\Delta\mathbf{c}_k^\ast
\label{eq:update_rule_constrained}
\end{equation}

To ensure the stability of the numerical method and make it possible to use large time steps, we normalize the levelset function after each time step. Specifically, $\bar{u}(x,y,t)$ and consequently the vector of coefficients $\mathbf{c}_k$ is saturated outside the interval $[-r,r]$. A typical choice of $r$ is $r = \max |u_0(x,y)|$ or equivalently $r = \max|\mathbf{c}_0|$. With the use of normalization we can choose a time step $\Delta t_k \sim 1/N$, which is significantly larger than the time step used for the solution of \eqref{eq:levelset_general}. A similar approach cannot be adopted for the solution of the full levelset equation because in each time step the evolution domain covers the whole graph and although the levelset function will be bounded, large time steps will produce noisy artifacts.

\begin{algorithm}
\caption{Constrained Curve Evolution on Graphs}
\label{algo:cceg}
\begin{algorithmic}
\State Initialize $\mathbf{c}_0$ from the initial condition using \eqref{eq:c_initial_cond}
\State Set threshold $r = \max{|\mathbf{c}_0|}$
\For{$k = 0$ until convergence}
\State Calculate the active points $\mathcal{U}^\ast_{t_k} = \partial U_{t_k}^+ \cup \partial U_{t_k}^- \cup \mathcal{U}_{t_k}^0$
\State Calculate $\mathbf{A}^\ast_k$ and $\mathbf{b}^\ast_k$
\State Solve the reduced linear system $\mathbf{A}_{k}^\ast \Delta\mathbf{c}_k^\ast = \Delta t_k\mathbf{b}^\ast_k$
\State Compute $\mathbf{c}_{k+1} = \min\left(\max\left(\mathbf{c}_{k} + \mathbf{I}_k^{\mathrm{T}}\Delta\mathbf{c}_k^\ast,-r\right),r\right)$ 
\EndFor
\end{algorithmic}
\end{algorithm}

\subsection{Estimation of Algorithmic Complexity}
As in the previous section, we will provide an estimate for the computational complexity of the proposed algorithm for the case of Delaunay graphs that correspond to images. We expect that under reasonable assumptions, similar results hold for more general Delaunay graphs.

For a $n \times n$ image with $N=n^2$ pixels, each of $\mathcal{U}^+_{t_k}$, $\mathcal{U}^-_{t_k}$ and $\mathcal{U}^0_{t_k}$ is the union of a number of 1D curves. Generally, typical 1D curves in a $n \times n$ grid contain $O(n)$ pixels. Thus, in almost all practical cases $\mathcal{U}^\ast_{t_k}$ will contain $O(n) = O(\sqrt{N})$ pixels.
Since the number of edges in the planar Delaunay graph is bounded by $3N-6$ a naive calculation of $\mathcal{U}^\ast_{t_k}$ using the definitions of its components requires $O(N)$ operations. Additionally, it is easy to verify that given the set of active points, the elements of $\mathbf{A}^\ast_k$ and $\mathbf{b}^\ast_k$ can be computed in $O(N)$ time.

Since $\mathbf{{A}(\mathbf{c}_k)}$ is a sparse band matrix, $\mathbf{A^\ast_k}$ will also be a band matrix, but its band length often can be $O(|\mathcal{U}^\ast_{t_k}|)$. However, using the Reverse Cuthill-McKee algorithm \cite{george_liu} which can be implemented in $O(|\mathcal{U}^\ast_{t_k}|)$ time we can obtain a permuted matrix with a band length of $O(\sqrt{|\mathcal{U}^\ast_{t_k}|})$ on average. Thus the solution of the constrained linear system is expected to require $O(n(\sqrt{n})^2) = O(n^2) = O(N)$ operations.

Thus, we have shown that on average, each time step of the algorithm requires $O(N)$ operations. If we compare this result with the number of operations required for the full curve evolution, we can see that the constrained curve evolution algorithm is faster by an order of magnitude. Its linear complexity with respect to the number of graph vertices makes it feasible to be used in practical applications.

\subsection{Fast Solution of Active Contour Models}
The results from the previous section hint towards a fast implementation of general active contour models that can take the form of \eqref{eq:levelset_general}. The levelset approach enabled us to handle topological changes in the curve in a solid and efficient manner but it introduces an extra computational burden because we have to evolve a 2D function instead of a curve. The answer to this problem is to try to limit our focus in a small area of interest near the curve --often referred as band-- and evolve the levelset function in this subset of the image. Similar methods, called narrow-band methods, were described for the case of images in \cite{adalstein, whitaker, peng_et_al}.

In this paper we will extend the previous ideas and generalize them on graphs. Instead of evaluating the curve evolution in the whole graph we will constrain the curve evolution near the curve using the method of Section \ref{section:constrained_acog}. Thus approximation is based on the assumption that the levelset function evolves ``isotropically'', thus points outside or inside the curve will not change status at least until the moment that the active contour reaches them. The above assumption is valid for the majority of the active contour methods, such as the simple Erosion/Dilation and the Geometric and Geodesic Active Contours. 

Generally, the set of active points is comprised of vertices that are either outside, inside or on the contour. For certain active contour models such as those mentioned above where the curves either expand or shrink, the set of active points can be further reduced. In the case of expanding contours, if a point is inside the contour then it will always remain inside it, so it suffices to choose $\mathcal{U}_t^\ast = \mathcal{U}_t^+$. Following the same logic, for shrinking contours $\mathcal{U}_t^\ast = \mathcal{U}_t^-$. Generally, this observation allows us to reduce the size of the $A^\ast$ by a factor of 2.

\section{Geodesic Active Regions}\label{section:gar}
The classic Geodesic Active Contours method tries to detect objects of interest using only edge information. This poses a serious limitation to the range of images where it can be used. Whilst it gives excellent results when applied in images where the background is smooth and the objects are easily distinguishable, we cannot expect to apply it successfully in complex real world images in which object boundaries are not always well-defined. Thus we need to incorporate more features in our model, such as the foreground and background color distributions that will help us separate the objects of interest from the background. Similar efforts that try to include region information in Active Contour Methods were presented in \cite{zhu_yuille} and \cite{paragios}.

In this section, we will discuss a modification of the GAC algorithm resembling the Geodesic Active Regions \cite{paragios} that accounts for these issues. We have to note that GACs belong to the class of semi-supervised methods; the user only needs to specify a curve that surrounds the object. Here we propose a supervised method in which the user needs to provide additional seed points inside the object(s) he wants to detect. The concept of our method is similar to that of GrabCut \cite{rother_et_al}; iteratively segment and reestimate region statistics until convergence.

Given a set of points in the plane, we form the corresponding Delaunay graph. First, the user specifies a region of interest that contains the object that is to be detected. The points outside this region are assumed to belong to the background $B$. The user also specifies one or more closed curves inside the objects of interest. The points that are inside the regions enclosed by these curves belong to the foreground $F$. Optionally, the user can specify an additional set of points that wants to be labeled as background. The remaining points $U$ are temporarily classified as undefined and our purpose is to determine their labeling.

Using the information provided above we will try to model the color distributions of the foreground and background. We build two 3-D histograms, one for the foreground and one for the background and train the 2 models using $F$ and $B$ respectively. Histograms are used instead of Gaussian Mixture Models because they tend to be faster and in our case yield slightly better results.

Next, for each vertex $v$ we compute the likelihood that it belongs to the foreground and the background based on the constructed histograms. We will denote these probabilities by $H_F(v)$ and $H_B(v)$ respectively. We also define the quantity
\begin{equation}
L = \ln\frac{H_F}{H_B}
\end{equation}

If $L(v)>0$ then $H_F(v) > H_B(v)$ which in turn means that the vertex $v$ is more likely to belong to the foreground, whereas if $L(v)<0$ then $v$ is more likely to belong to the background. In order to sharpen $L$ and result in better separation of foreground and background pixels we apply soft thresholding on $L$ using a sigmoid function and obtain
\begin{equation}
\mathcal{L} = \sigma(L) = \sigma\left(\ln\frac{H_F}{H_B}\right)
\label{eq:region_l}
\end{equation}
In this paper, we chose $\sigma$ to be the common logistic function
\begin{equation}
\sigma(x) = \frac{1}{1+\exp(-\alpha x)}
\end{equation}
where $\alpha$ a sufficiently large positive constant. The role of $\sigma$ is to push faster positive values of $L$ to $1$ and negative values to $0$ producing a near-binary result which will be useful for the curve evolution.

In the GAC model the initial contour can only shrink. Our model can include both shrinking and expanding contours and enables us to combine region and edge information in a straightforward way. If we choose $g_I$ to be the usual edge stopping function defined in \eqref{eq:edge_stopping_function} our modified Geodesic Active Contour model becomes
\begin{equation}
\frac{\partial u}{\partial t} = \left(\diverg{\mathcal{P}\frac{\nabla u}{\norm{\nabla u}}}+\beta g_I \mathcal{P}\right)\norm{\nabla u}
\label{eq:region_gac}
\end{equation}
where 
\begin{equation}
\mathcal{P} = 
\begin{cases}
\mathcal{L}, &\text{ for expanding contours}\\
1-\mathcal{L}, &\text{ for shrinking contours}\\
\end{cases}
\label{eq:statistical_curve_evolution}
\end{equation}
The use of $\mathcal{P}$ is to stop the contour evolution at the desired regions. In the case of expanding contours, we expand the curve until it reaches the background, where $\mathcal{P} = \mathcal{L} \approx 0$. Conversely, in the case of shrinking contours, we shrink the curve until it reaches the foreground, where $\mathcal{P} = 1 -\mathcal{L} \approx 0$. Concurrently, the balloon force is also controlled by $g_I$ which vanishes at strong edges.

The core principle our algorithm is to alternate between expanding and shrinking curves to successfully extract the objects of interest. In the first iteration of the algorithm we train the color models using the initial seed points and expand the object seed boundaries according to \eqref{eq:statistical_curve_evolution} until convergence. We then obtain new estimates for the foreground and background and use them to retrain the color histograms. Vertices in the undefined region are not used for model training. The main reason for this choice is that it could result in a biased estimation of the statistics. Next we shrink the curves using the new estimations in order to make sure that points that were falsely classified as foreground in the first iteration are discarded. This process is repeated until convergence, i.e. alternating between expanding and shrinking contours and retraining the models after each step until the segmentation result between two successive iterations is sufficiently similar.

The last step of the algorithm will be to provide a refinement of the segmentation result we obtained. We will try to alter the position of the contour in order to better capture the object boundaries. This process is different from the various matting techniques (see for example Bayesian matting in \cite{chuang_et_al}) because our goal is to produce only a binary result and not continuous alpha-values. To do this we limit our focus in a small band $B$ around the contour. For each pixel in this area we compute the likelihoods $H_F$ and $H_B$ as we defined them previously, but now the foreground and background model are trained using only vertices that are near the region boundaries. Then for each vertex $v \in B$, if $H_F(v) \geq H_B(v)$ $p$ is assigned to the foreground, else it is marked as background. However, since this process can create a noisy output resembling salt-and-pepper noise, we apply median filtering to the result obtained.

\begin{algorithm}
\caption{Geodesic Active Regions}
\label{algo:rbgac}
\begin{algorithmic}
\State Initialize $F$, $B$ and $U$ and initial contour from user labeling.
\While{not converged}
\State Compute region statistics $\mathcal{L}$ using \eqref{eq:region_l}
\State Expand contours using \eqref{eq:region_gac} with $\mathcal{P} = \mathcal{L}$
\State Compute region statistics $\mathcal{L}$ using \eqref{eq:region_l}
\State Expand contours using \eqref{eq:region_gac} with $\mathcal{P} = 1-\mathcal{L}$
\EndWhile
\State Refine segmentation near the curve
\end{algorithmic}
\end{algorithm}

\section{Applications to Segmentation}\label{section:experimental}
In this section we will provide experimental evidence for the effectiveness of the proposed framework. Our analysis will focus on segmentation applications, which is one of the main areas that active contours have traditionally been used in. We will use three different active contour models as representatives; the Chan-Vese ACWE \cite{chan_vese}, the GAC \cite{cas_kim_sap_97} and the Geodesic Active Regions (GAR) variant that is proposed in this paper. Our framework will be tested on Delaunay graphs formed from regular grids as well as on more general Delaunay graphs. For the case of images where the number of graph vertices is of the order of $10^5$, the evolution equation of the GAR was solved using the Constrained Contour Evolution framework described in Section \ref{section:constrained_acog}.

\subsection{Image Segmentation}
In this part, we will present how our framework can be used to solve active contour models for image segmentation. As an example, in \figurename{\ref{fig:example_results_chanvese}} we depict the curve evolution of the ACWE model \cite{chan_vese} for grayscale image segmentation using the proposed framework. Also, in \figurename{\ref{fig:example_results_rbgac}} we can see an application of the GAR model that is also solved with the computational framework of Section \ref{section:constrained_acog}.

\begin{figure}[!thb]

\centering
\subfloat[Initial Contour]{\includegraphics[width=.22\columnwidth]{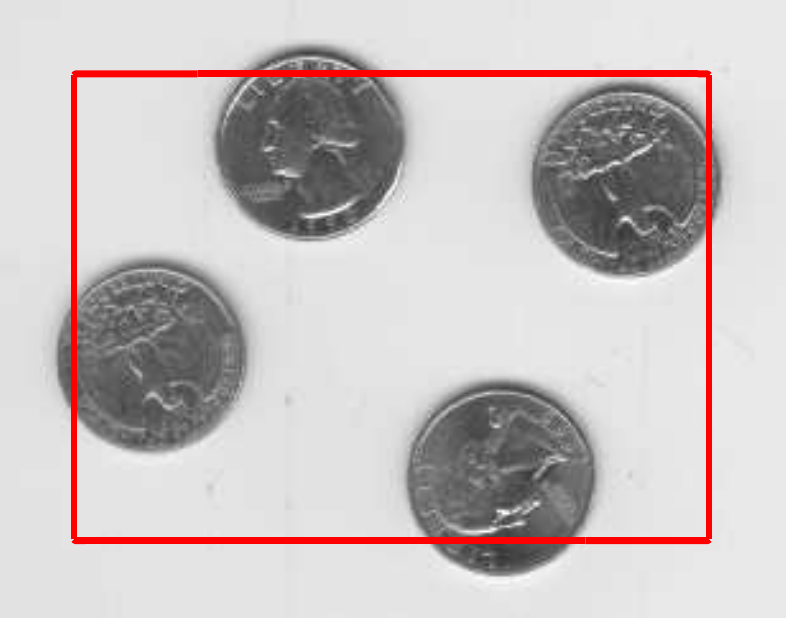}}
~
\subfloat[30 iterations]{\includegraphics[width=.22\columnwidth]{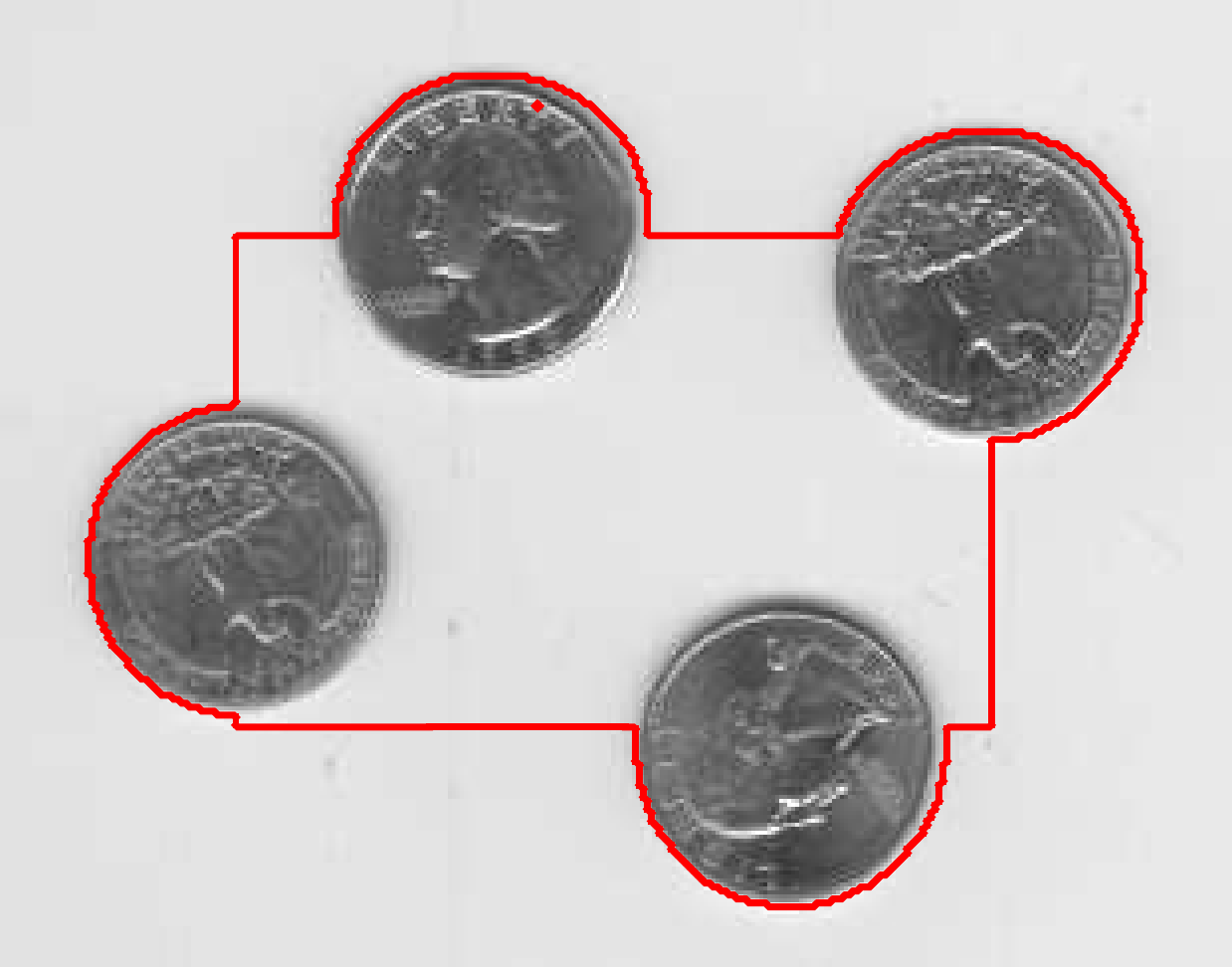}}
\hfil
\subfloat[80 iterations]{\includegraphics[width=.22\columnwidth]{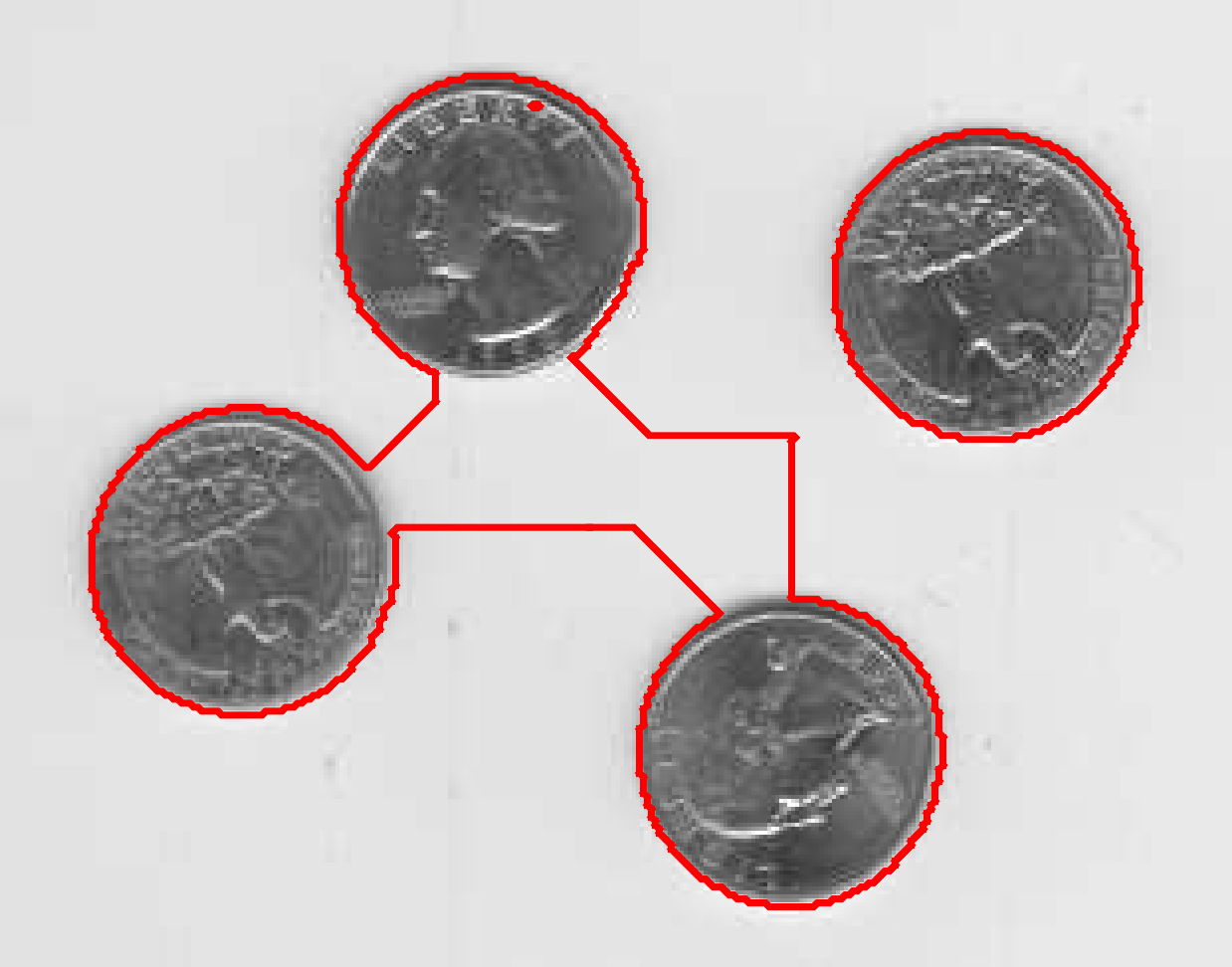}}
~
\subfloat[Converged after 107 iterations]{\includegraphics[width=.22\columnwidth]{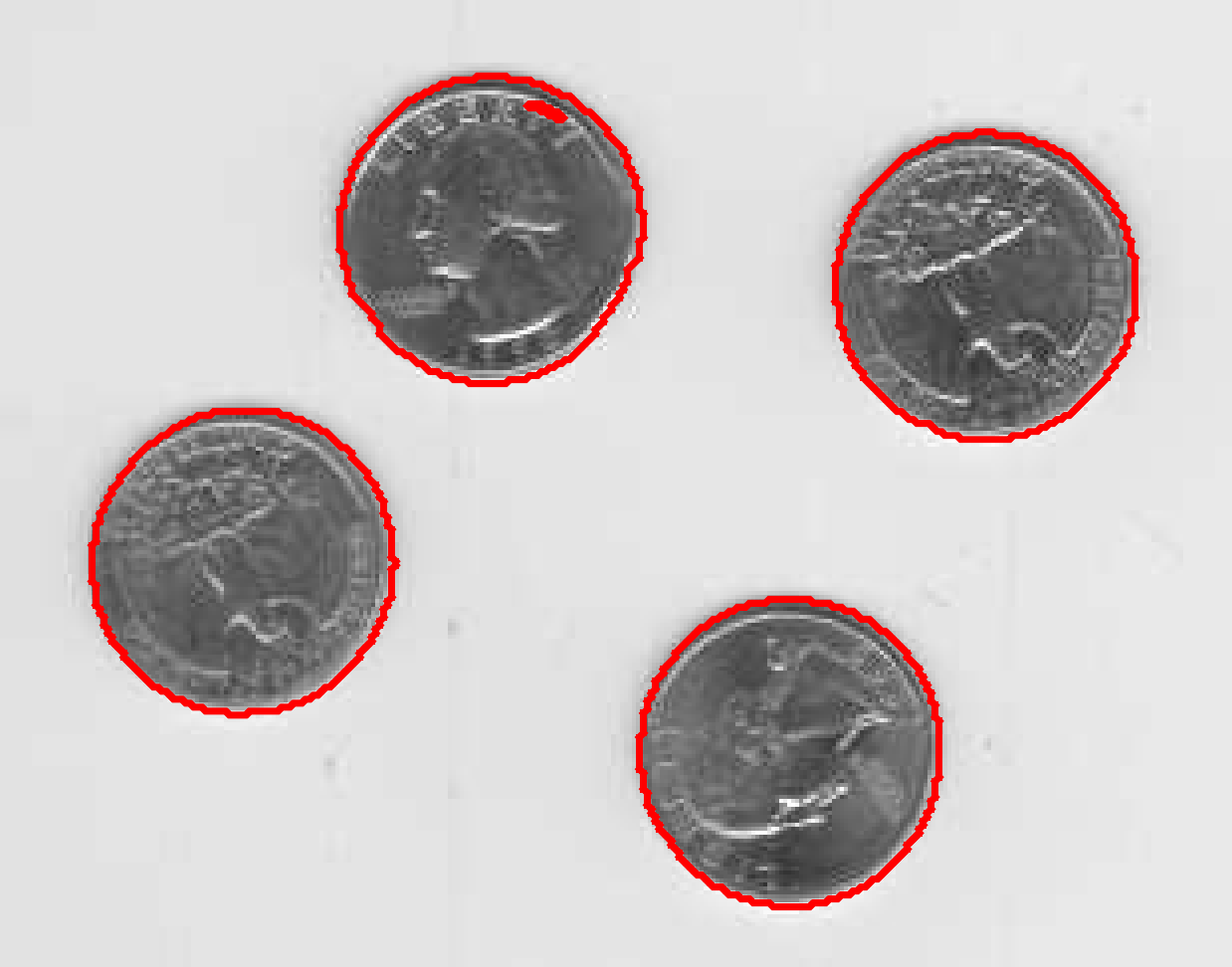}}
\hfil
\caption{Curve evolution of ACWE \cite{chan_vese} for grayscale image segmentation using the proposed Locally Constrained Curve Evolution framework.}
\label{fig:example_results_chanvese}
\end{figure}

\begin{figure}[!thb]

\centering
\subfloat{\includegraphics[width=.3\columnwidth]{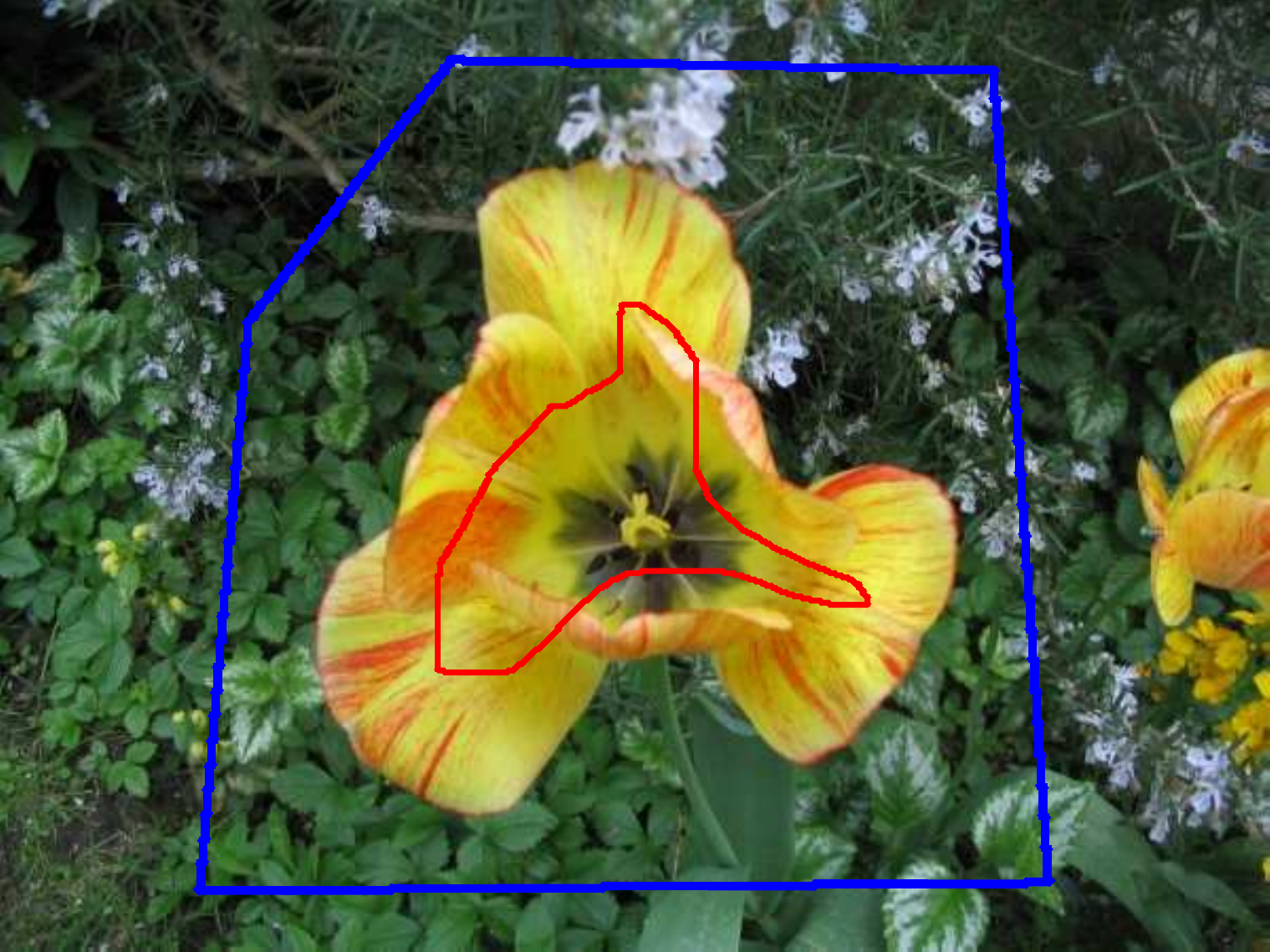}}
~
\subfloat{
{
\setlength{\fboxsep}{0pt}
\setlength{\fboxrule}{.1pt}
\fbox{\includegraphics[width=.3\columnwidth]{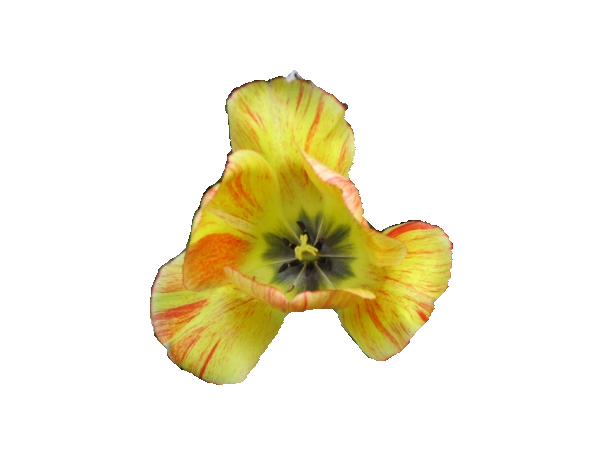}}}}
\hfil
\subfloat{\includegraphics[width=.3\columnwidth]{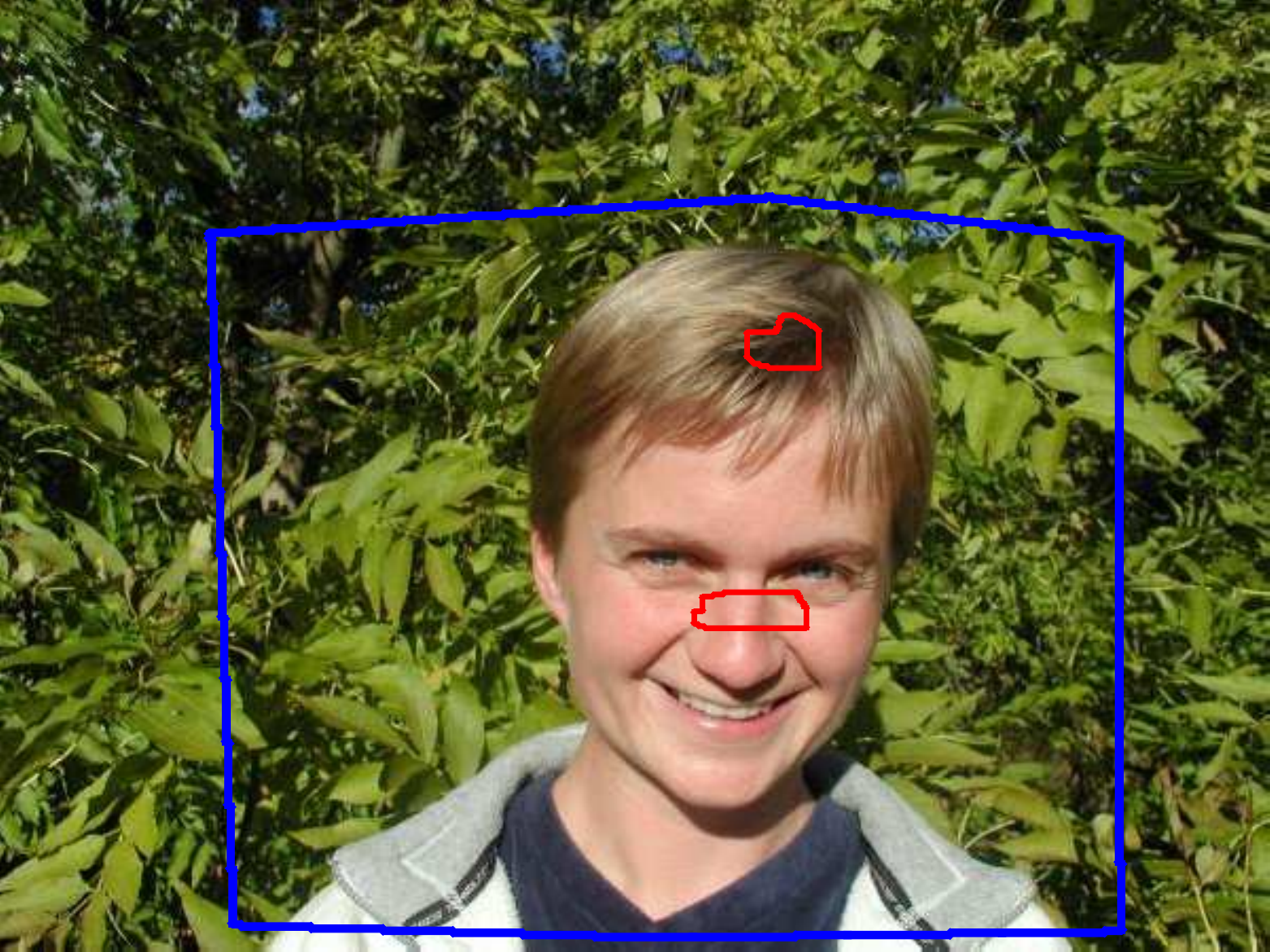}}
~
\subfloat{
{
\setlength{\fboxsep}{0pt}
\setlength{\fboxrule}{.1pt}
\fbox{\includegraphics[width=.3\columnwidth]{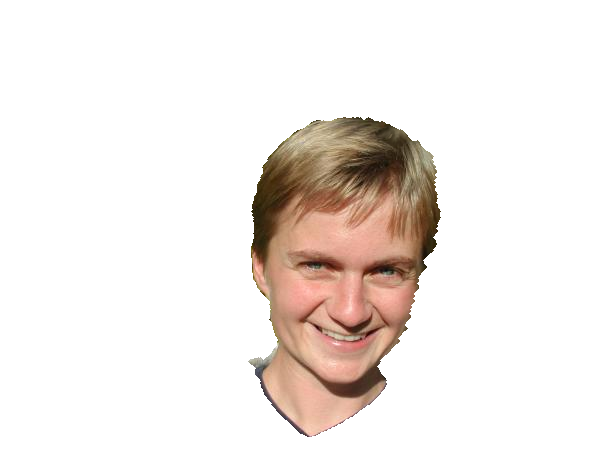}}}}
\hfil
\caption{Illustration of GAR algorithm. Left column: Images and seeds. The red curves are the initial contours and the blue polygons mark the regions of interest. Right column: Segmentation results.}
\label{fig:example_results_rbgac}
\end{figure}

Next we will evaluate the performance of the proposed GAR for image segmentation and provide comparisons against four state-of-the-art seeded image segmentation methods: GrabCut \cite{rother_et_al}, Laplacian Coordinates \cite{casaca_et_al}, Power Watersheds \cite{couprie_grady_najman_talbot} and Random Walker \cite{grady} as well as with the ACWE algorithm \cite{chan_vese}. For the benchmarks we use 2 different datasets. First, the popular GrabCut database consisting of 50 images, the ground truth segmentations and foreground and background seed locations. 20 of those images are taken from the Berkeley Segmentation Dataset \cite{martin_et_al}. The second dataset consists of 40 images from the PASCAL VOC dataset \cite{everingham}. For the purposes of this benchmark, the foreground and background seed images were obtained by eroding the respective regions in the ground truth segmentation image, producing seeds similar to those in the GrabCut dataset.

The metrics we use to compare the above methods are
\begin{itemize}
\item \textbf{Rand Index}: Measures the similarity between the segmentation and the ground truth by calculating the fraction of point pairs that are classified in the same set in the two segmentations. In our benchmarks we use the adjusted form of the Rand Index as proposed in \cite{hubert}.

\item \textbf{Intersection over Union}: It is the ratio of the intersection of the segmentation region with the ground truth divided by their union.

\item \textbf{Variation of Information} (VOI): It is a measure of the distance between a segmentation and the ground truth in terms of their entropies and their mutual information  \cite{meila}.

\item \textbf{Error rate}: It is simply the percentage of pixels that were misclassified.
\end{itemize}

All the above metrics were calculated in the unlabeled regions as provided in the trimap images. Also we exclude the pixels in the ground truth segmentations that are marked as undefined. Typically, these pixels are located near the boundaries and the image resolution is not sufficient to classify them either as foreground or as background. By comparing our benchmark results with those presented in other papers using the same publicly available implementations we noticed a slight decrease in performance. This difference is due to the fact that in the calculations of the segmentation metrics in these papers all image pixels are used, including those already marked as foreground and background. Also, in the case of the Rand Index, there are several different definitions in the bibliography and thus its exact value depends on the particular choice.

In Table {\ref{table:comparison}} we can see the benchmark results for the 5 methods on the GrabCut dataset. We can see that GrabCut outperforms the other algorithms. Our method performs very well, achieving similar results and an accuracy of over 90\% and at the same time surpasses Power Watershed and Random Walker.

Additionaly, in Table {\ref{table:comparison2}} we compare the methods on the subset of the Pascal dataset. We can see that our method outperforms all other methods except for GrabCut and that the performance margin between these 2 methods is small. Thus from the results on both datasets we can conclude that our algorithm can be used as a reliable alternative to the other widely used methods.

In \figurename{\ref{fig:comparison_results}} we can see the different segmentation results for 5 images from the GrabCut datavase and 4 images from the PASCAL dataset. This visualization helps us understand some features of each method. For example we can see that GrabCut and our method perform better when dealing with sharp objects or objects that have holes and non smooth boundaries. Laplacian Coordinates tend to produce smoother object boundaries and this why it fails in cases similar to those of the second and fourth image.

\begin{table}[!h]
\centering
\caption{Comparison of the Methods on the GrabCut Dataset}
\begin{tabular}{|c|c c c c c|}
\hline
\hline
\textbf{Method} & \textbf{RI} ($\uparrow$) & \textbf{IoU} ($\uparrow$) &\textbf{VoI} ($\downarrow$)  &\textbf{Error} ($\downarrow$)&\\    \hline
GC & {0.7861} & 0.8796 & {0.5419}& 5.869 \%&\\ \hline
LC & 0.7763 & 0.8671 & 0.5642 & 6.208 \%&\\ \hline
PW & 0.7171 & 0.8358 & 0.6768 & 7.977 \%&\\ \hline
RW & 0.7200& 0.8343 & 0.6652 & 7.854 \%&\\ \hline
CV & 0.2899& 0.4833 & 1.2244 & 24.828 \%&\\ \hline
Ours & 0.7268 & 0.8519 &  0.6704& 7.793 \%&\\ \hline
\end{tabular}
\label{table:comparison}
\end{table}

\begin{table}[!h]
\centering
\caption{Comparison of the Methods on the PASCAL Dataset}
\begin{tabular}{|c|c c c c c|}
\hline
\hline
\textbf{Method} & \textbf{RI} ($\uparrow$) & \textbf{IoU} ($\uparrow$) &\textbf{VoI} ($\downarrow$)  &\textbf{Error} ($\downarrow$)&\\    \hline
GC & 0.6939 & 0.8321 & {0.7113}& 8.945 \%&\\ \hline
LC & 0.5861 & 0.7566 & 0.8834 & 12.421 \%&\\ \hline
PW & 0.5683 & 0.7639 & 0.9345 & 12.926 \%&\\ \hline
RW & 0.3898 & 0.6872 & 1.1578 & 20.329 \%&\\ \hline
CV & 0.2045& 0.4142 & 1.2056 & 29.744 \%&\\ \hline
Ours & 0.6858 & 0.8317 &  0.7266& 9.309 \%&\\ \hline
\end{tabular}
\label{table:comparison2}
\end{table}

\begin{figure*}[!thb]
\captionsetup[subfloat]{farskip=5pt,captionskip=0pt}
\centering
\subfloat{\includegraphics[height=.1\columnwidth]{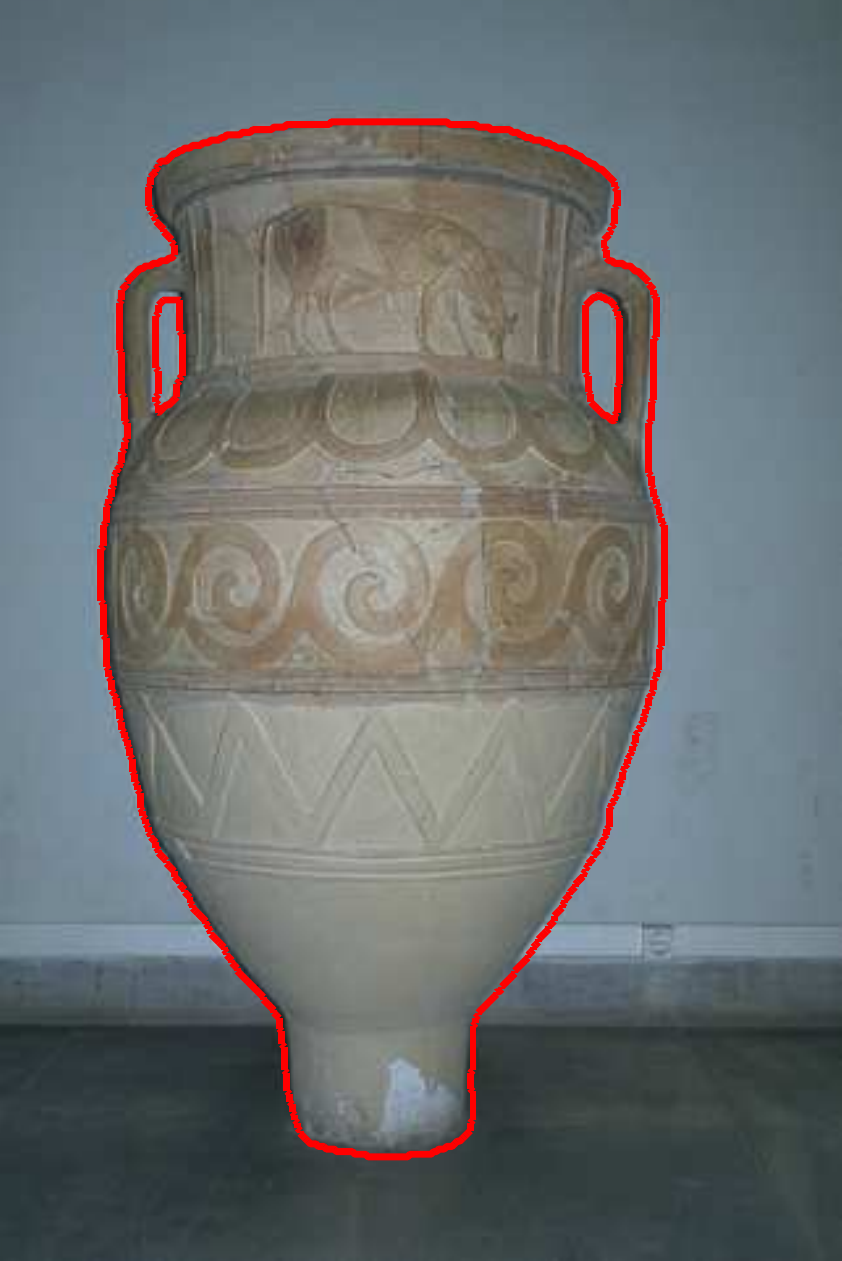}}
\,
\subfloat{\includegraphics[height=.1\columnwidth]{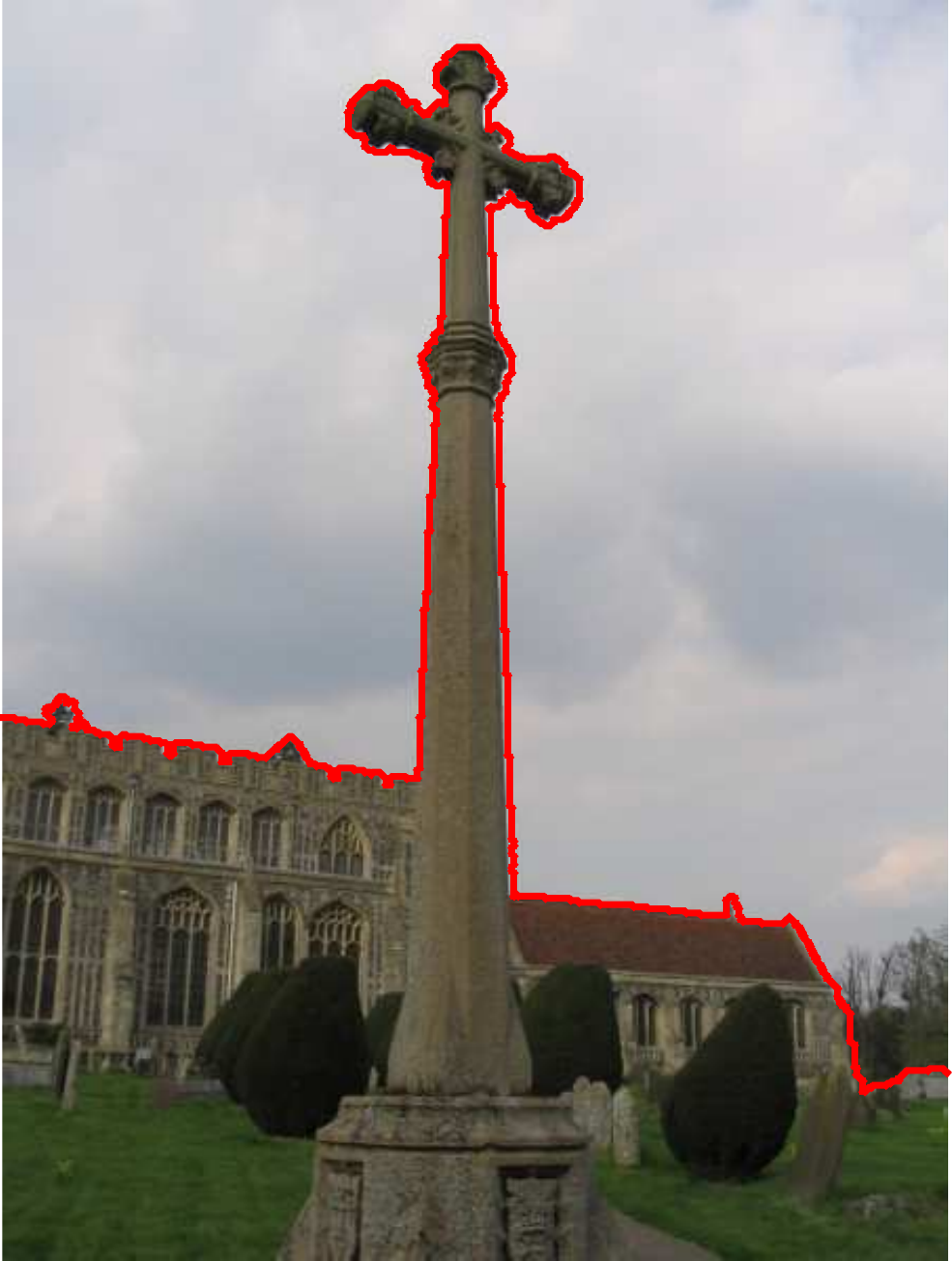}}
\,
\subfloat{\includegraphics[height=.1\columnwidth]{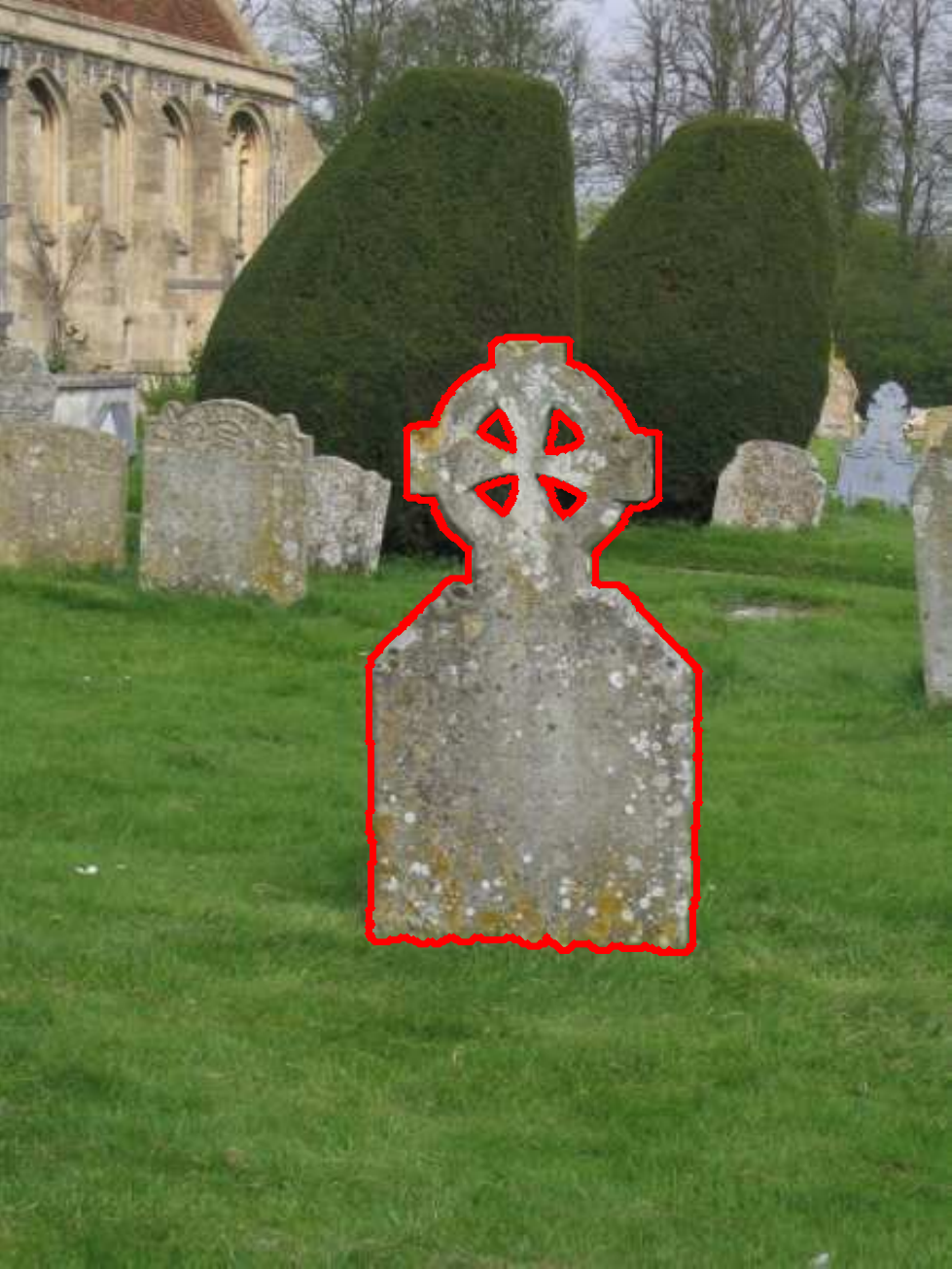}}
\,
\subfloat{\includegraphics[height=.1\columnwidth]{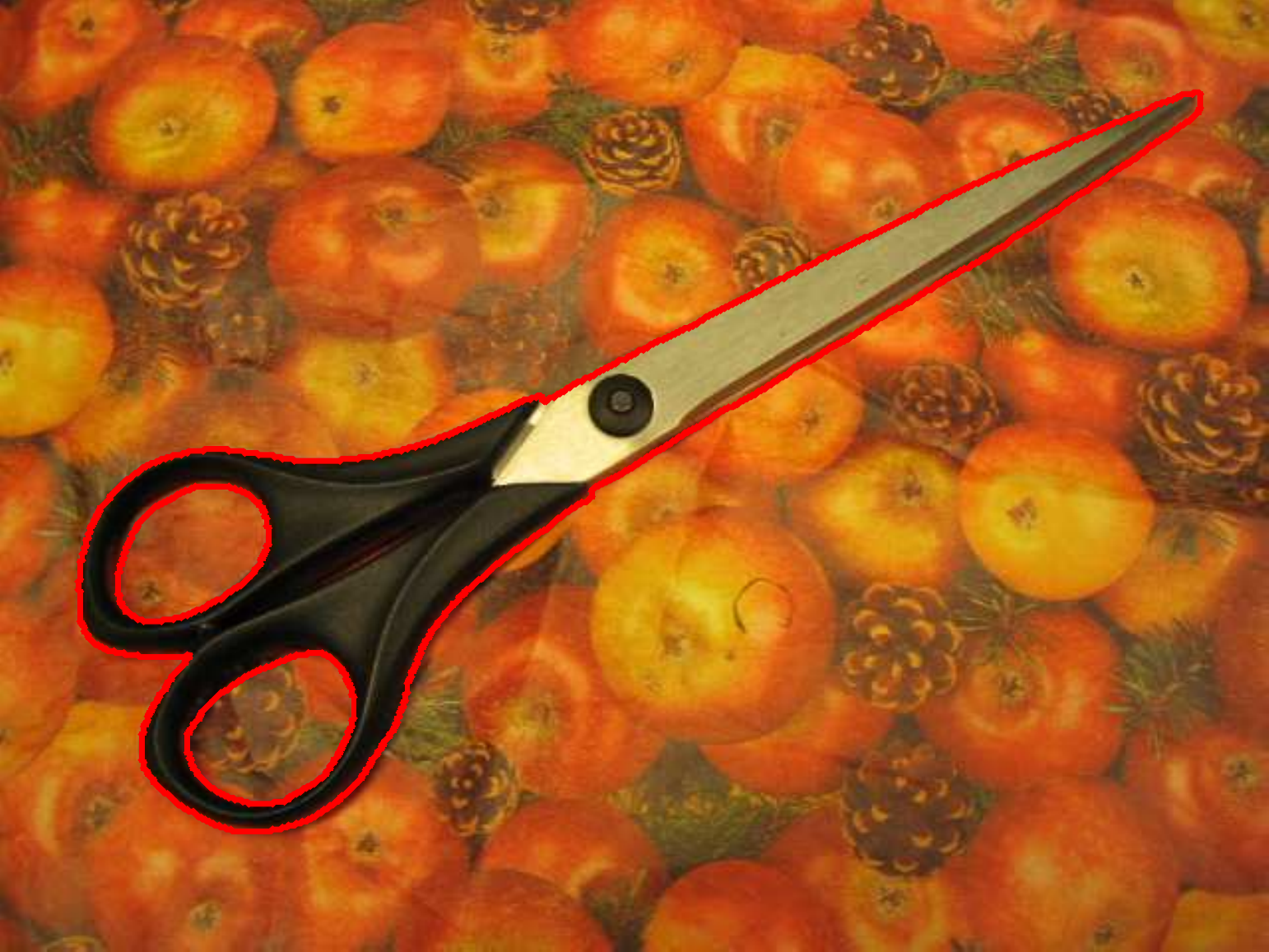}}
\,
\subfloat{\includegraphics[height=.1\columnwidth]{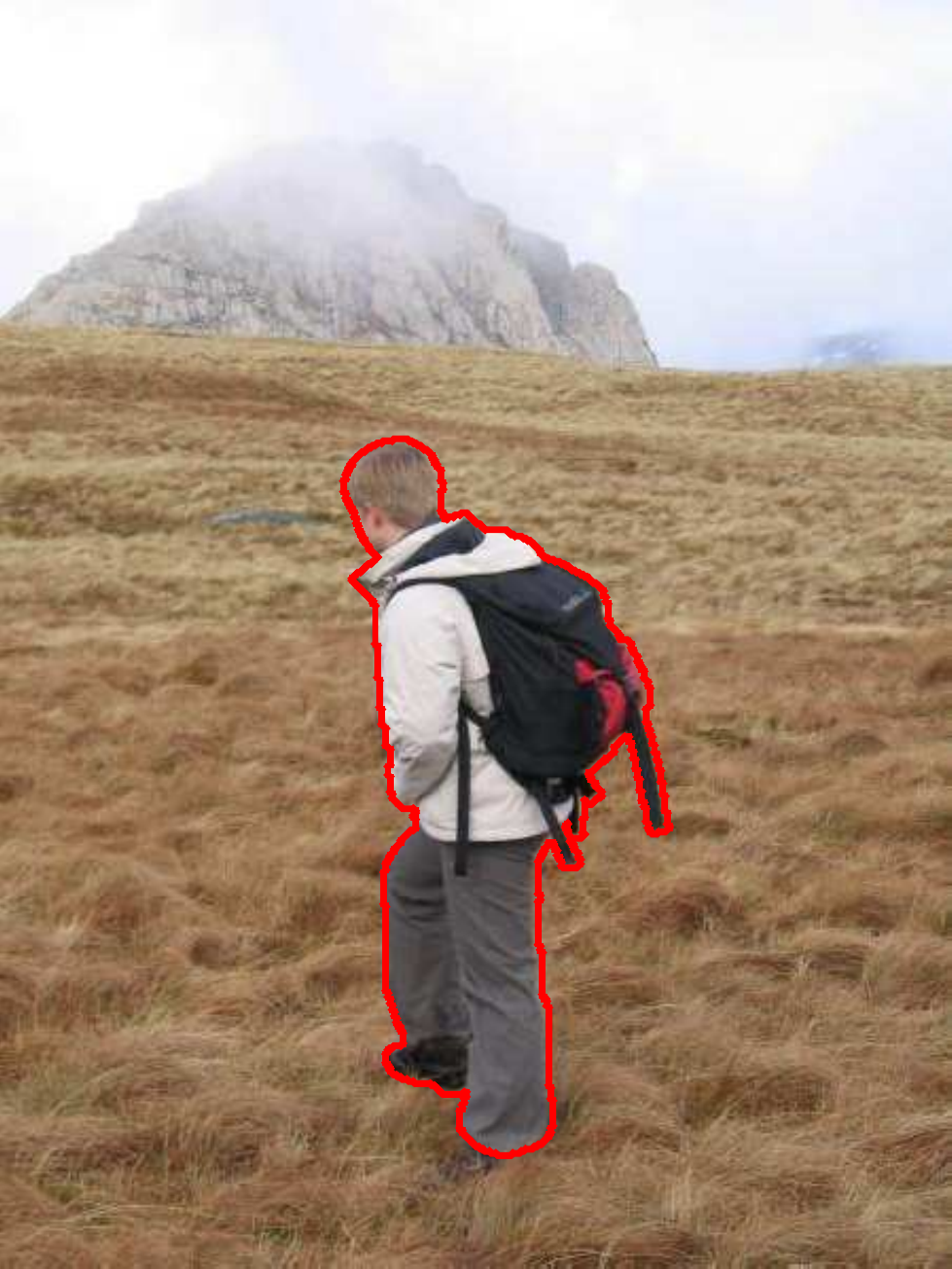}}
\,
\subfloat{\includegraphics[height=.1\columnwidth]{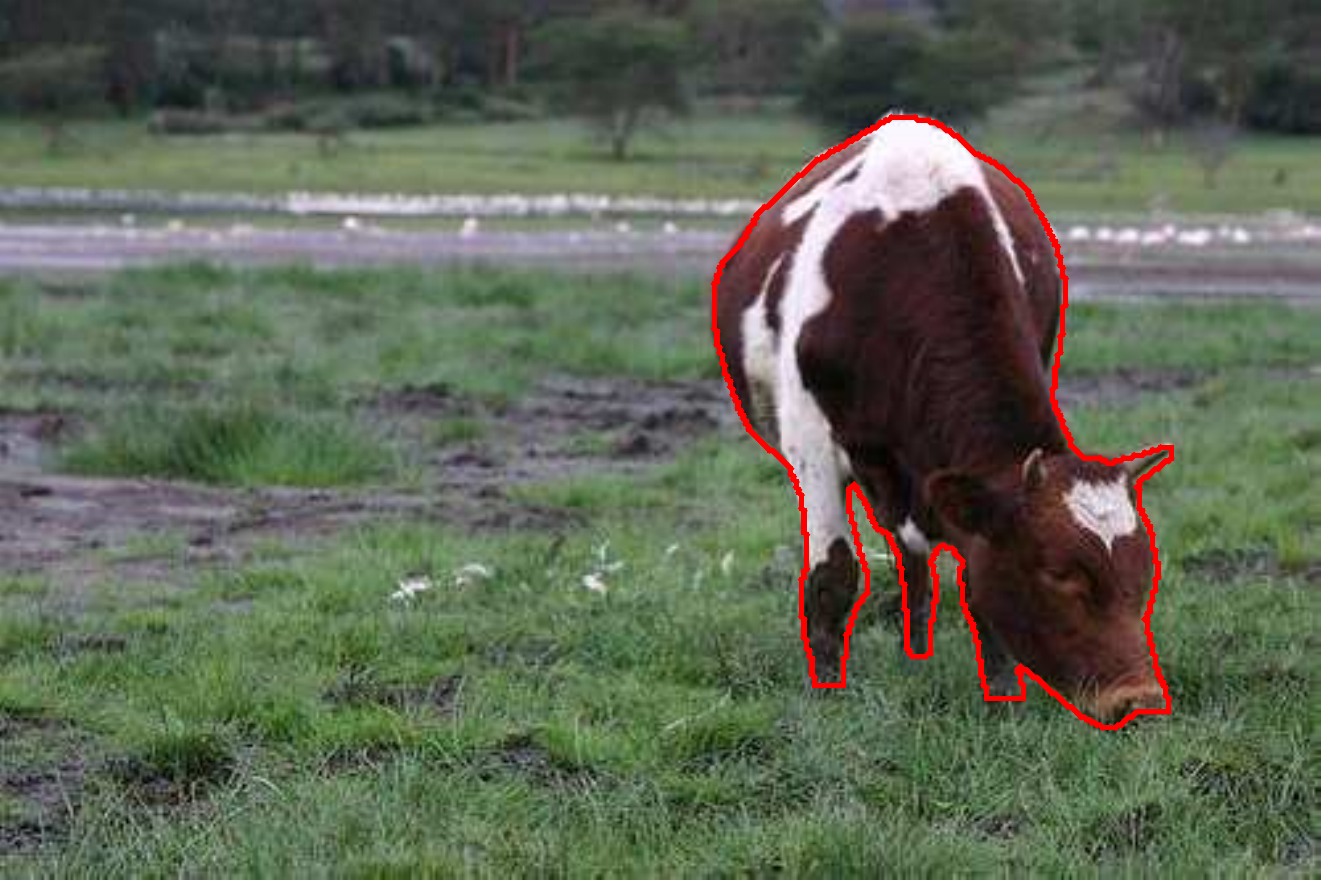}}
\,
\subfloat{\includegraphics[height=.1\columnwidth]{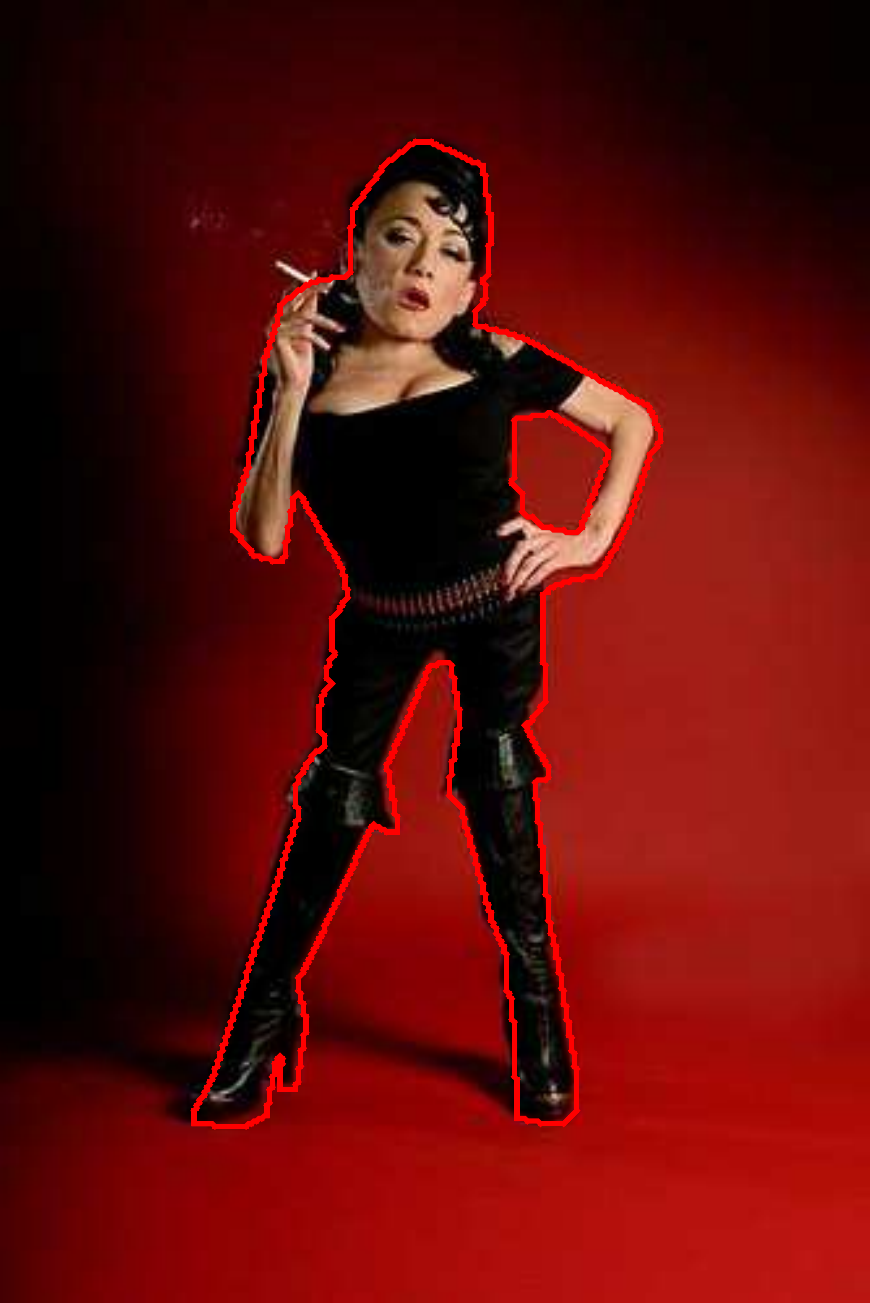}}
\,
\subfloat{\includegraphics[height=.1\columnwidth]{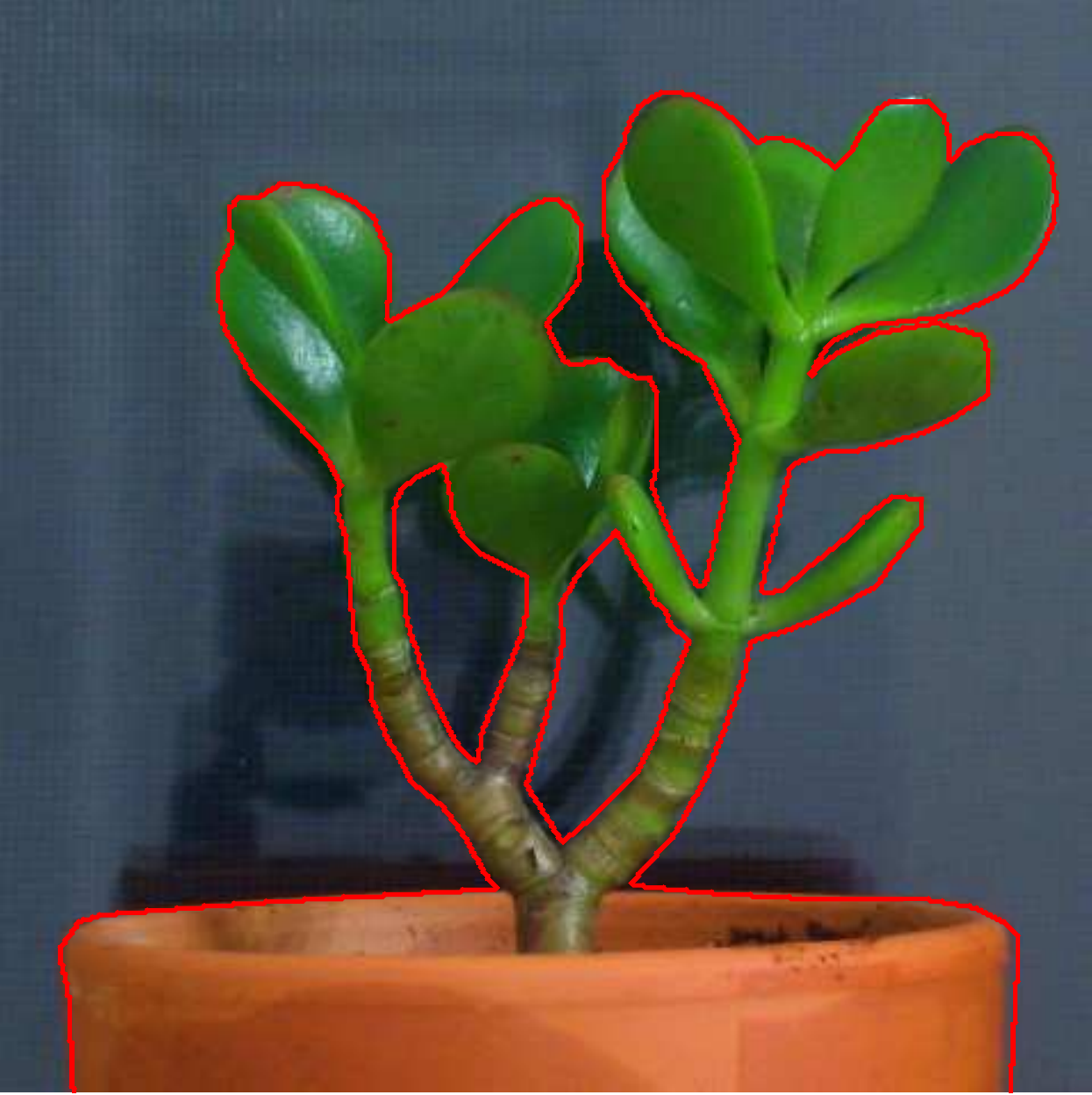}}
\,
\subfloat{\includegraphics[height=.1\columnwidth]{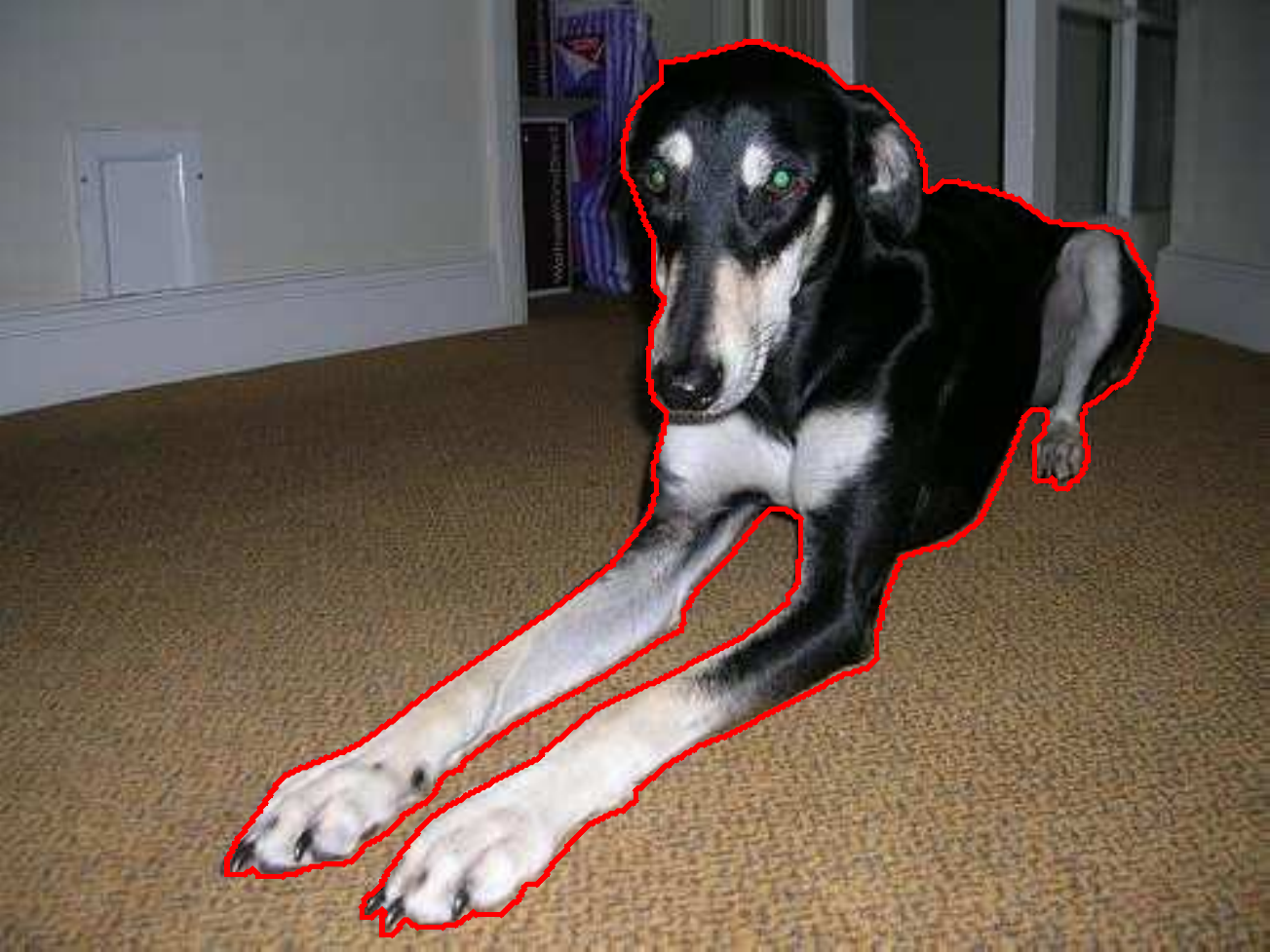}}
\hfil
\\
\subfloat{\includegraphics[height=.1\columnwidth]{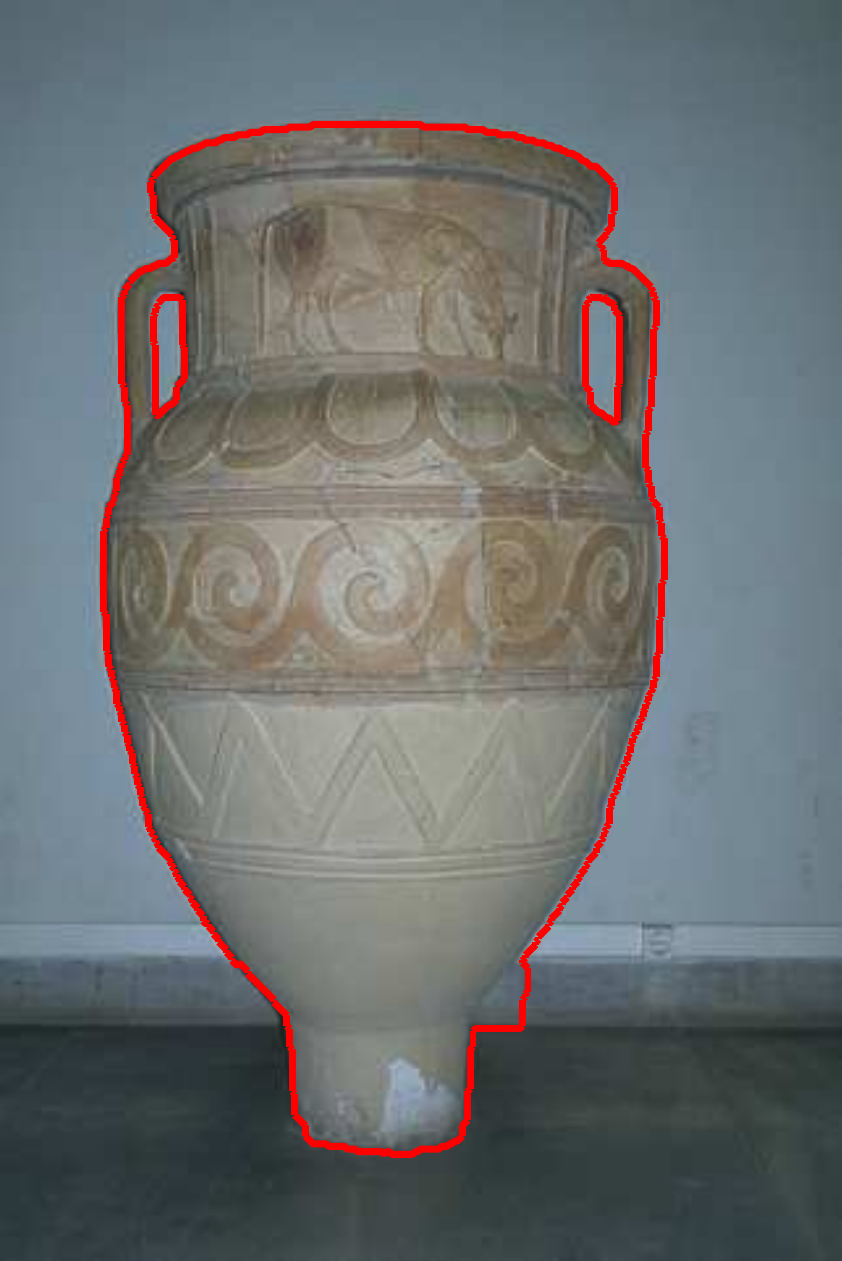}}
\,
\subfloat{\includegraphics[height=.1\columnwidth]{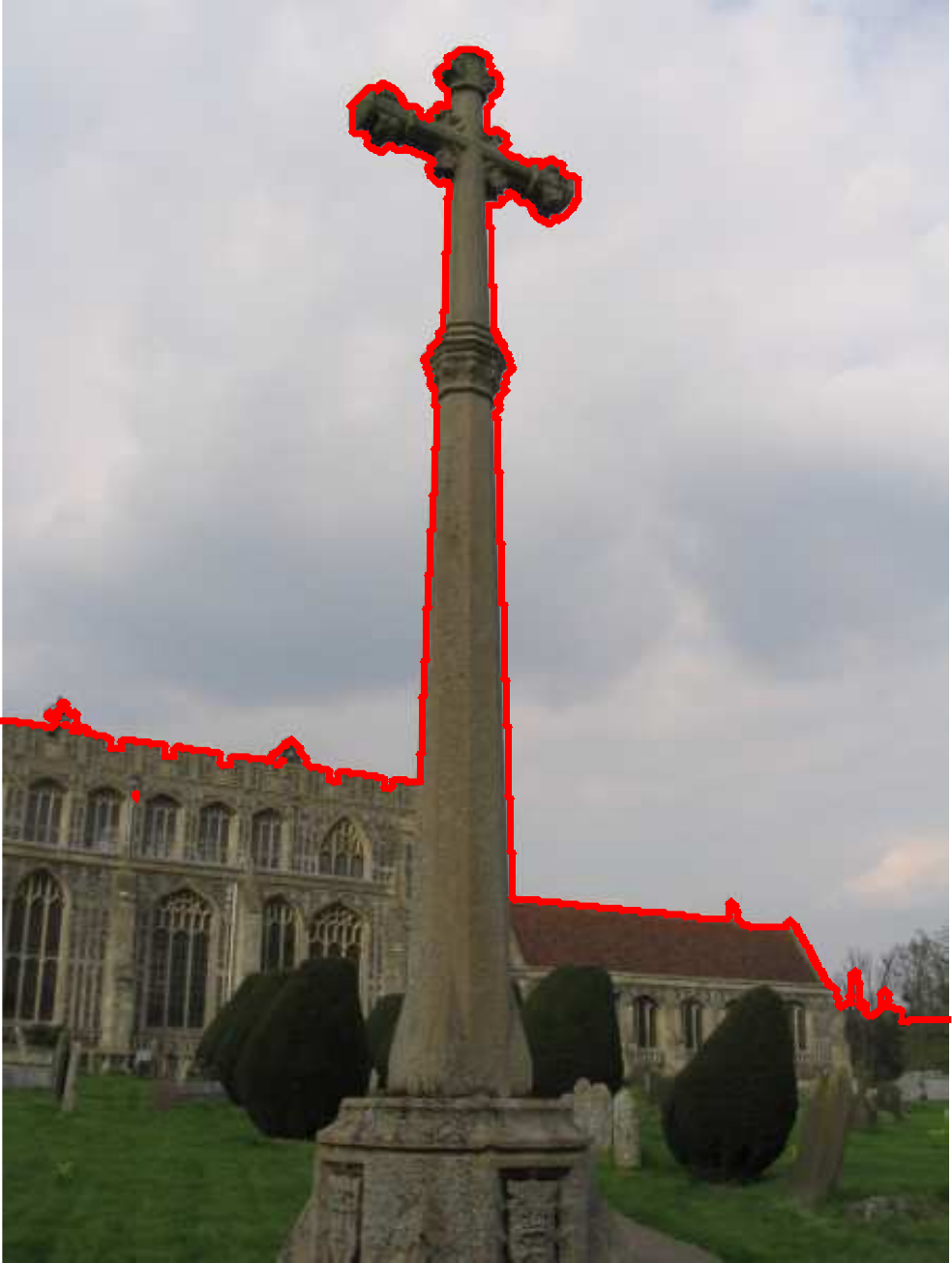}}
\,
\subfloat{\includegraphics[height=.1\columnwidth]{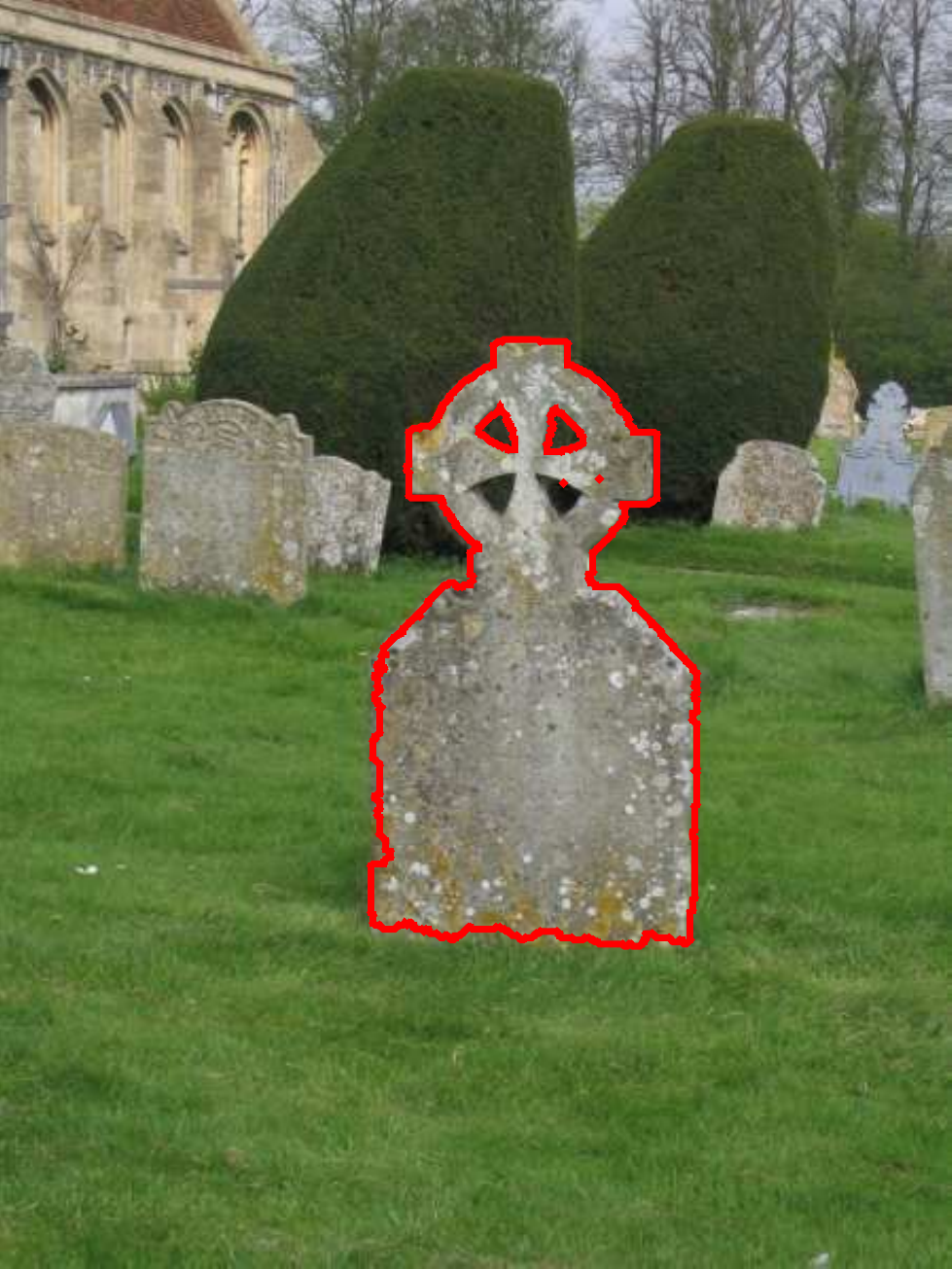}}
\,
\subfloat{\includegraphics[height=.1\columnwidth]{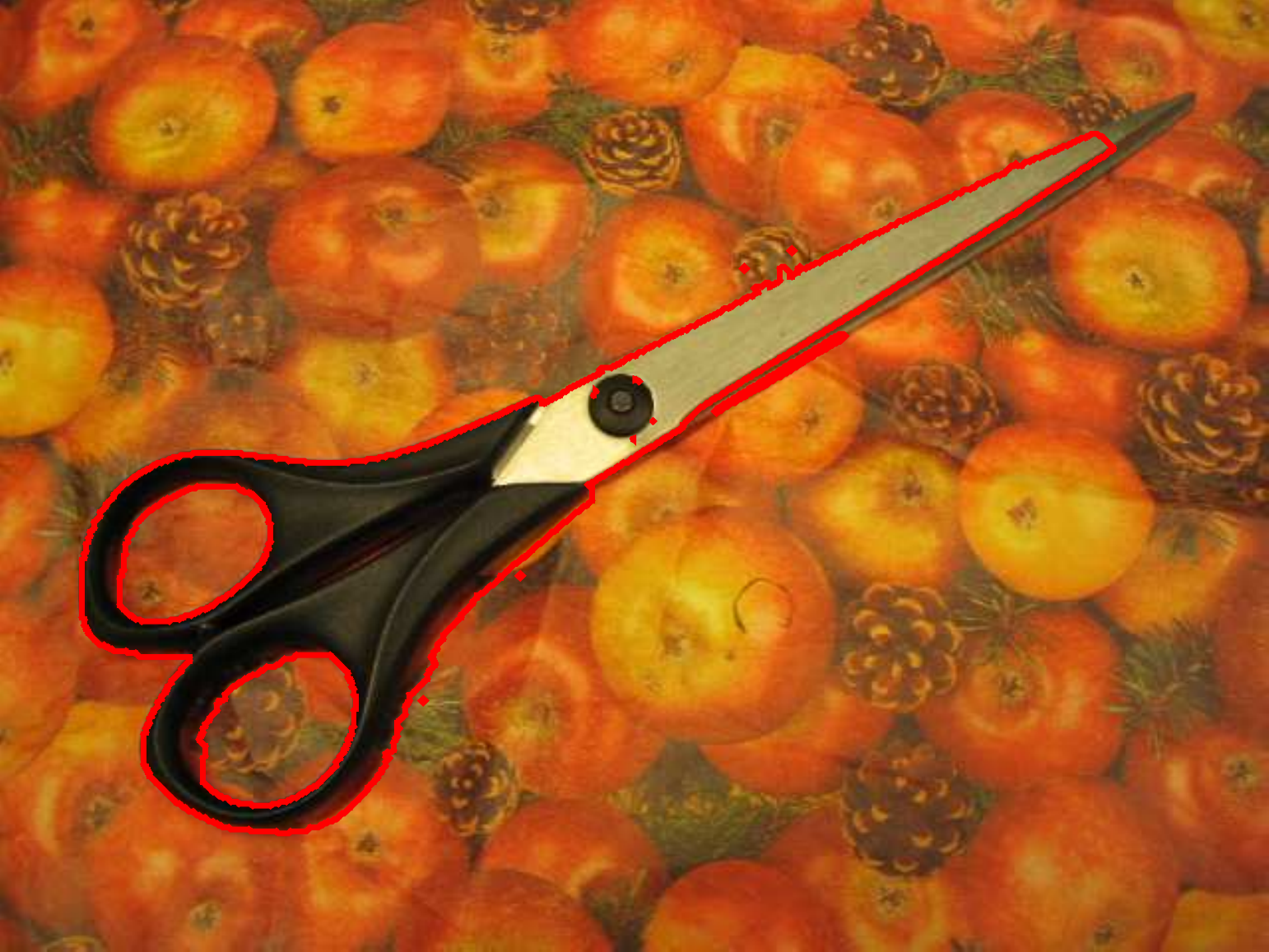}}
\,
\subfloat{\includegraphics[height=.1\columnwidth]{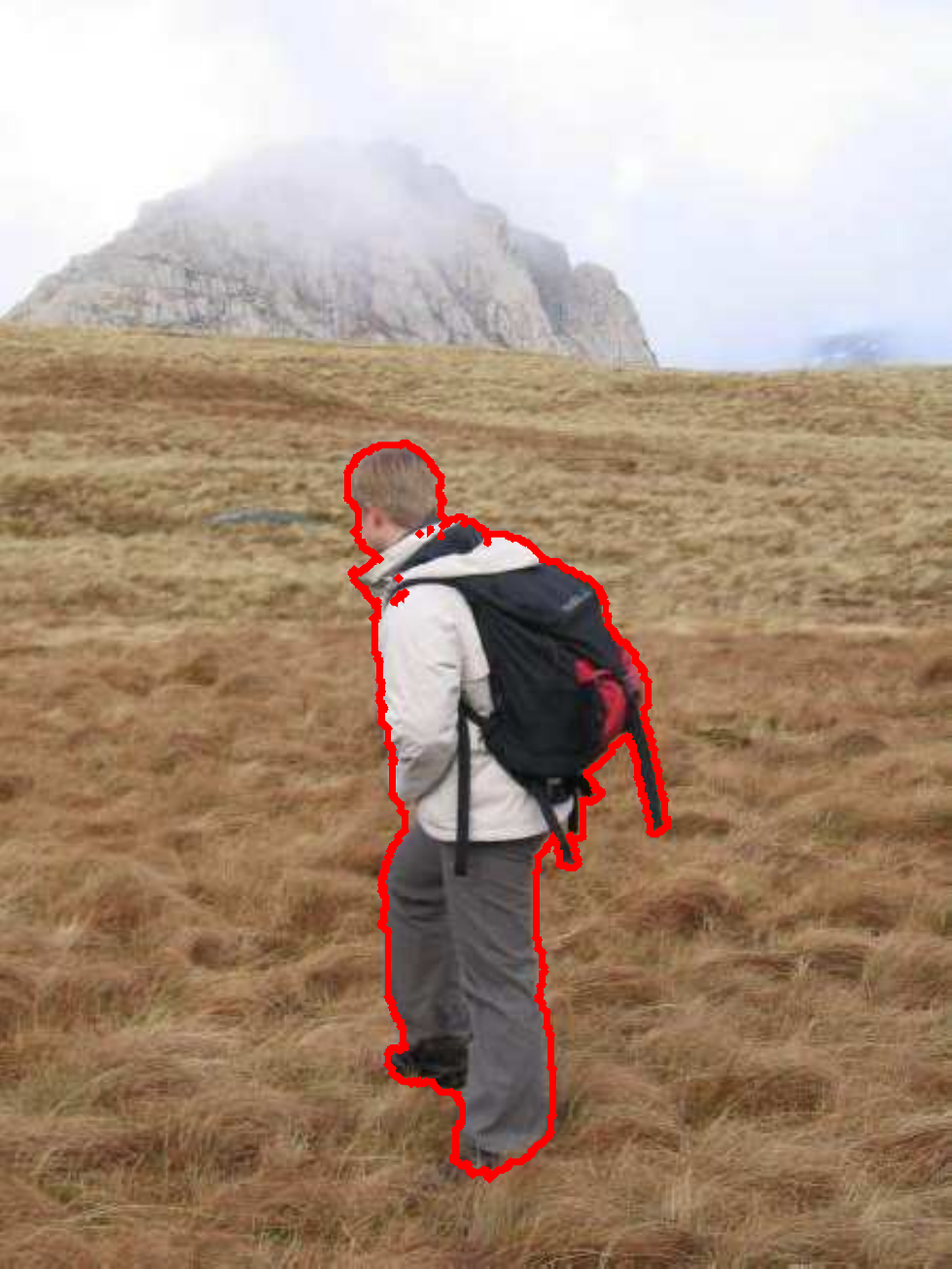}}
\,
\subfloat{\includegraphics[height=.1\columnwidth]{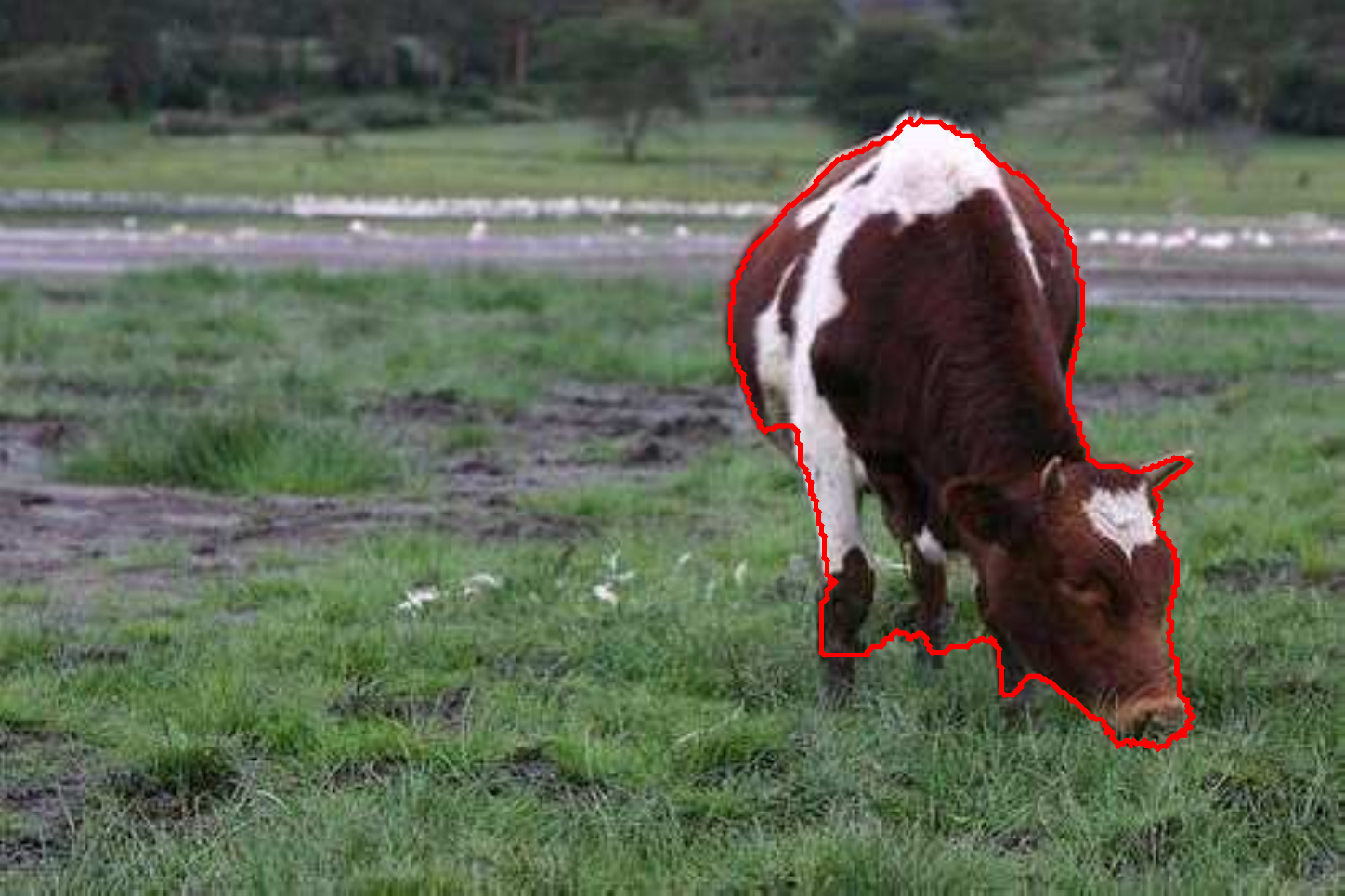}}
\,
\subfloat{\includegraphics[height=.1\columnwidth]{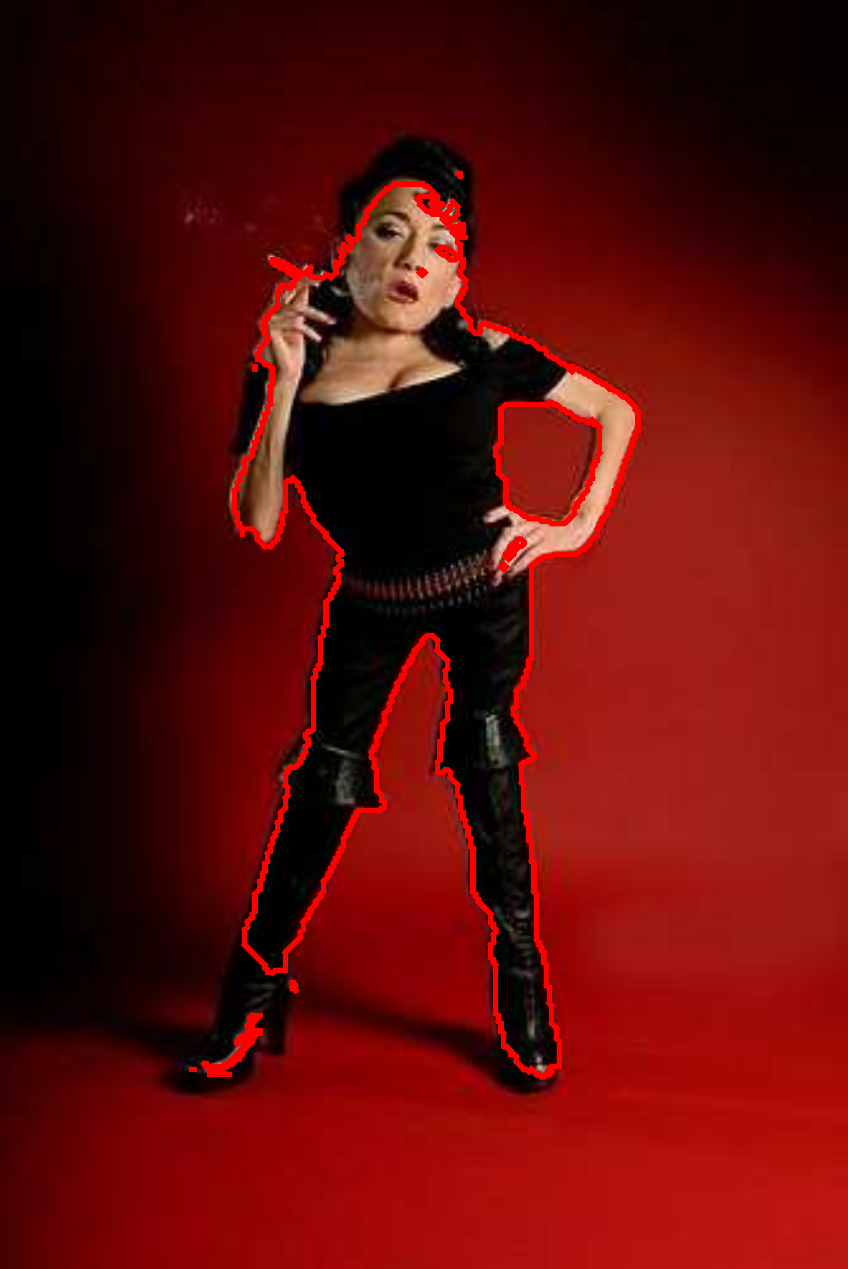}}
\,
\subfloat{\includegraphics[height=.1\columnwidth]{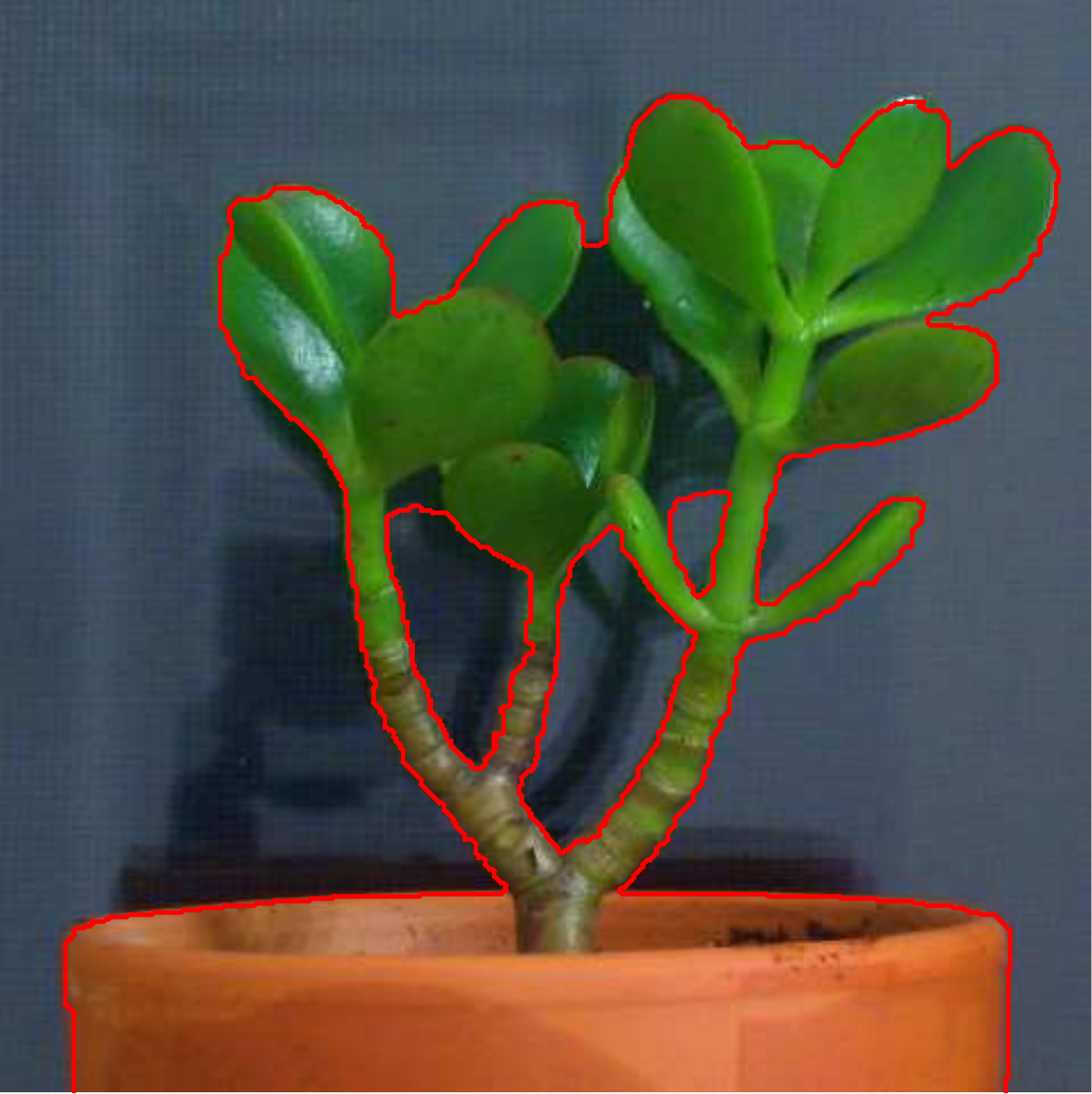}}
\,
\subfloat{\includegraphics[height=.1\columnwidth]{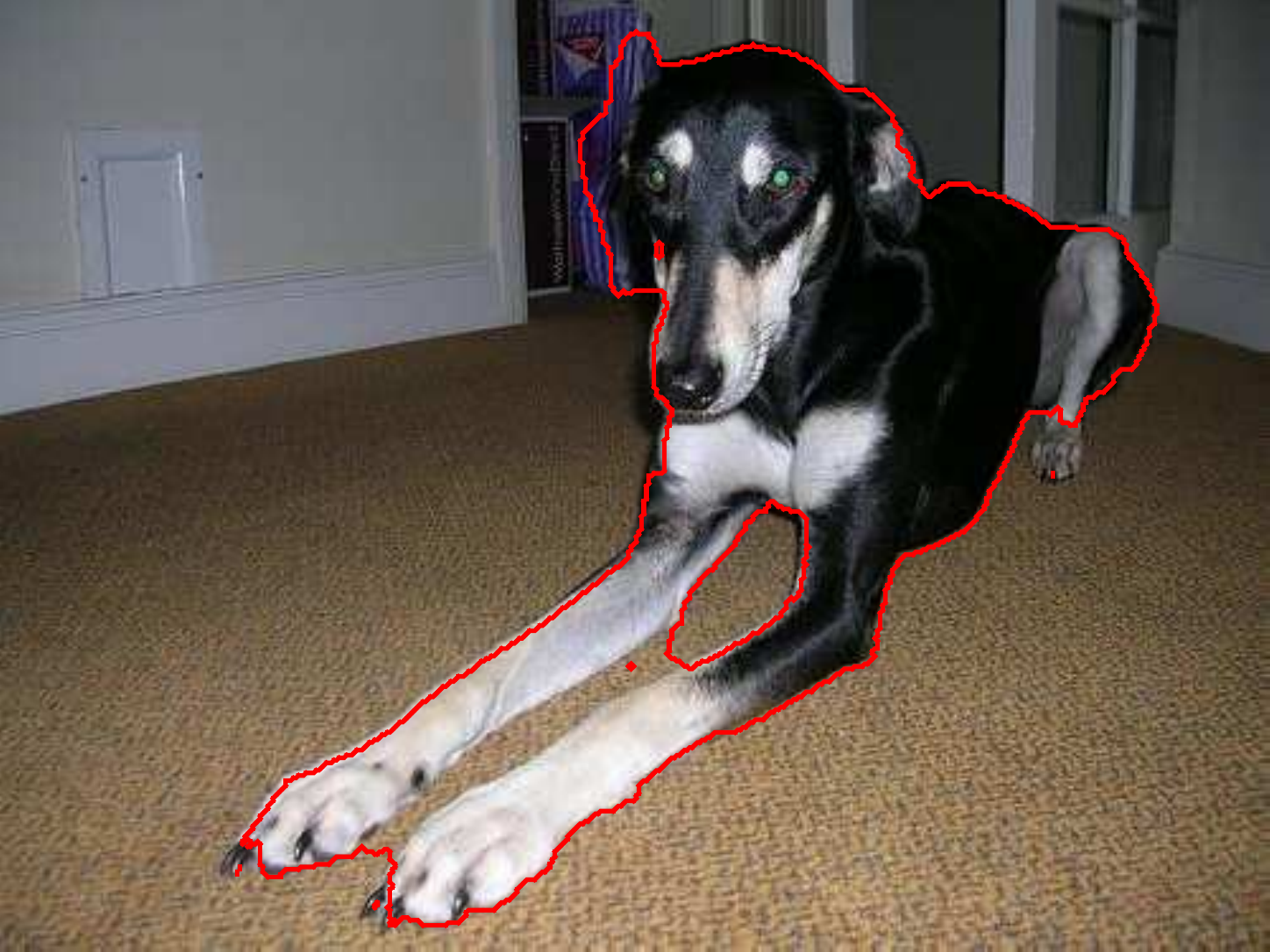}}
\hfil
\\
\subfloat{\includegraphics[height=.1\columnwidth]{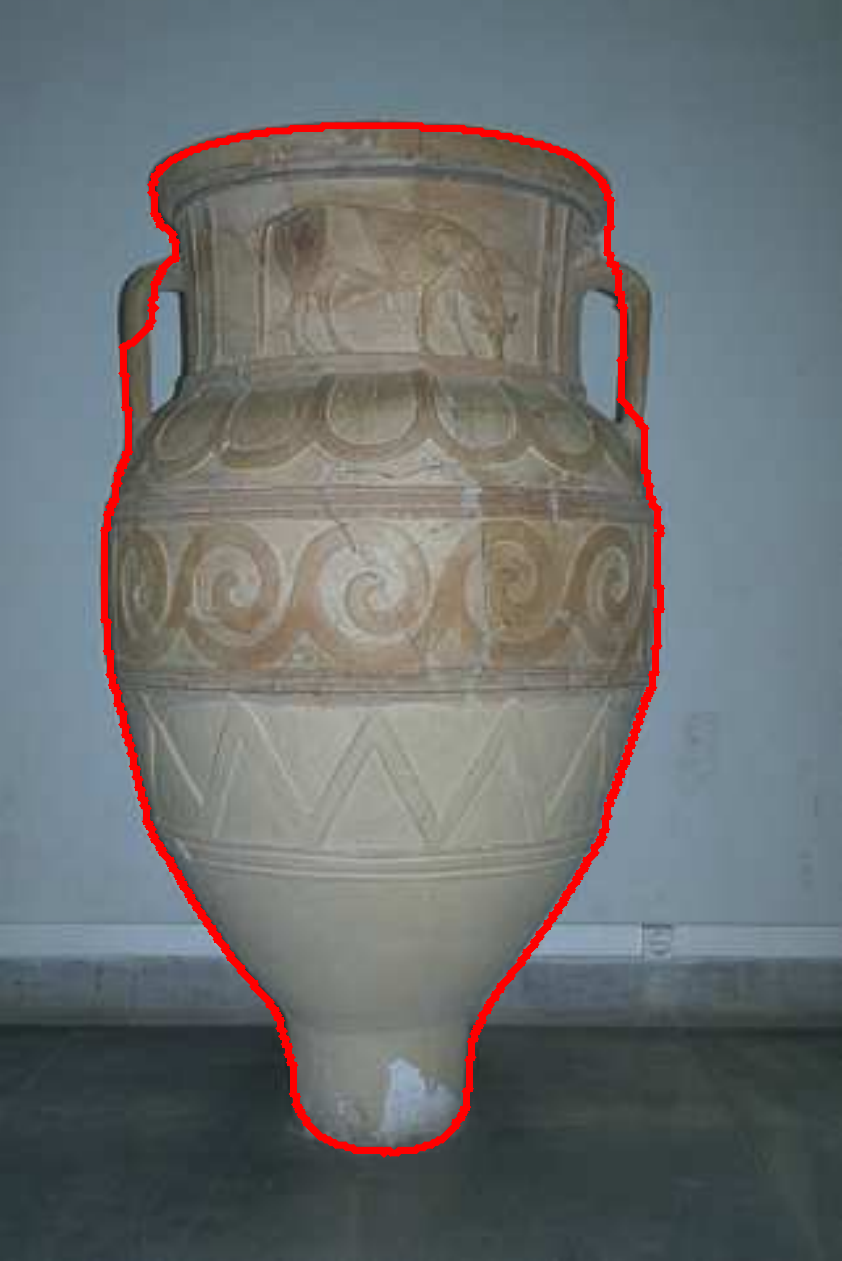}}
\,
\subfloat{\includegraphics[height=.1\columnwidth]{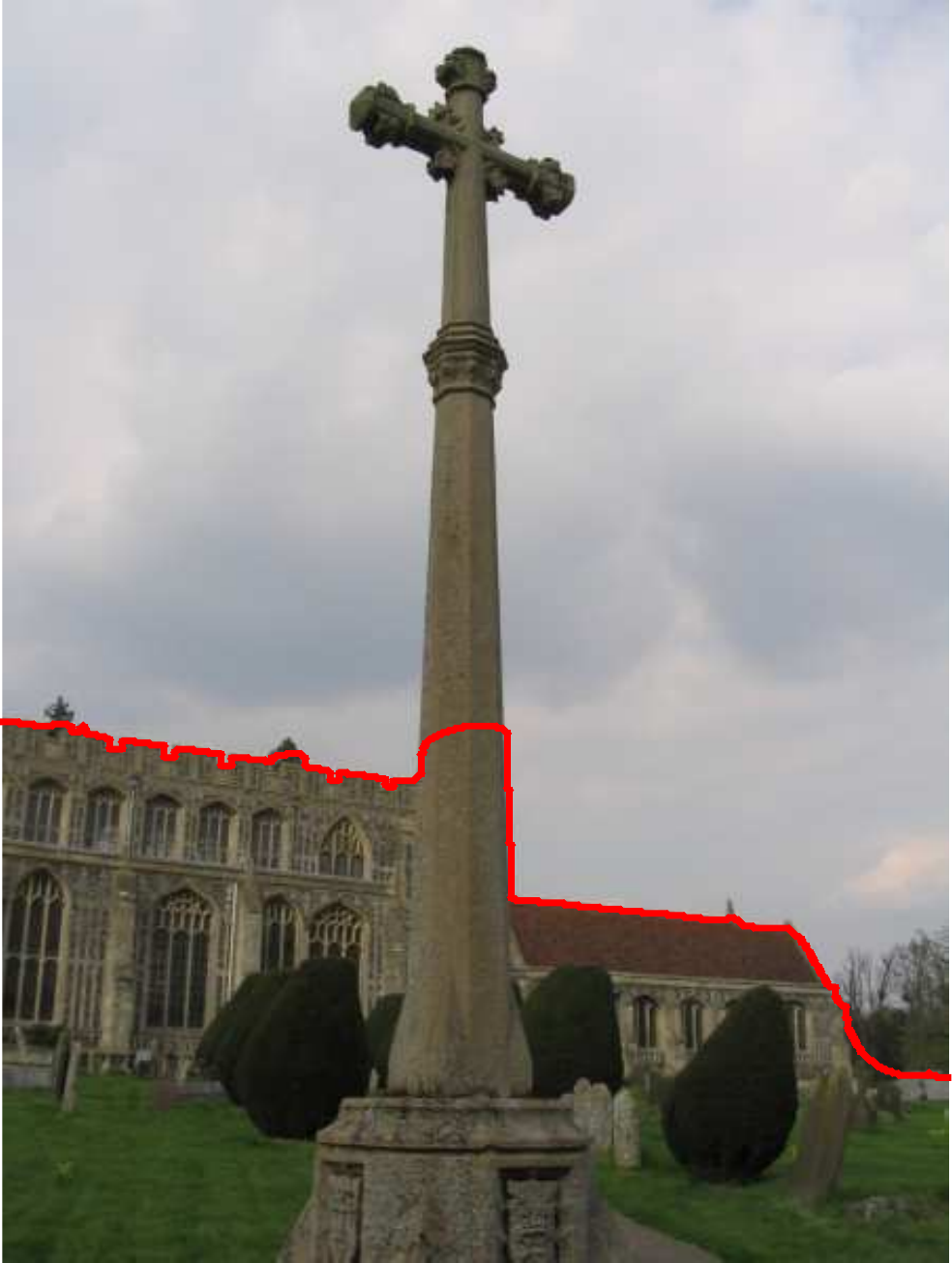}}
\,
\subfloat{\includegraphics[height=.1\columnwidth]{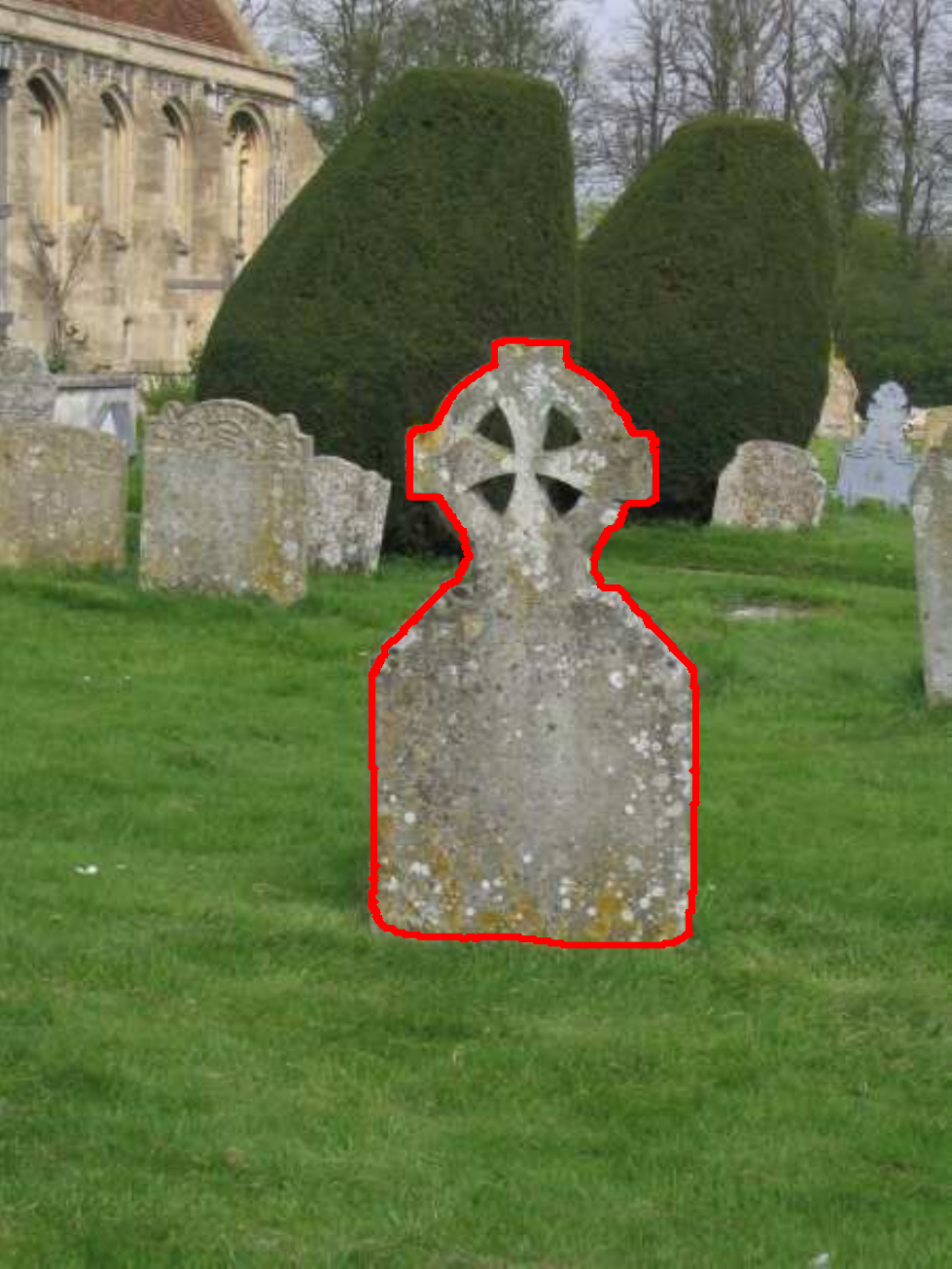}}
\,
\subfloat{\includegraphics[height=.1\columnwidth]{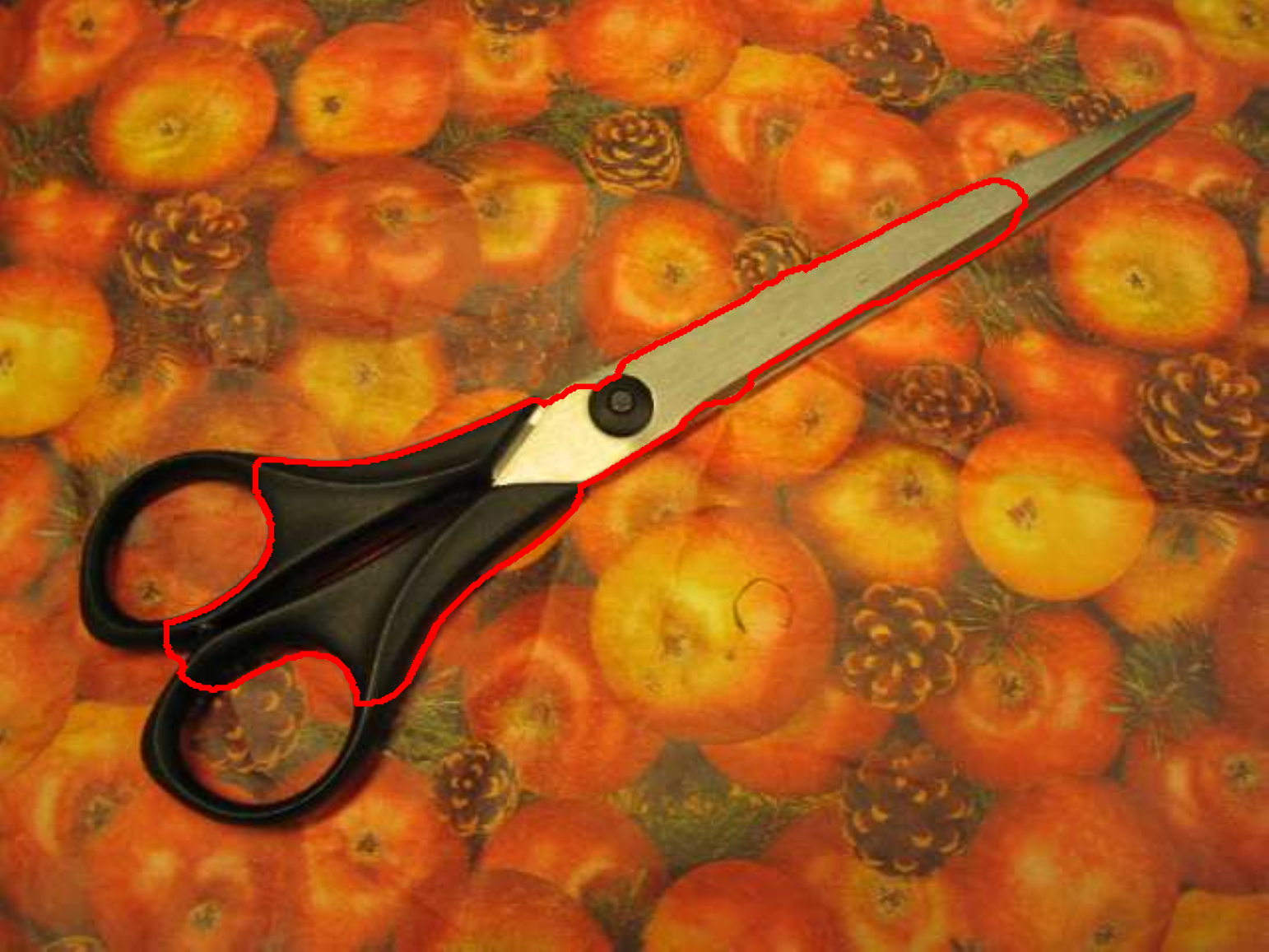}}
\,
\subfloat{\includegraphics[height=.1\columnwidth]{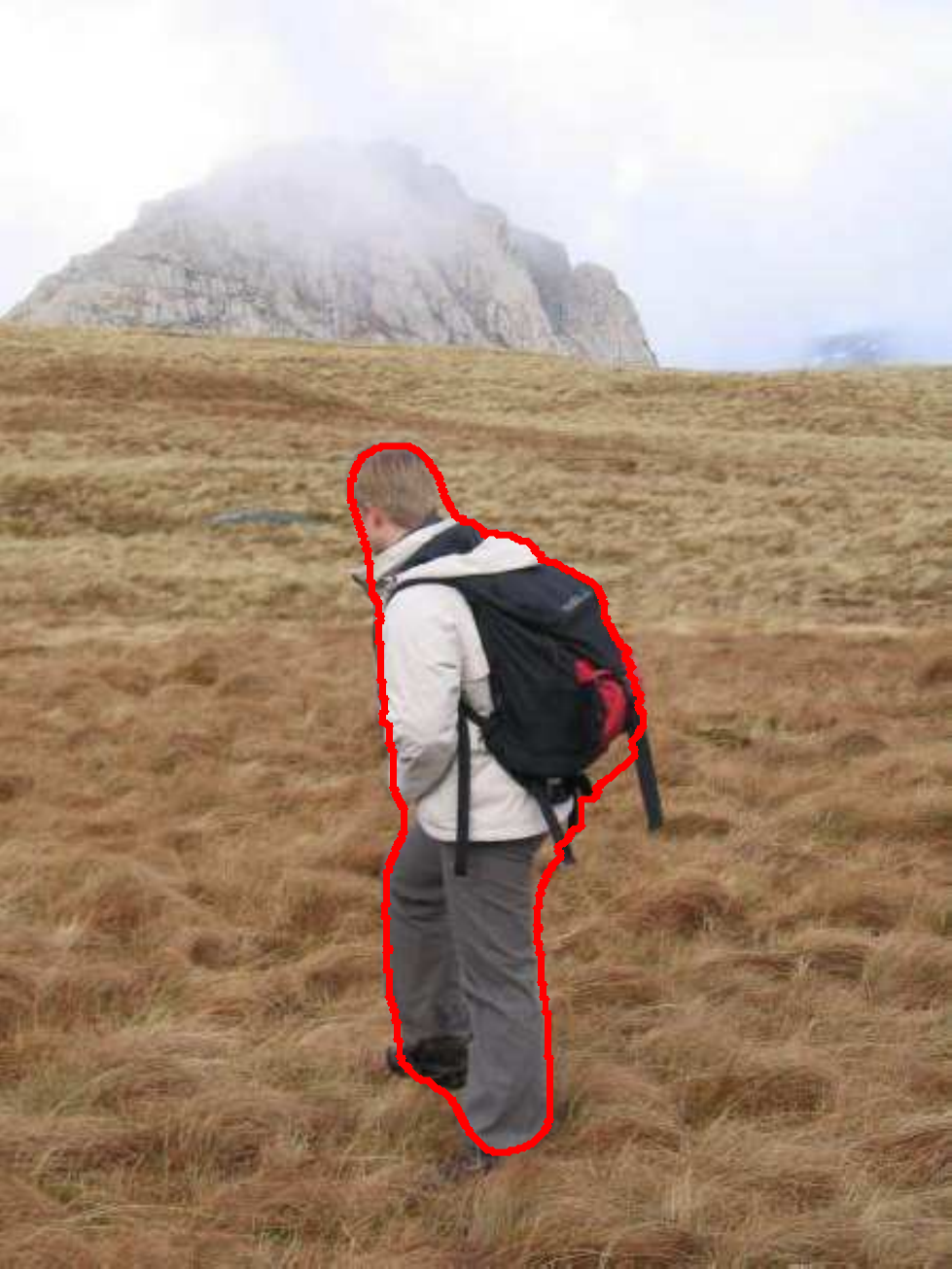}}
\,
\subfloat{\includegraphics[height=.1\columnwidth]{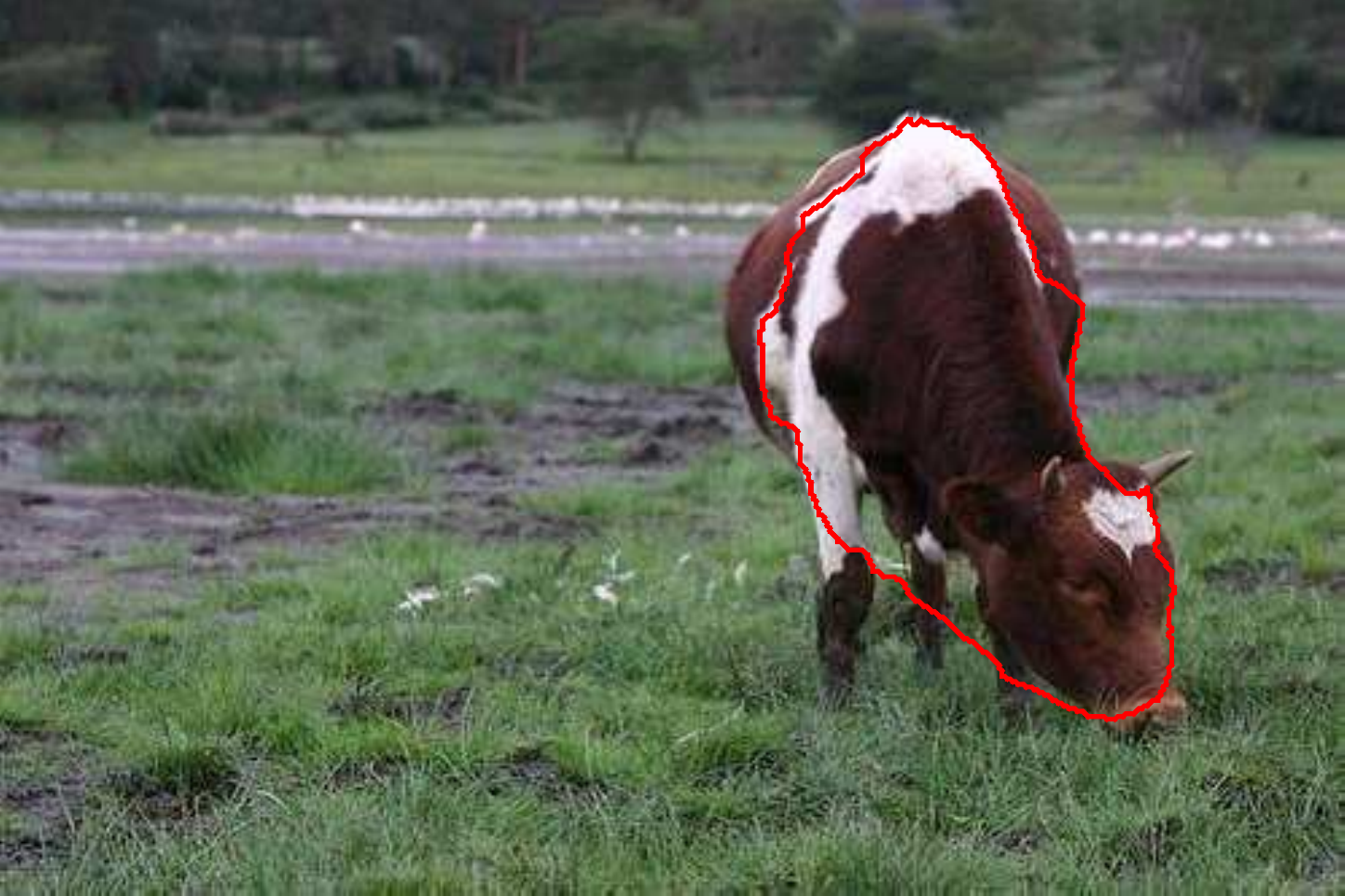}}
\,
\subfloat{\includegraphics[height=.1\columnwidth]{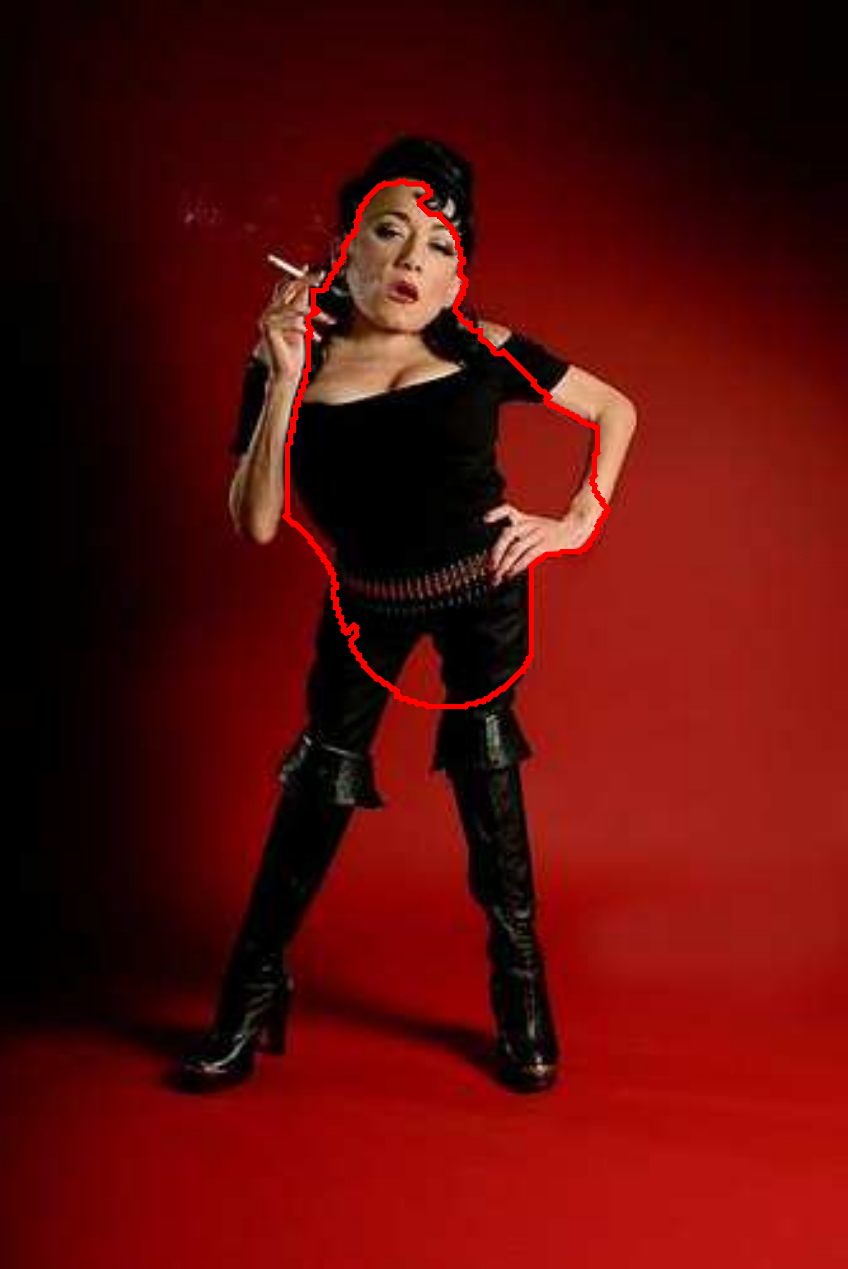}}
\,
\subfloat{\includegraphics[height=.1\columnwidth]{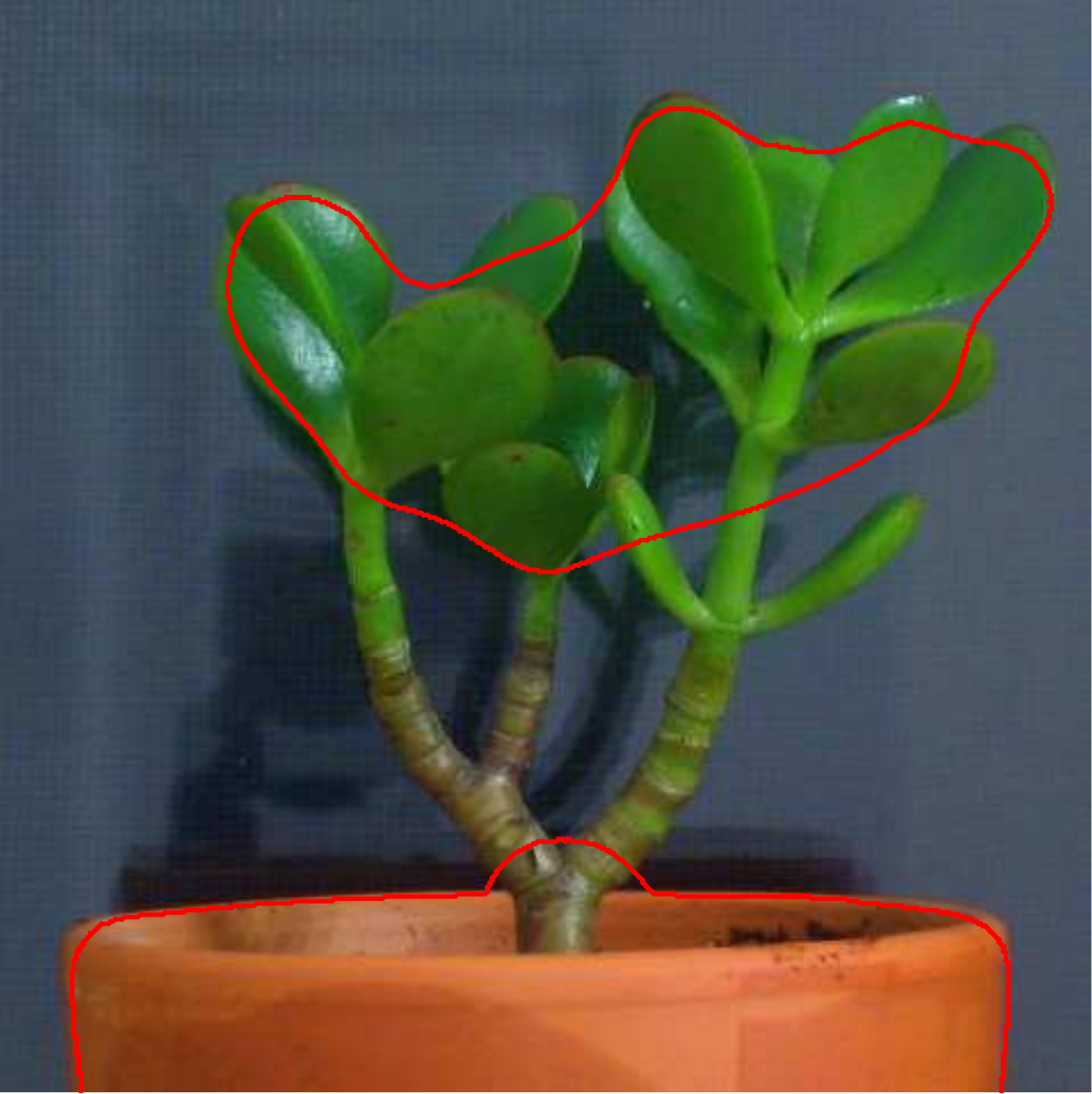}}
\,
\subfloat{\includegraphics[height=.1\columnwidth]{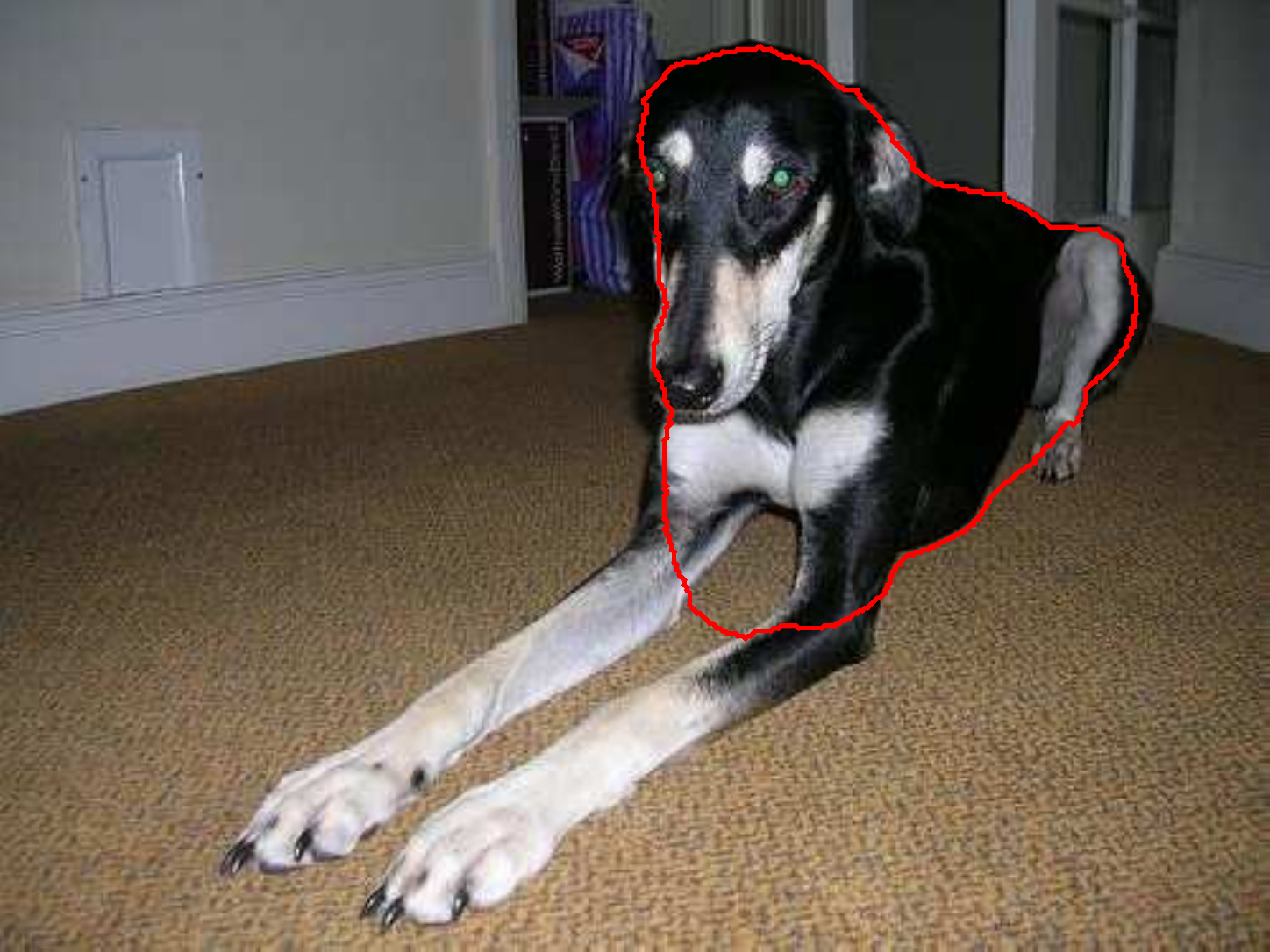}}
\hfil
\\
\subfloat{\includegraphics[height=.1\columnwidth]{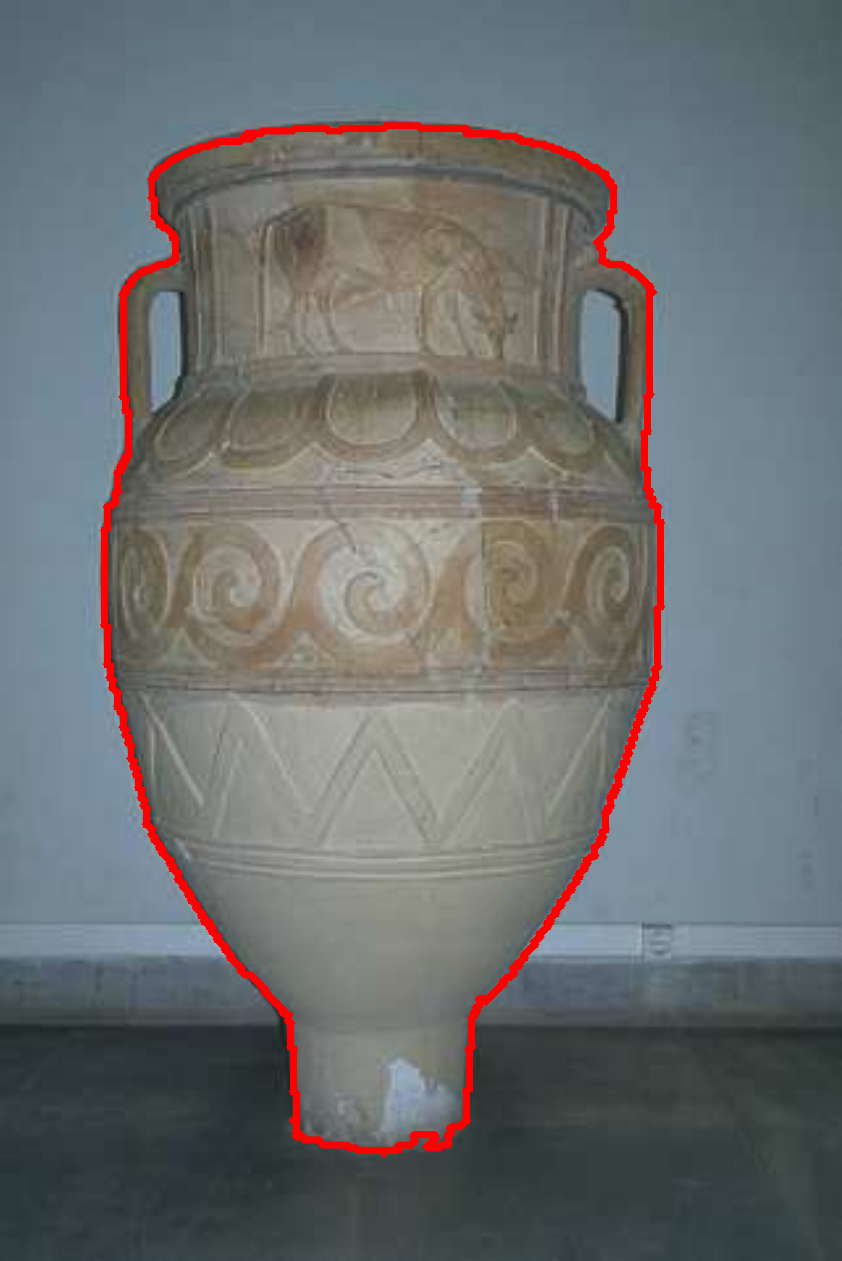}}
\,
\subfloat{\includegraphics[height=.1\columnwidth]{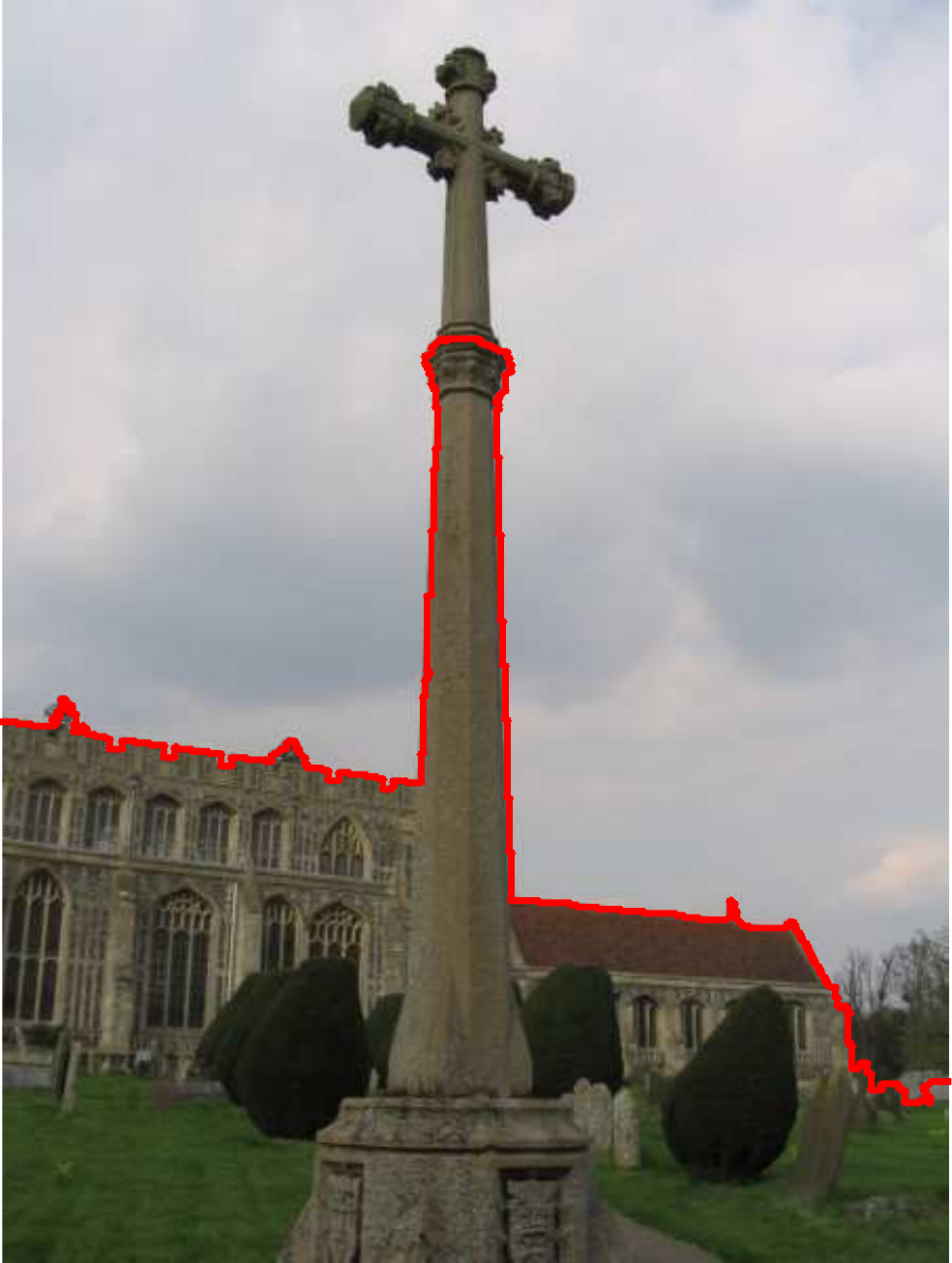}}
\,
\subfloat{\includegraphics[height=.1\columnwidth]{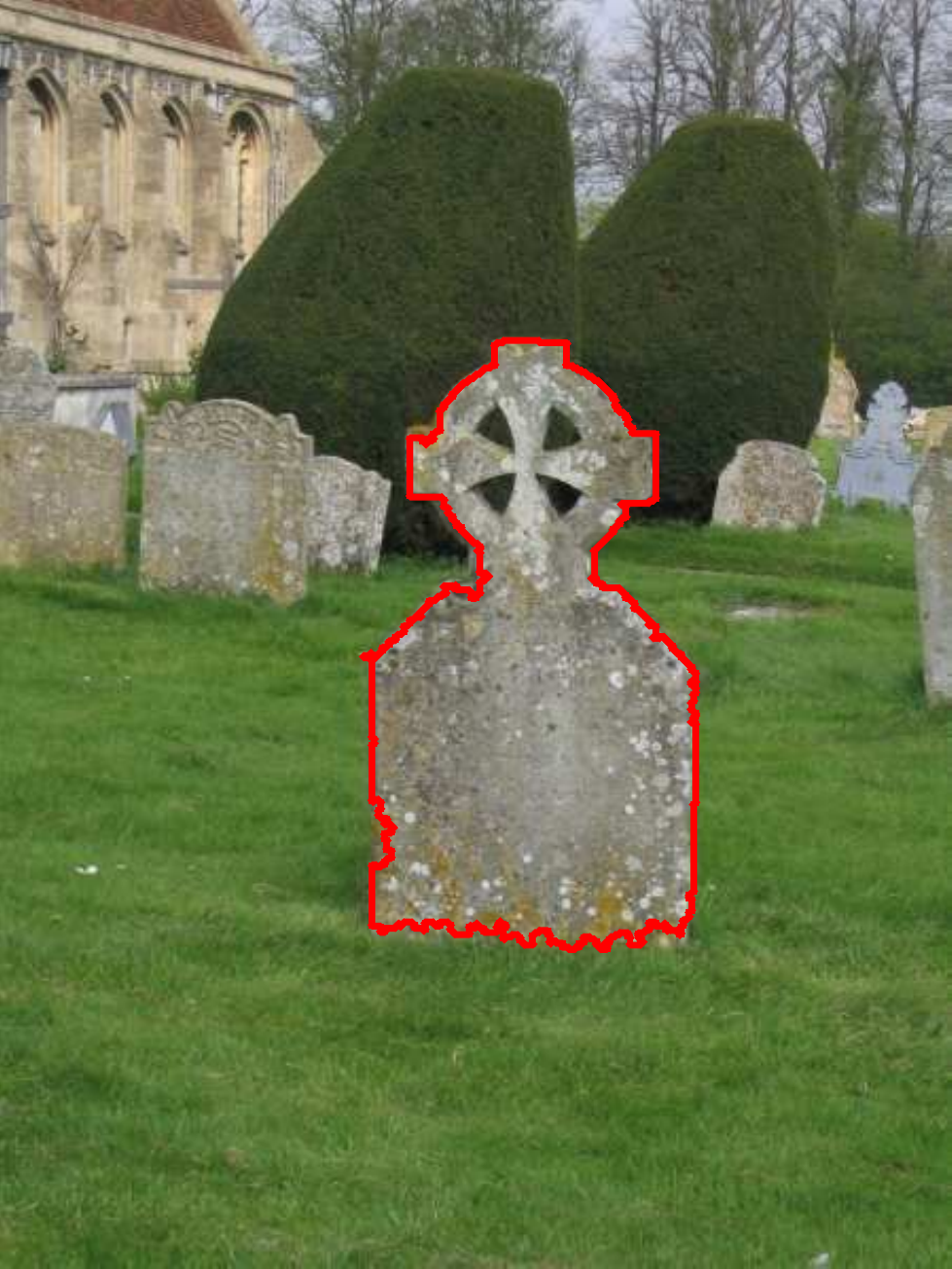}}
\,
\subfloat{\includegraphics[height=.1\columnwidth]{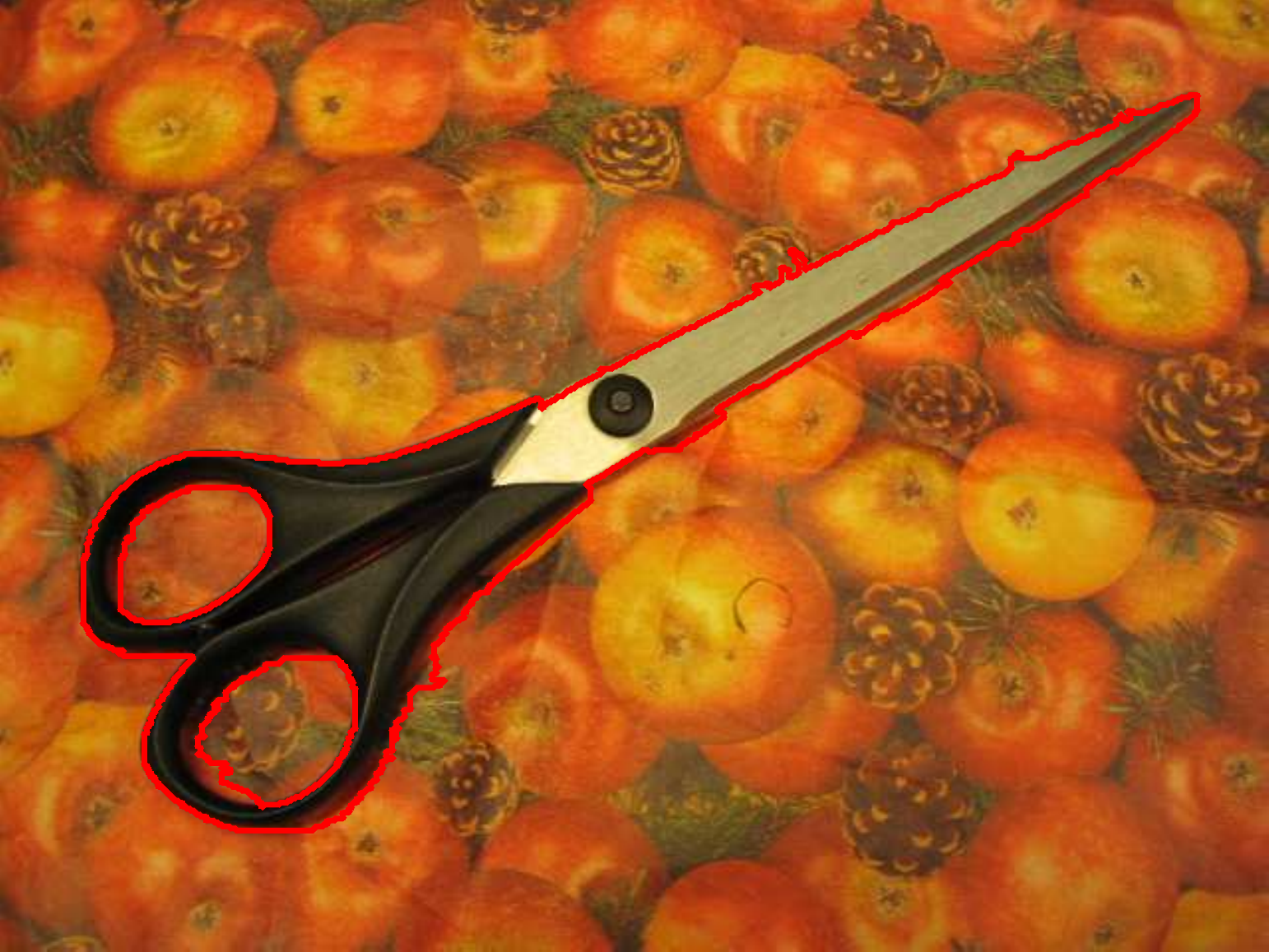}}
\,
\subfloat{\includegraphics[height=.1\columnwidth]{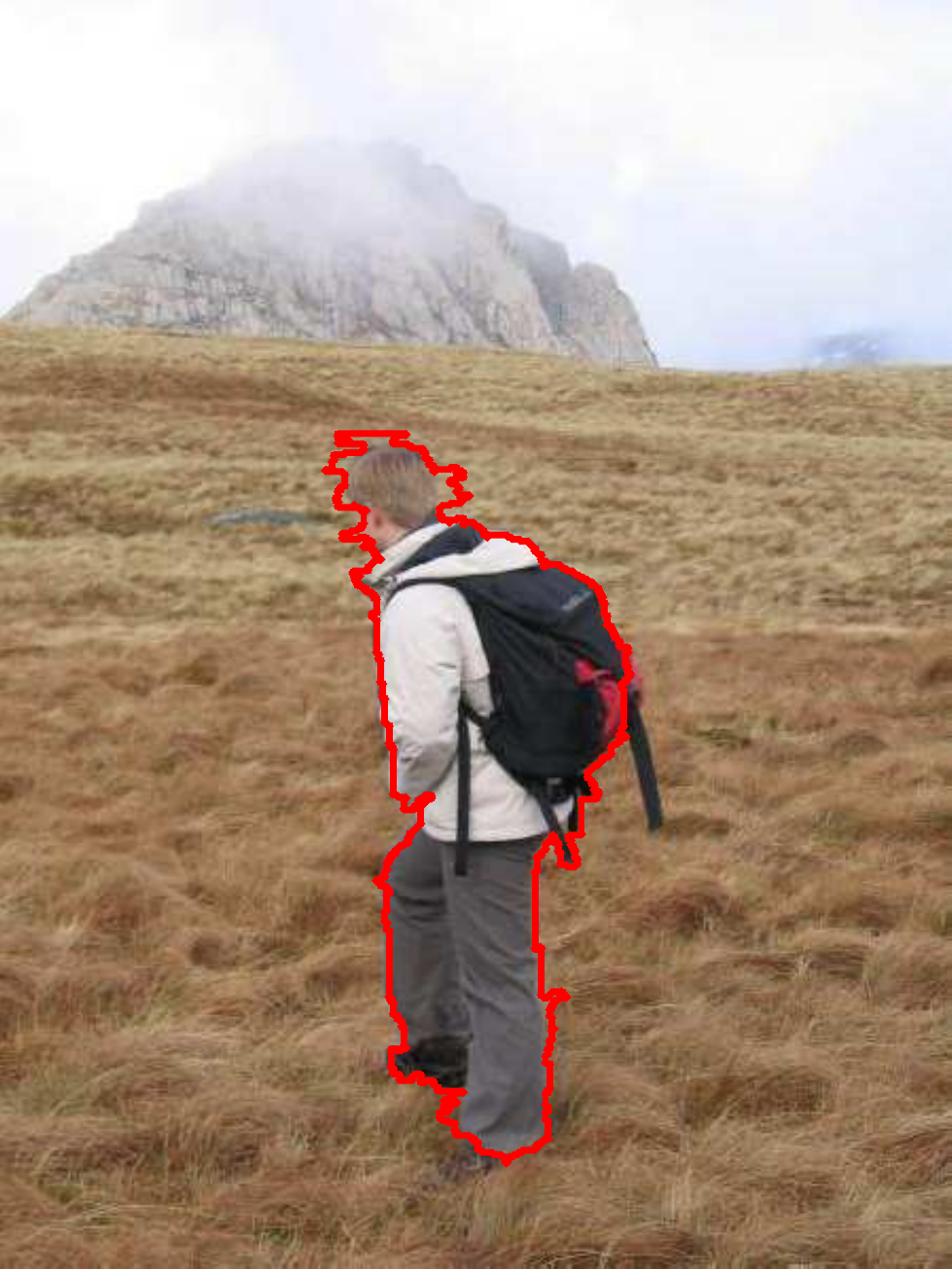}}
\,
\subfloat{\includegraphics[height=.1\columnwidth]{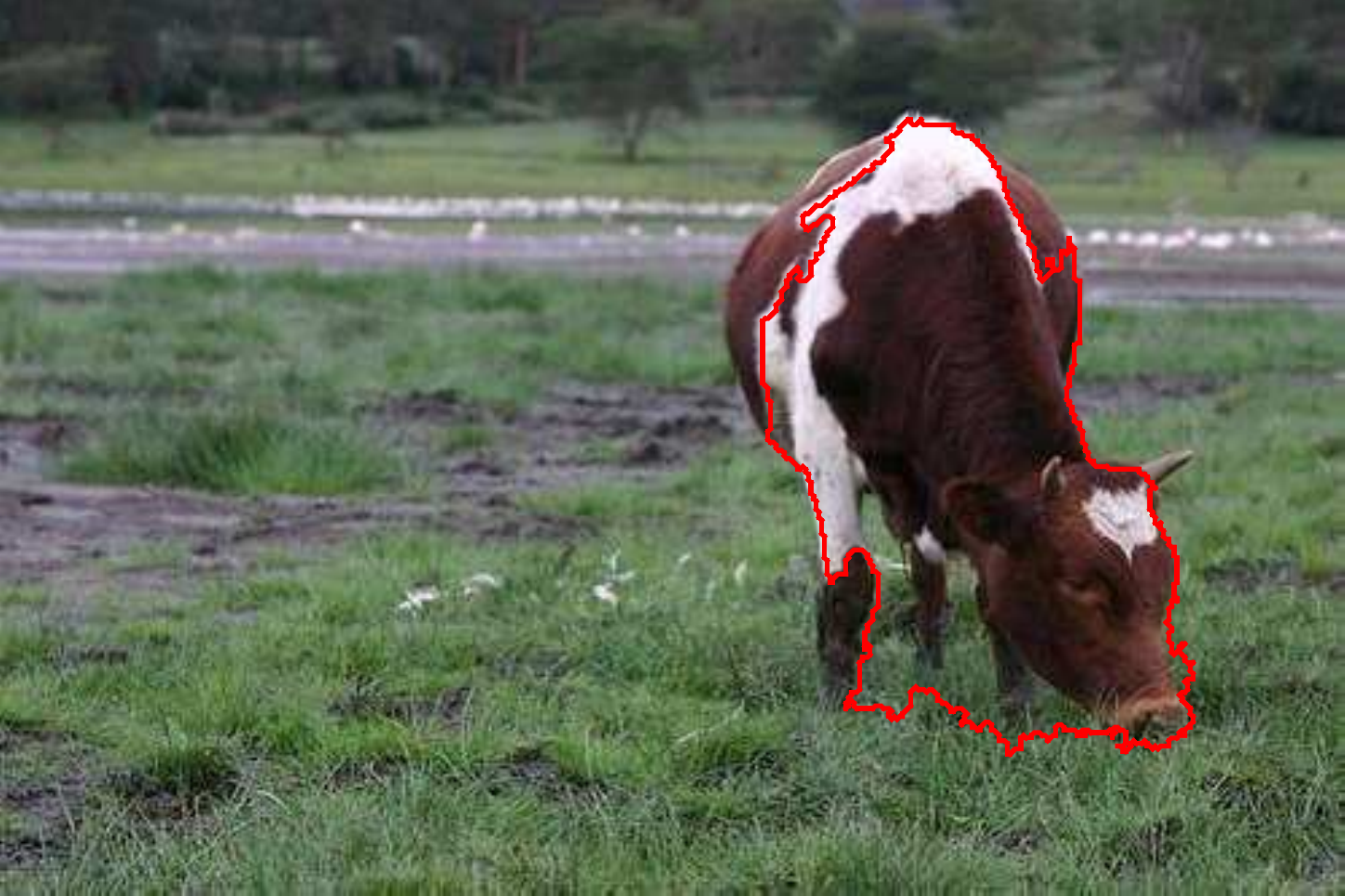}}
\,
\subfloat{\includegraphics[height=.1\columnwidth]{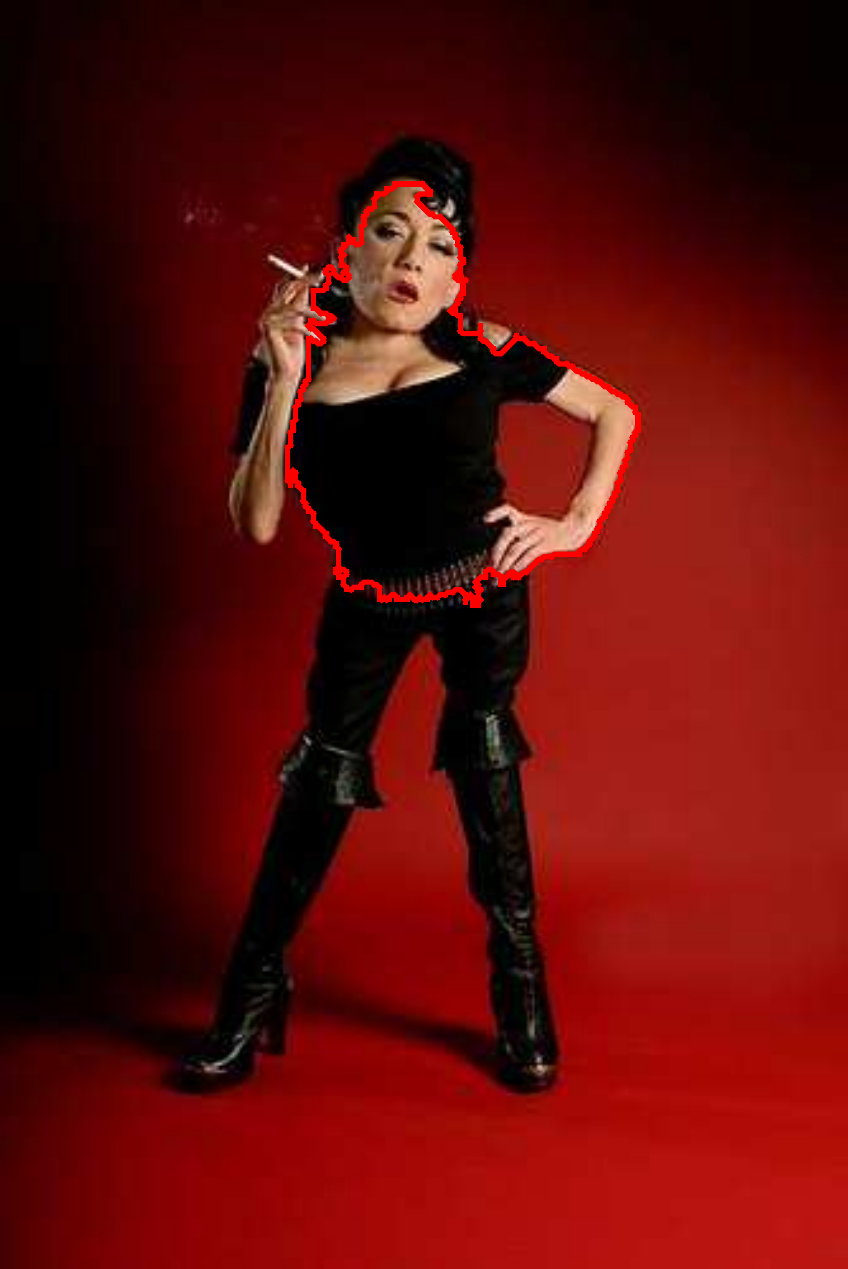}}
\,
\subfloat{\includegraphics[height=.1\columnwidth]{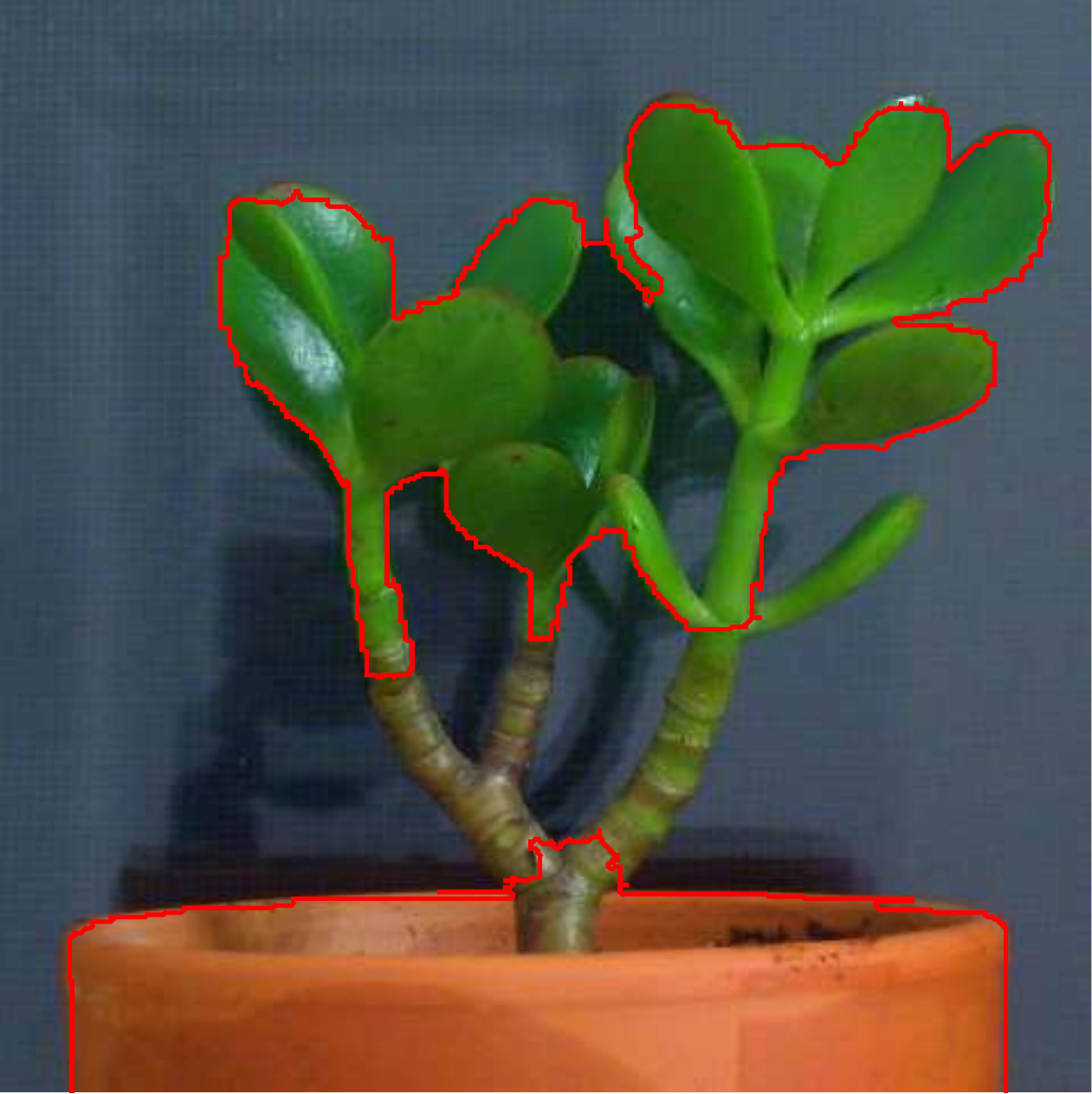}}
\,
\subfloat{\includegraphics[height=.1\columnwidth]{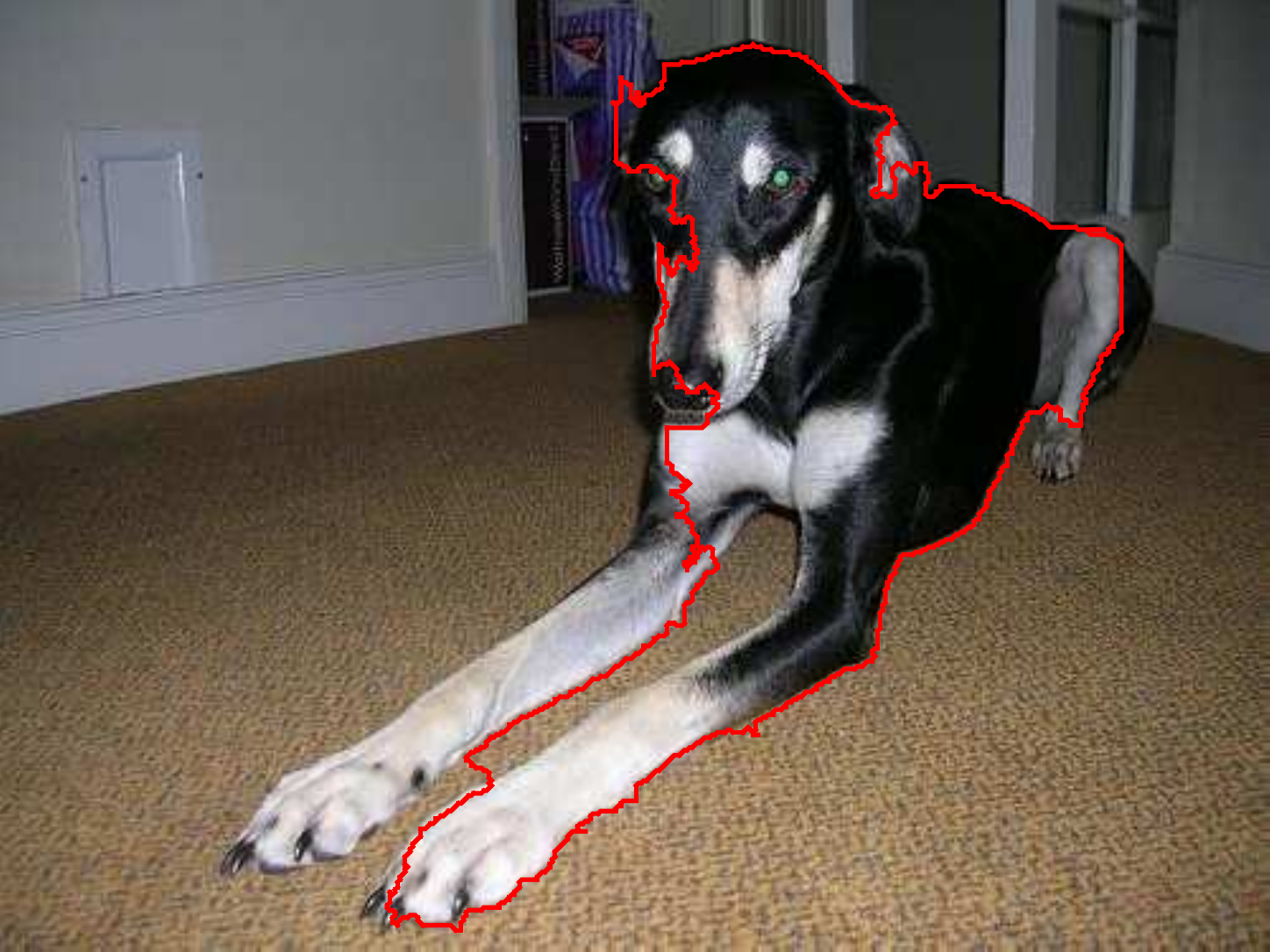}}
\hfil
\\

\subfloat{\includegraphics[height=.1\columnwidth]{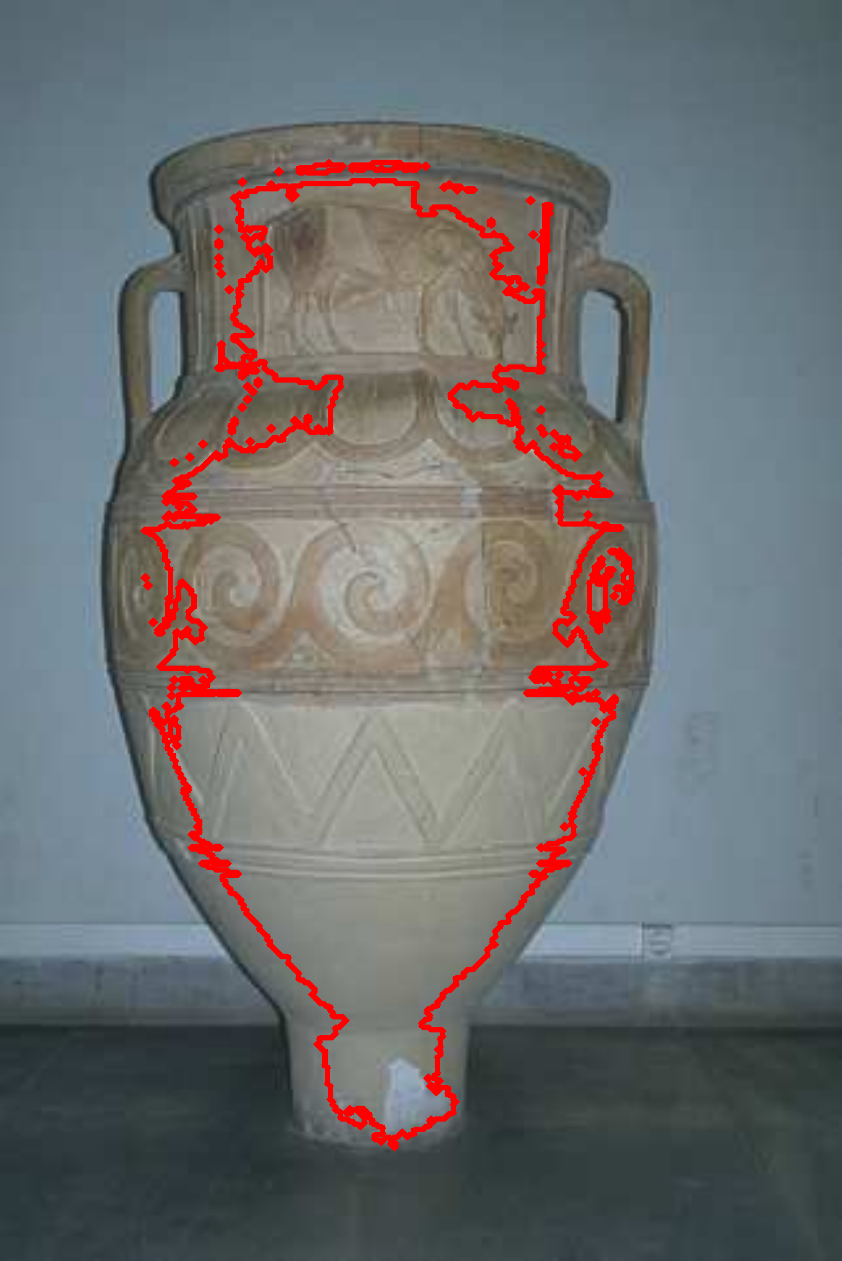}}
\,
\subfloat{\includegraphics[height=.1\columnwidth]{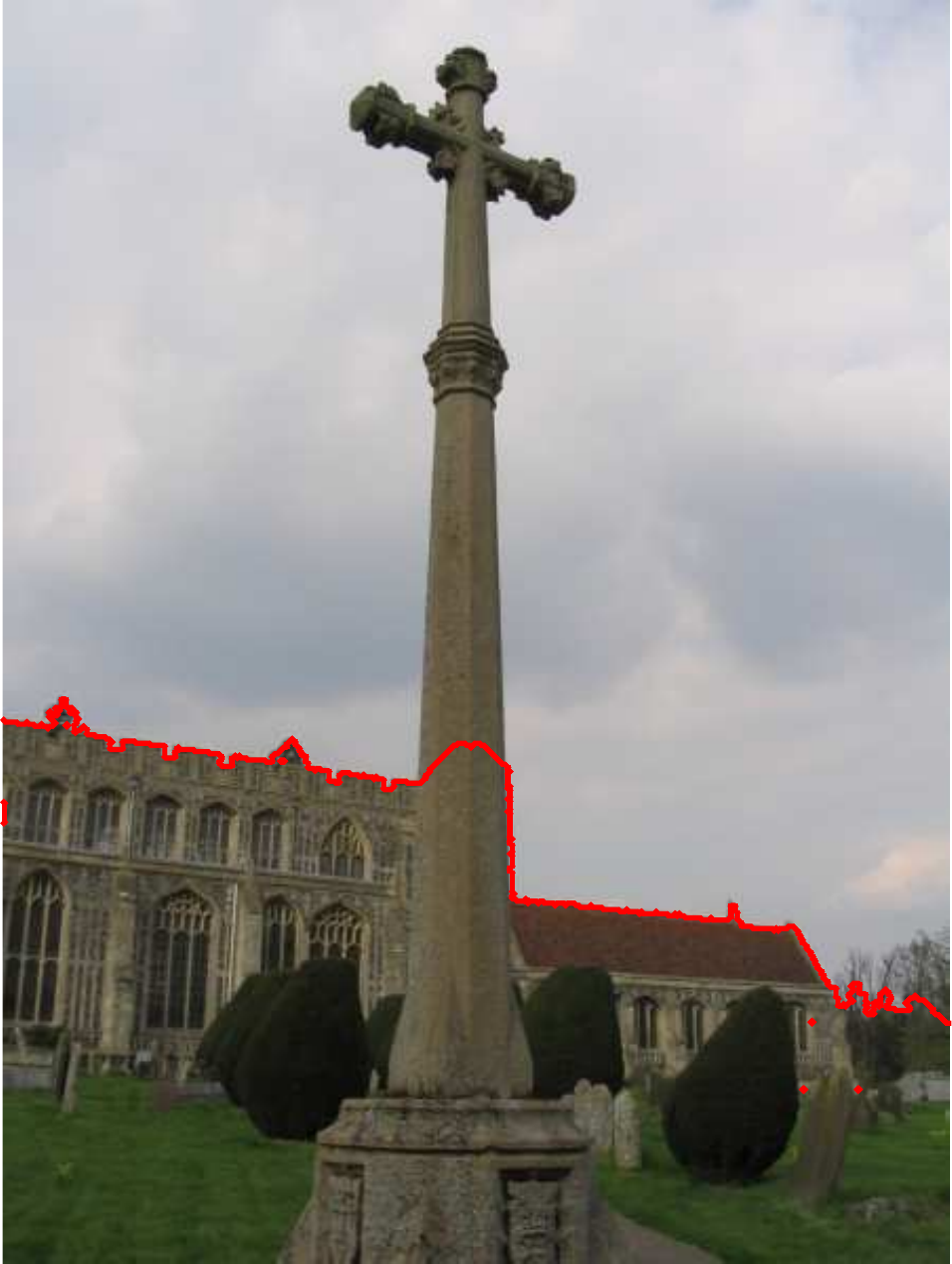}}
\,
\subfloat{\includegraphics[height=.1\columnwidth]{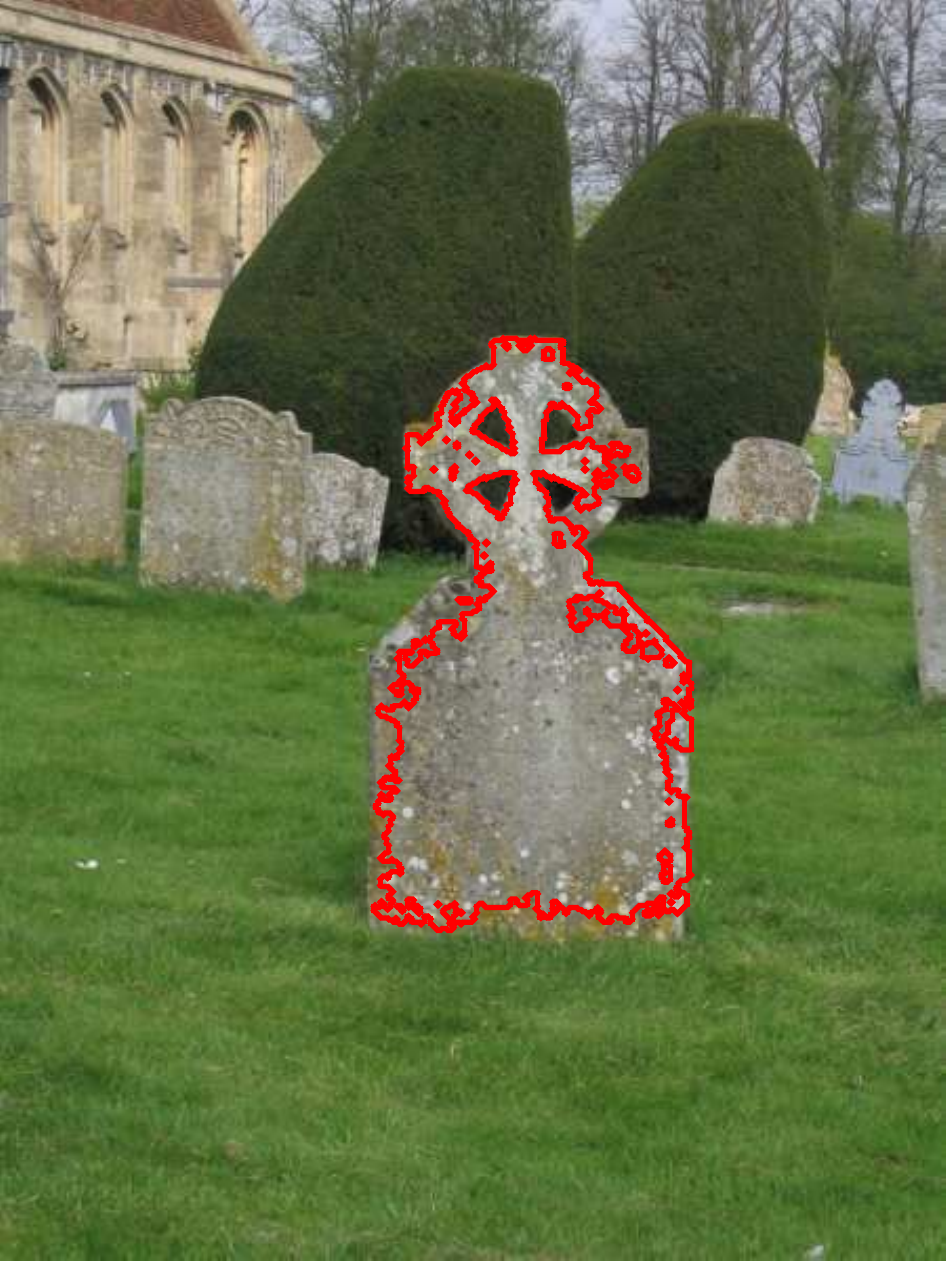}}
\,
\subfloat{\includegraphics[height=.1\columnwidth]{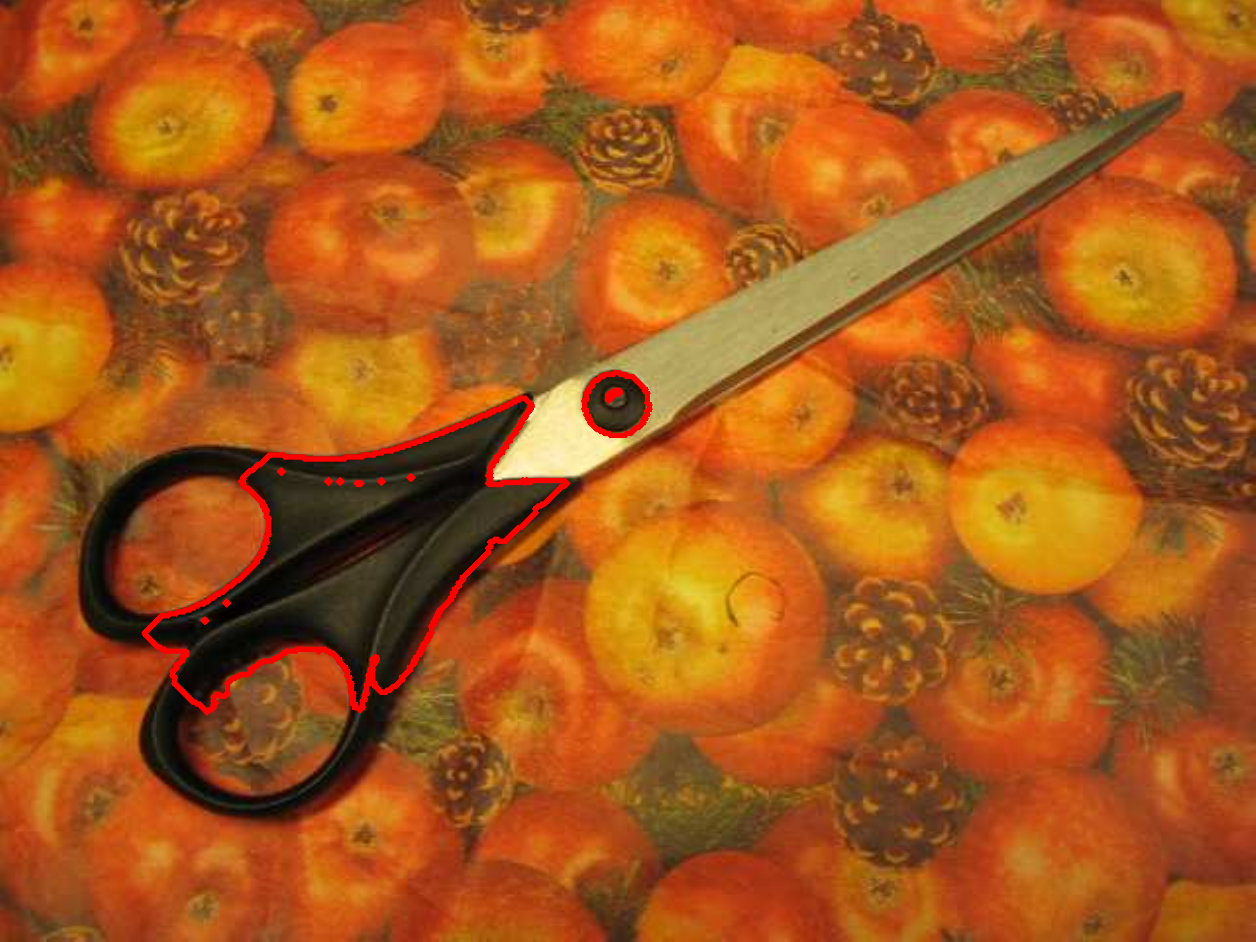}}
\,
\subfloat{\includegraphics[height=.1\columnwidth]{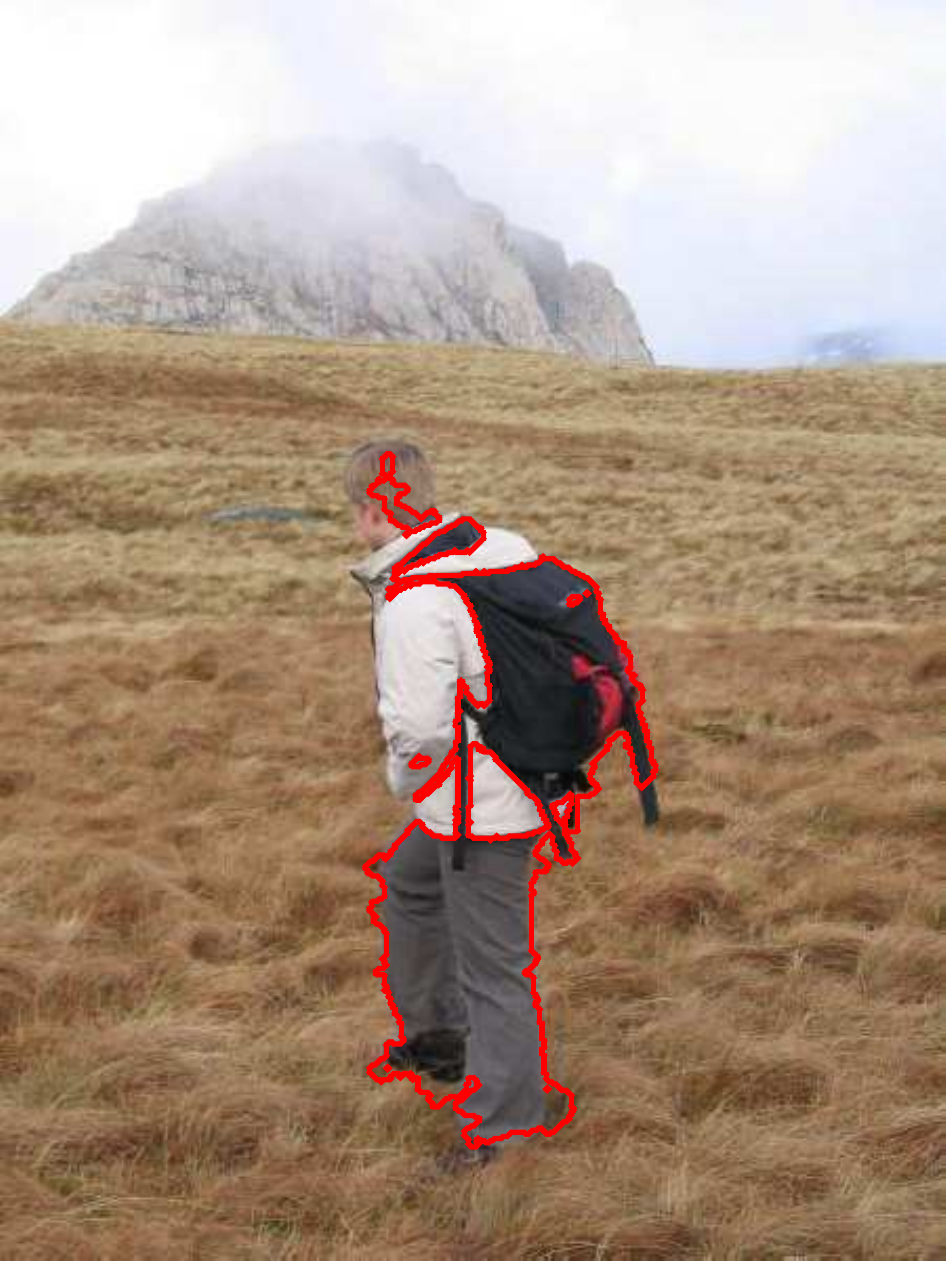}}
\,
\subfloat{\includegraphics[height=.1\columnwidth]{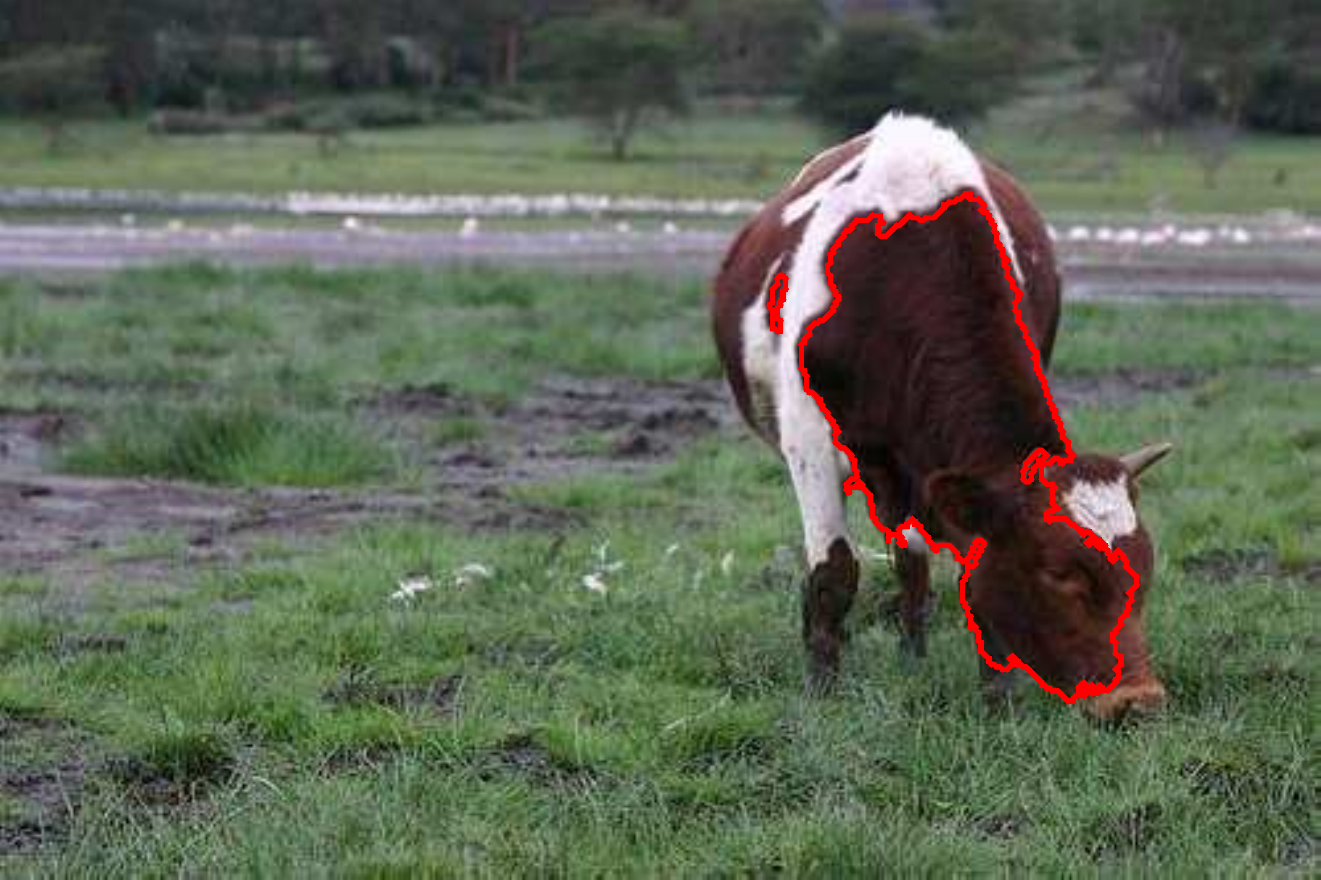}}
\,
\subfloat{\includegraphics[height=.1\columnwidth]{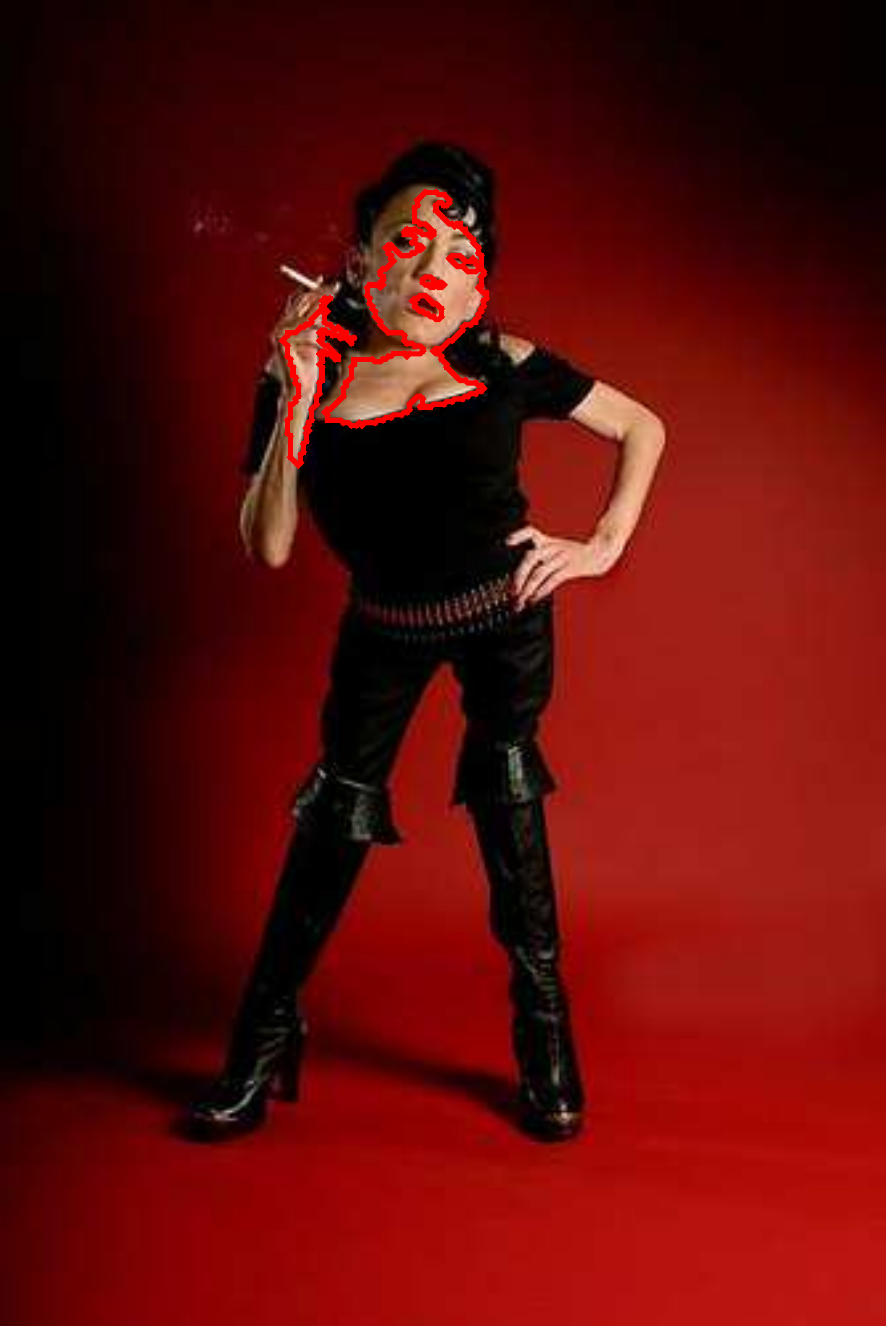}}
\,
\subfloat{\includegraphics[height=.1\columnwidth]{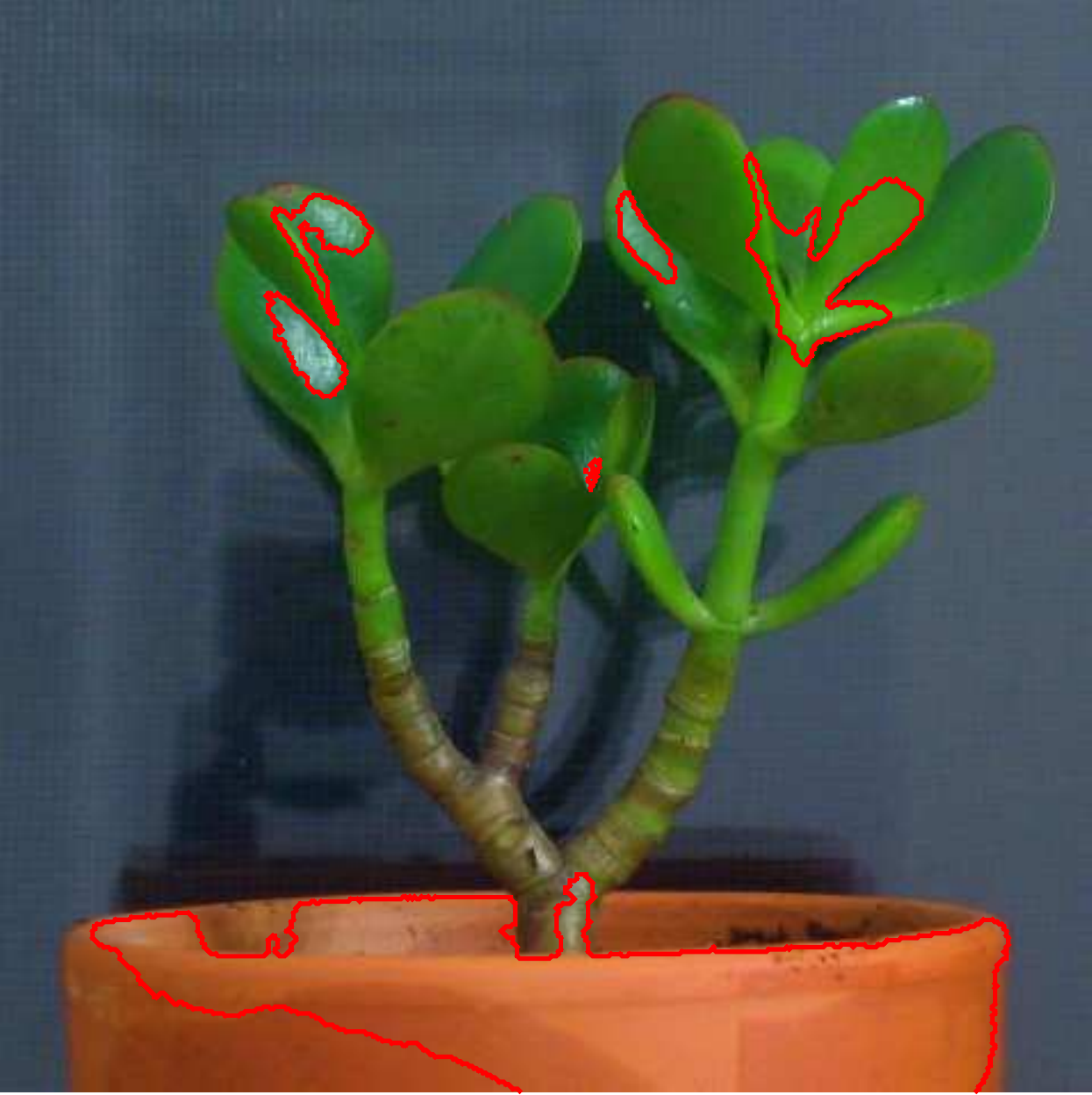}}
\,
\subfloat{\includegraphics[height=.1\columnwidth]{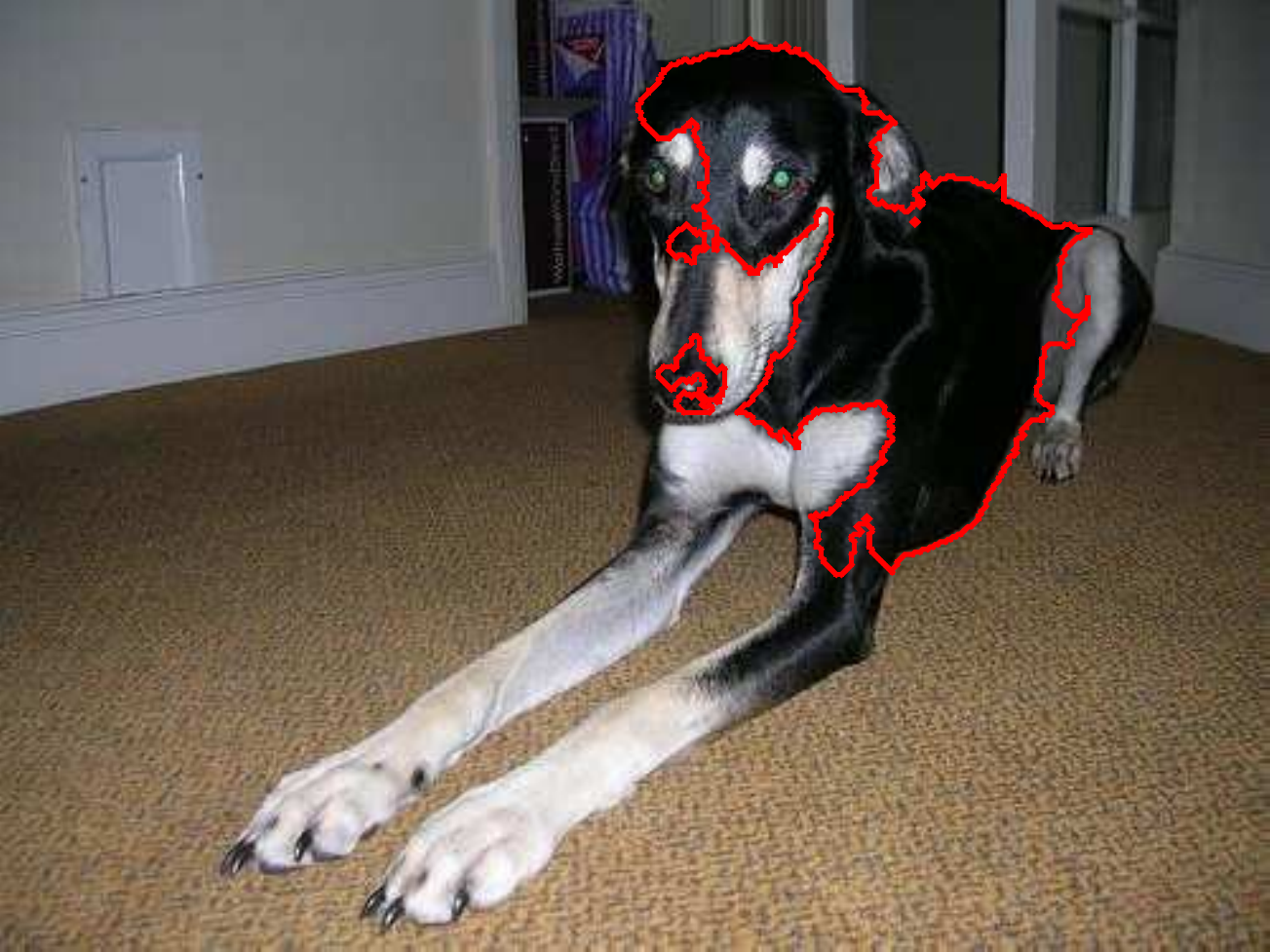}}
\hfil
\\
\subfloat{\includegraphics[height=.1\columnwidth]{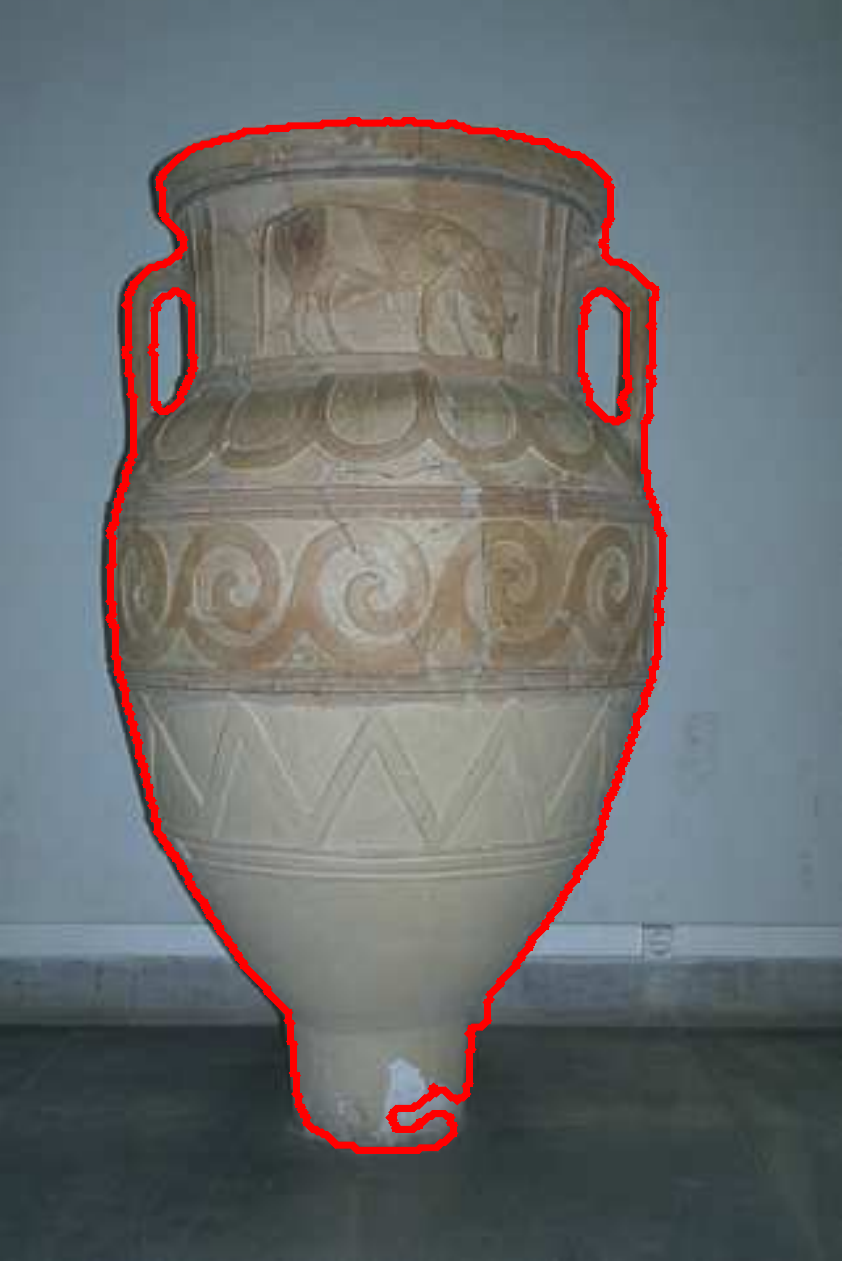}}
\,
\subfloat{\includegraphics[height=.1\columnwidth]{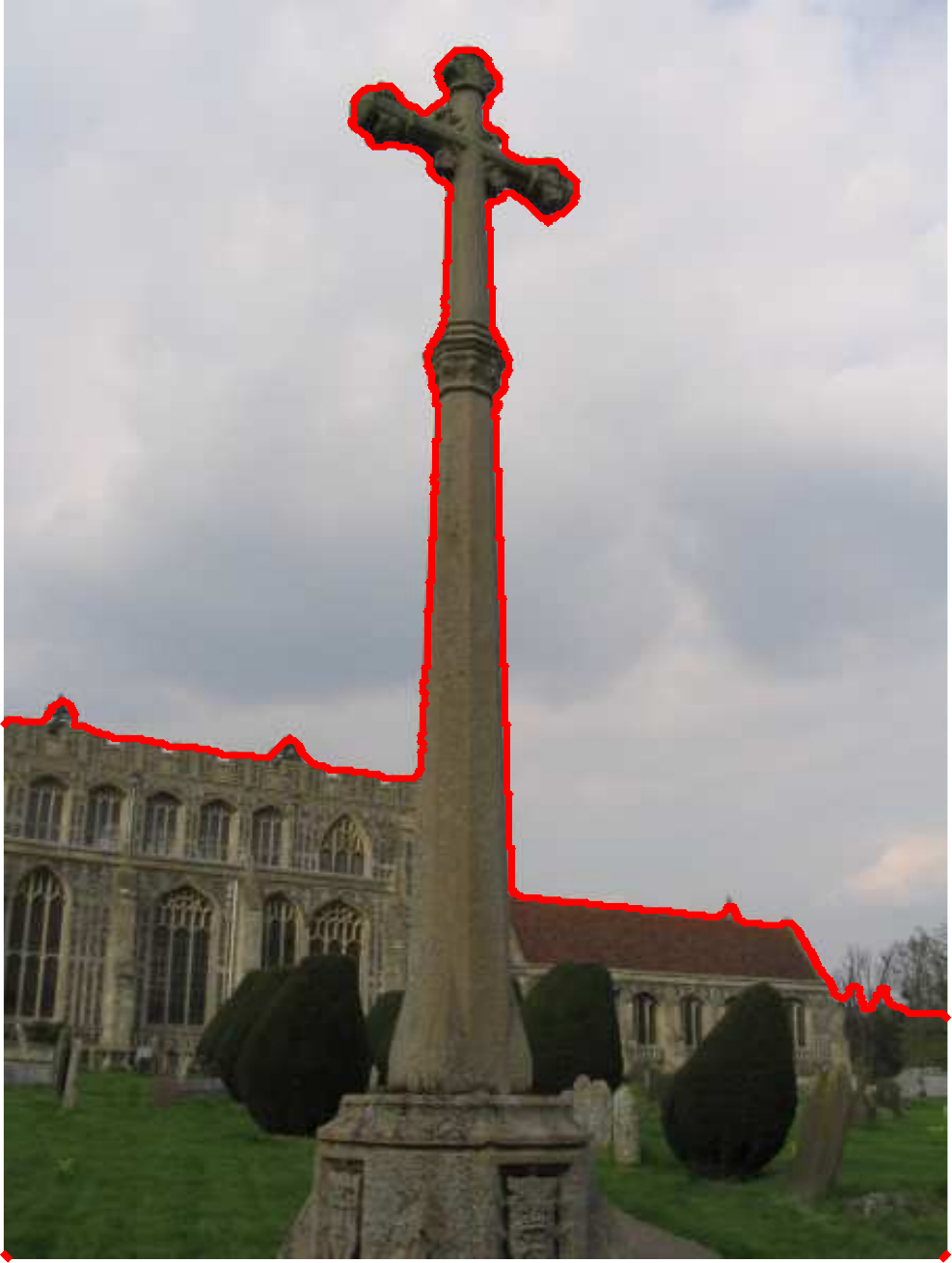}}
\,
\subfloat{\includegraphics[height=.1\columnwidth]{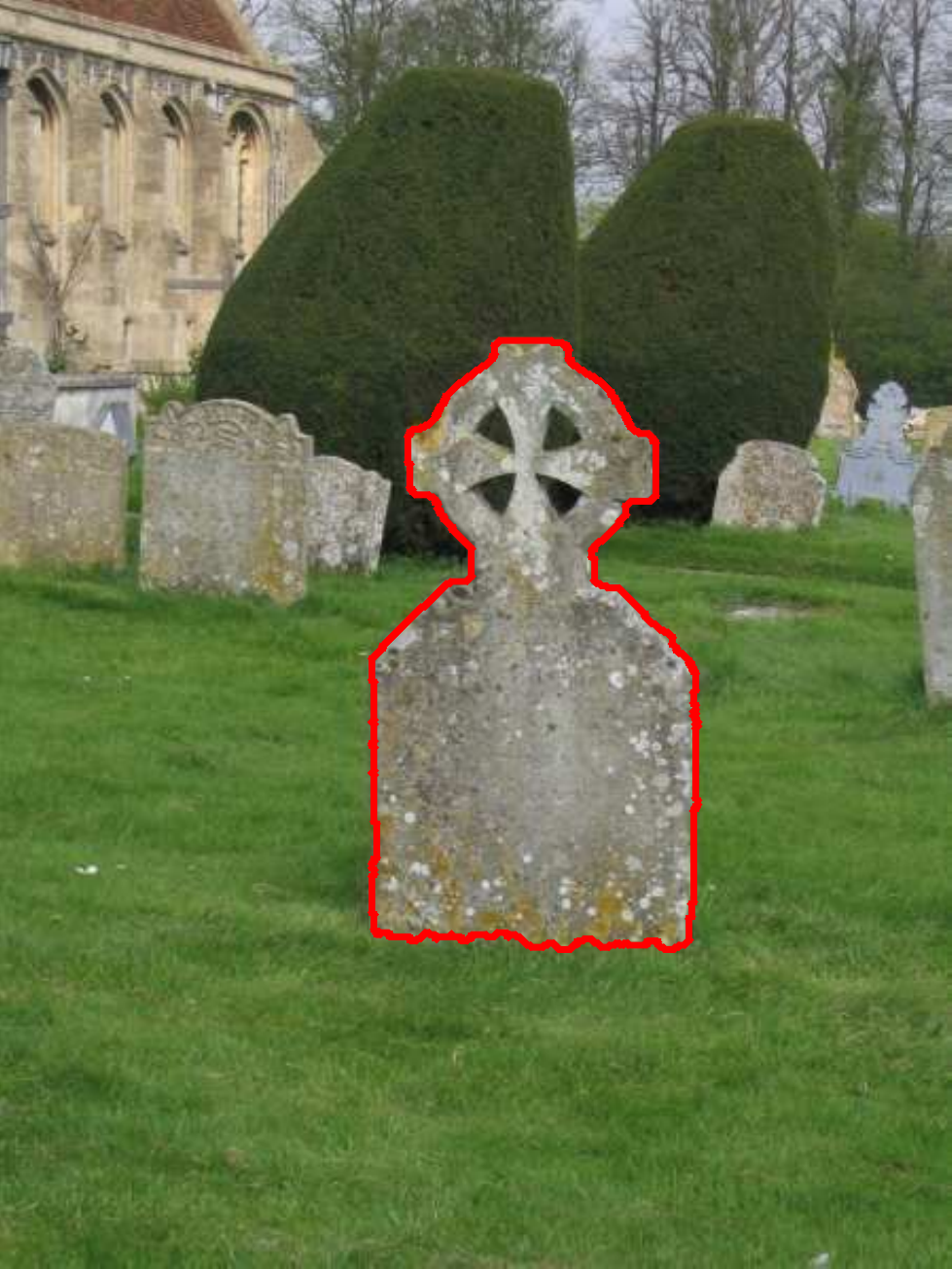}}
\,
\subfloat{\includegraphics[height=.1\columnwidth]{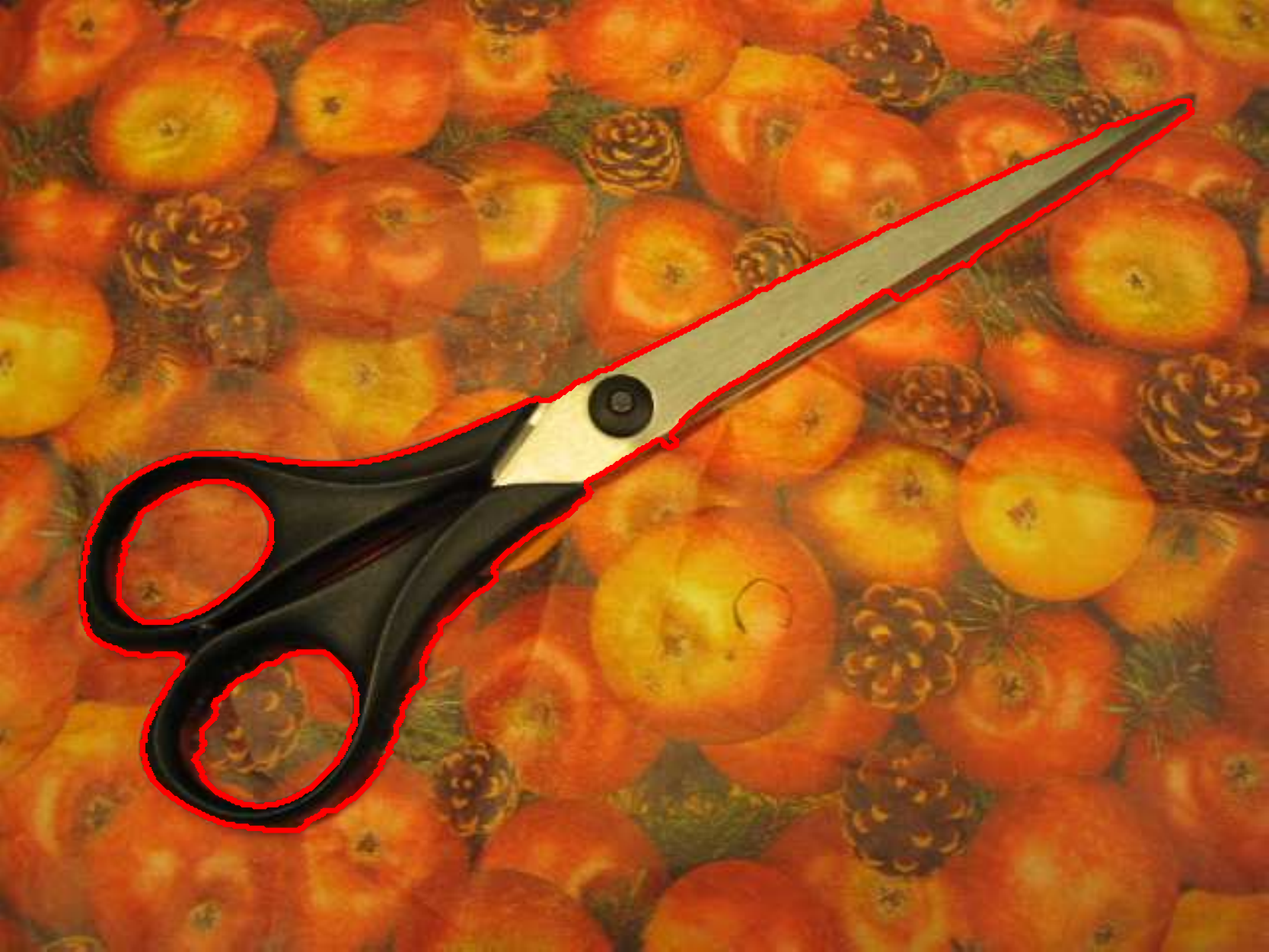}}
\,
\subfloat{\includegraphics[height=.1\columnwidth]{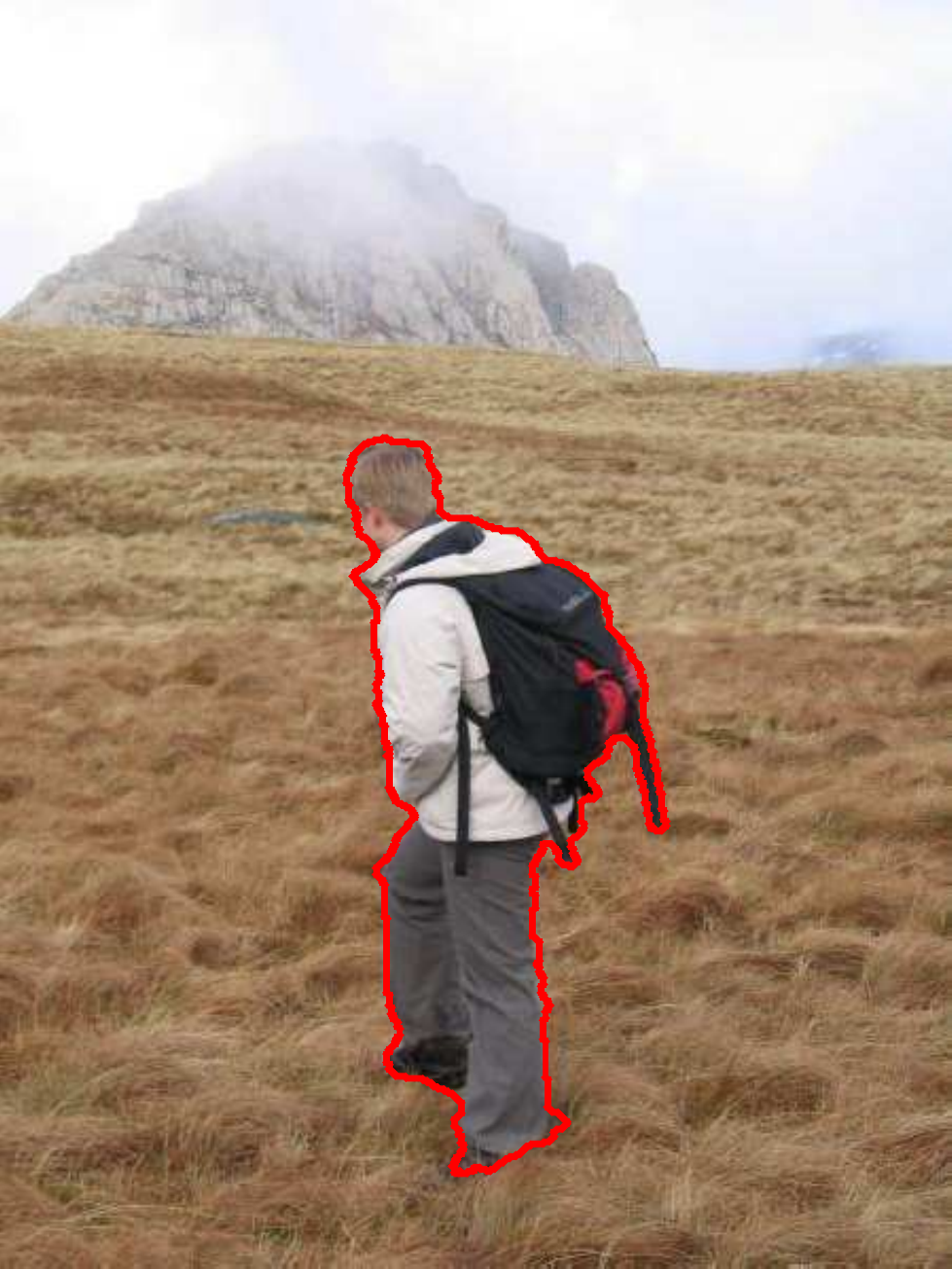}}
\,
\subfloat{\includegraphics[height=.1\columnwidth]{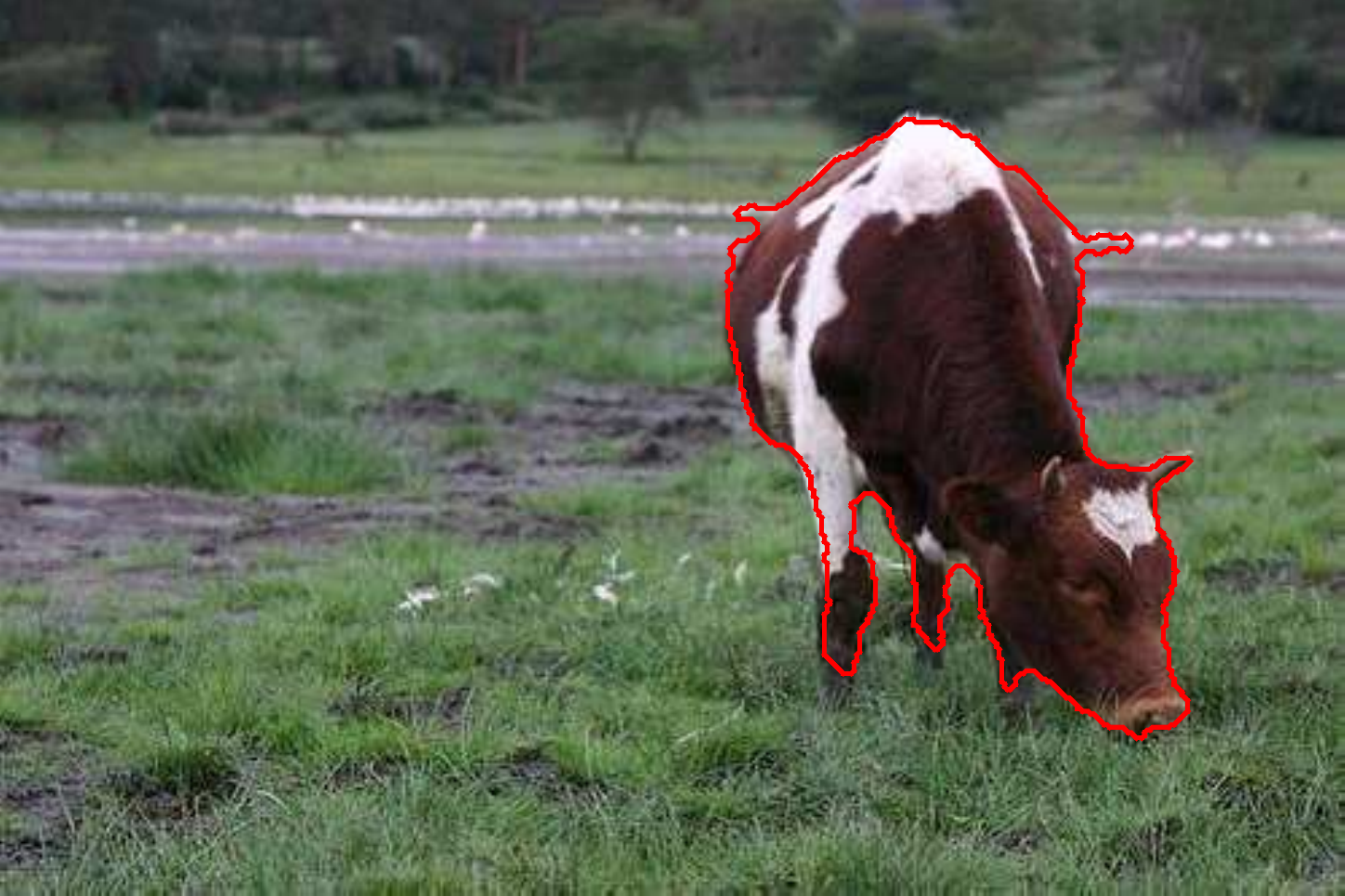}}
\,
\subfloat{\includegraphics[height=.1\columnwidth]{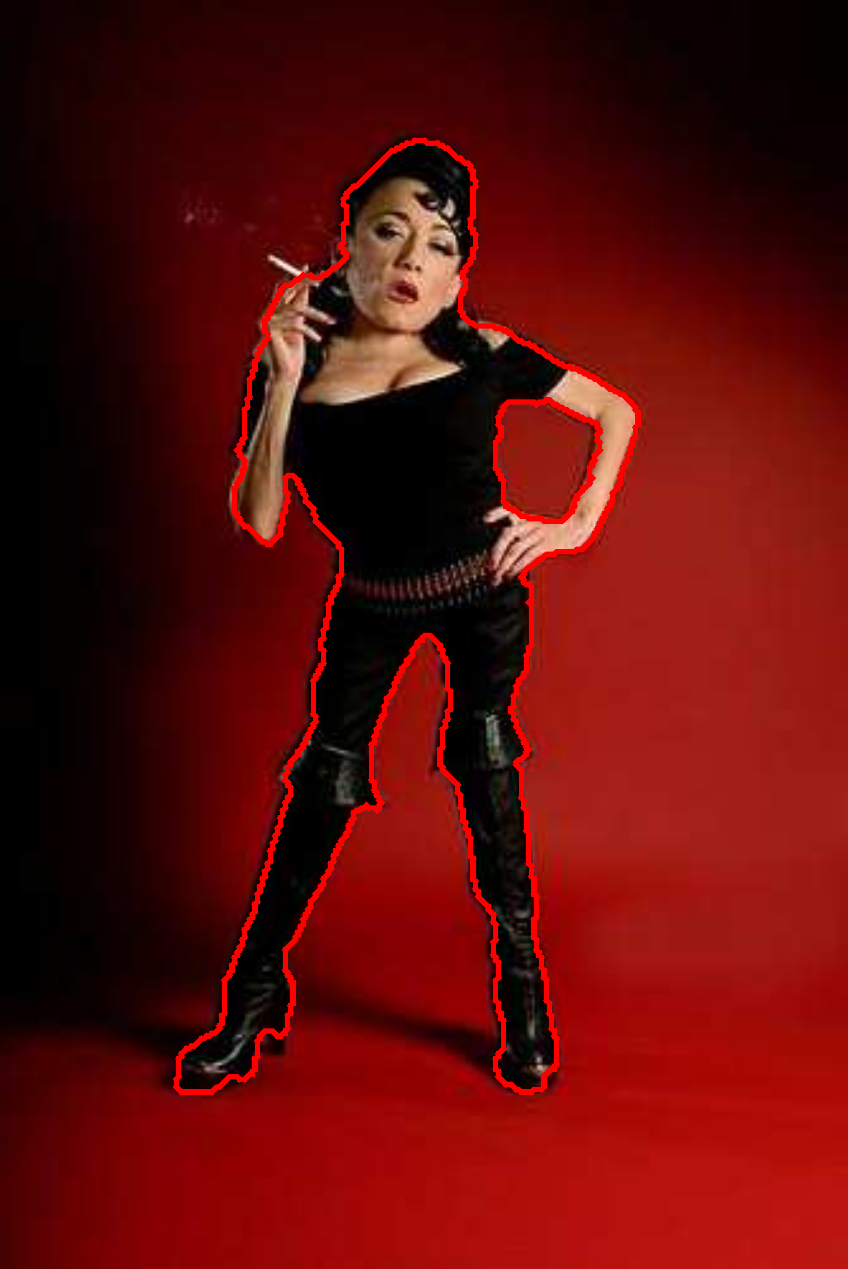}}
\,
\subfloat{\includegraphics[height=.1\columnwidth]{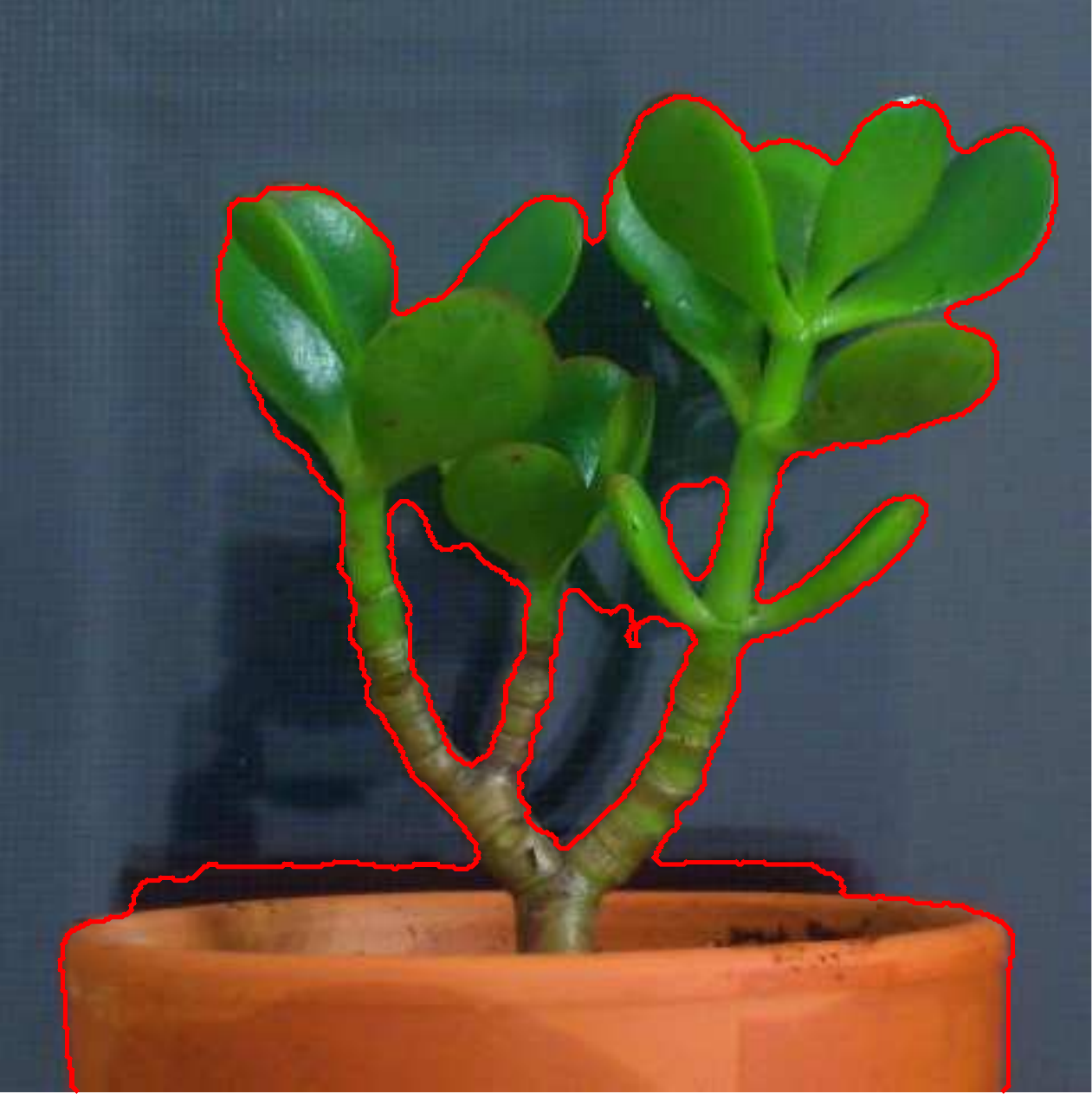}}
\,
\subfloat{\includegraphics[height=.1\columnwidth]{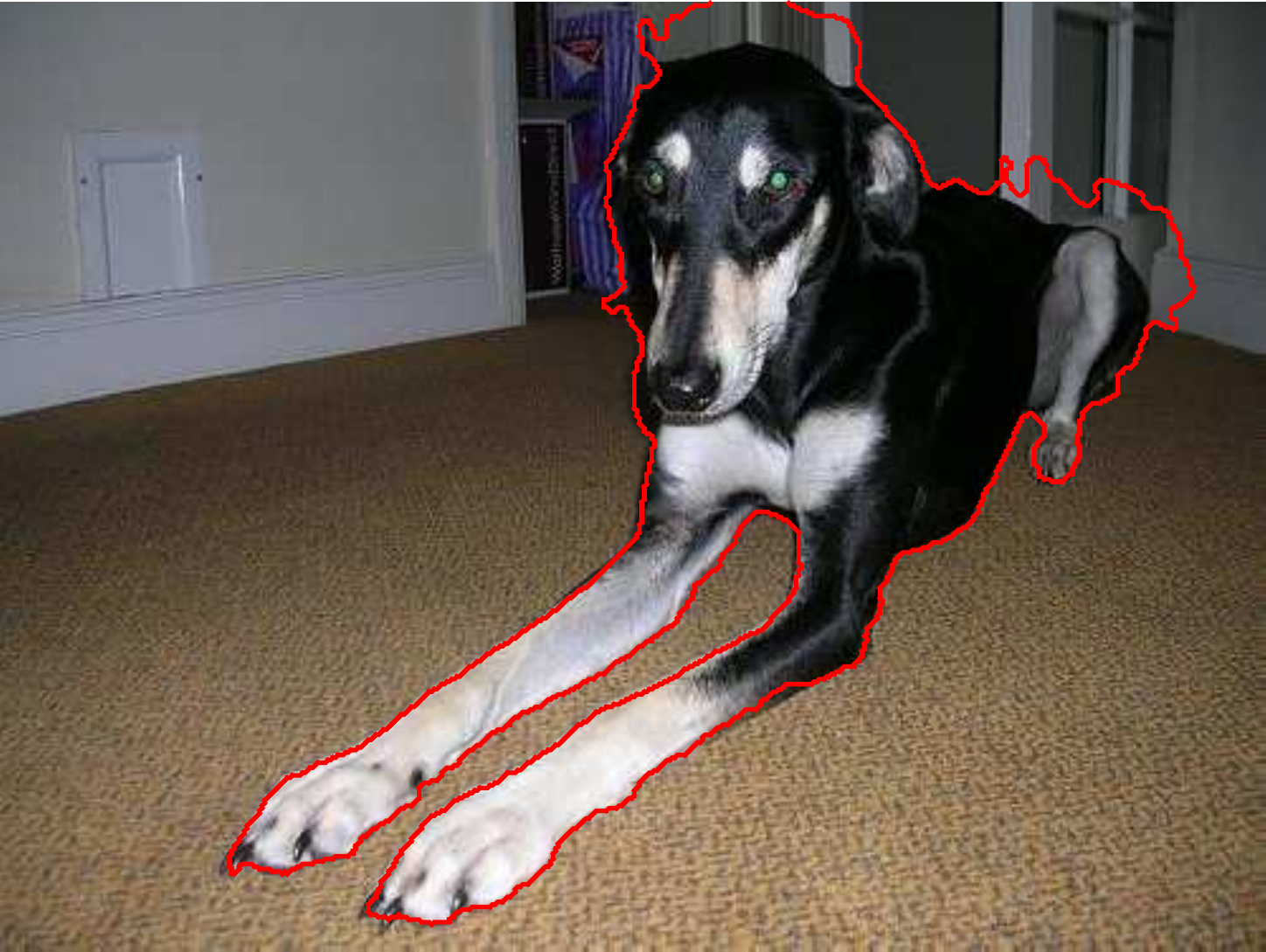}}
\hfil
\caption{Segmentation Results. Columns 1-5: GrabCut Dataset. Columns 6:9: Pascal Dataset. From top to bottom: Ground Truth, segmentation results from GrabCut, Laplacian Coordinates, Power Watershed, Chan-Vese and our method}
\label{fig:comparison_results}
\end{figure*}

\subsection{Graph Segmentation}
In this part, we will show how our active contour framework can be used for graph segmentation. First, we implemented the GAC method on Delaunay graphs, as discussed earlier. In \figurename{\ref{fig:referendum}} we present an example of geographical data segmentation, using the GAC method \cite{cas_kim_sap_97}. The data used is from the Greek Referendum of 01/07/2015 and the ``feature'' function is the percentage of ``NO'' votes in each polling place. The vertices of the graph correspond to the positions of the polling stations that are scaled to fit in the unit square. It is worth noting that the classic GAC method can only extract clusters that are inside the initial curve and the segmentation result depends greatly on the position of the initial contour. 

\begin{figure}[!thb]

\centering
\subfloat{\includegraphics[width=.25\columnwidth]{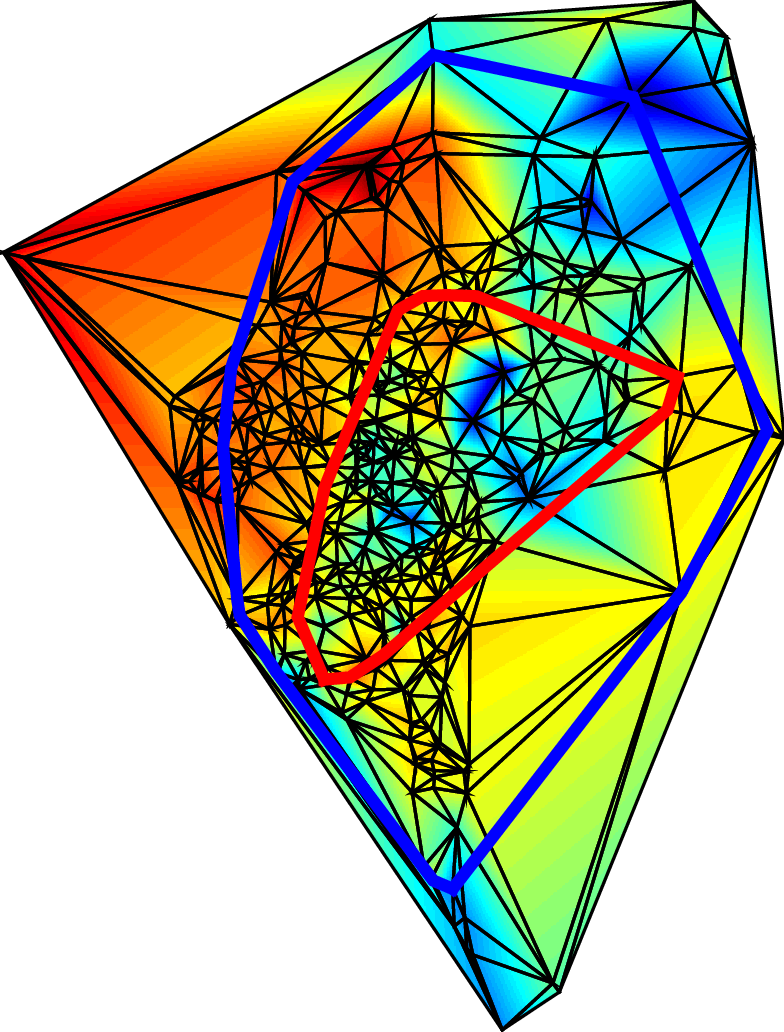}}
\caption{Data from the Greek Referendum of 01/07/2015 in the Athens metropolitan area. The graph numbers 443 vertices, the positions of the polling places.
The color of the surface plot varies from blue to red, with values closer to red indicating a higher percentage of ``NO''. The thin black edges correspond to
the graph edges. The position of the initial contour is shown with a thick blue line. The algorithm located the cluster surrounded by the thick red contour.}
\label{fig:referendum}
\end{figure}

Additionally, we will explore how the proposed Geodesic Active Regions model performs on non-regular graphs, by testing it in sub-sampled versions of the images from the GrabCut dataset. Specifically, let $I$ be an image from the above dataset. For a positive number $N$ we create a Delaunay graph $G$ with $N$ vertices that lie inside the unit square. Then for each vertex $v_i = (x_i,y_i)$ we define $I_G(v_i) = I(x_i,y_i)$. Our purpose is to provide a segmentation on $G$ equipped with the feature function $I_G$. Usually $N$ is significantly smaller than the total number of pixels in an image and this subsampling can help us obtain a fast estimation of the segmentation regions in cases we can afford slightly reduced accuracy. As we can see in Table \ref{table:grabcut_subsampling}, there is a noticeable performance drop as the graph size decreases, but for the most part this can be attributed to the loss of image detail as we move to lower resolutions. An additional reason for the performance degradation is the fact that active contour methods are continuous methods modeled with PDEs and thus their accuracy may suffer when using a coarse discretization. In this series of experiments, the vertices of the Delaunay graphs are irregularly spaced and are created using a variant of the method described in \cite{persson_strang}. Another possible option for the creation of the graphs that is used in \cite{bampis} is to perform watershed segmentation in the original image and choose the vertices as the centroids of the oversegmented regions.

\begin{table}[!h]
\centering
\caption{Performance of the GAR method on the GrabCut Dataset for different subsampling values}
\begin{tabular}{|c|c c c c c|}
\hline
\hline
\textbf{Method} & \textbf{RI} ($\uparrow$) & \textbf{IoU} ($\uparrow$) &\textbf{VoI} ($\downarrow$)  &\textbf{Error} ($\downarrow$)&\\    \hline
64$\times$ & 0.4920 & 0.7147 &  1.1019& 14.368 \%&\\ \hline
32$\times$ & 0.5265 & 0.7418 &  1.0438& 12.771 \%&\\ \hline
16$\times$ & 0.5624 & 0.7622 &  0.9832& 11.236 \%&\\ \hline
8$\times$ & 0.5917 & 0.7782 &  0.9375& 10.173 \%&\\ \hline
Full size & 0.7268 & 0.8519 &  0.6704& 7.793 \%&\\ \hline
\end{tabular}
\label{table:grabcut_subsampling}
\end{table}

\section{Conclusion}\label{section:concl}
In this paper we discussed the problem of active contours on graphs. With the use of the Finite Element method we generalized active contour models on graphs developed a novel computational framework to solve the corresponding levelset equations. Our method can be implemented in arbitrary graphs, however we focused on the family of Delaunay graphs that under certain conditions ensures good convergence properties for the solution of PDEs. One of the main advantages of our method is that it can give more accurate results in small graphs where Finite Difference approaches struggle to approximate efficiently the differential operators involved. Next, we extended the proposed framework to solve locally constrained active contour models and presented a generalization of narrow band levelset methods on graphs that allows to perform fast contour evolution on large graphs. This extension allowed us to reduce the computational complexity by an order of magnitude. Subsequently, we presented several applications to image and graph segmentation. In order to demonstrate the effectiveness of our framework, we developed a supervised extension of the Geodesic Active Contours that uses statistical region information. The proposed method achieves results that are within state-of-the-art. In addition, our method is applicable both on images and irregularly spaced graphs, while being supported by a solid theoretical model and an efficient algorithm.

\bibliographystyle{IEEEtranS}
\bibliography{IEEEabrv,paper.bib}

\end{document}